\documentclass{article} 
\usepackage{iclr2027_conference_added_preprint_flag,times}
\iclrpreprint

\usepackage[utf8]{inputenc} 
\usepackage[T1]{fontenc}    
\usepackage{hyperref}       
\usepackage{url}            
\usepackage{booktabs}       
\usepackage{amsfonts}       
\usepackage{nicefrac}       
\usepackage{microtype}      
\usepackage[dvipsnames]{xcolor}         
\usepackage{amsthm}
\usepackage{multirow}

\newtheorem{theorem}{Theorem}
\newtheorem{lemma}[theorem]{Lemma}
\newtheorem{corollary}[theorem]{Corollary}

\newtheorem{proposition}[theorem]{Proposition}
\newtheorem{remark}[theorem]{Remark}

\usepackage{algorithm}

\usepackage{graphicx}
\usepackage{subcaption}

\def\safedef#1{%
   \ifx#1\undefined
      \expandafter\def\expandafter#1%
   \else
      \errmessage{The \string#1 is defined already}%
      \expandafter\def\expandafter\tmp
   \fi
}

\usepackage{amsthm,amsmath,bbm,amsfonts,amssymb}
\usepackage{dsfont} 
\usepackage{mathtools}
\usepackage{commath}
\mathtoolsset{showonlyrefs=false}
\allowdisplaybreaks 

\usepackage{xspace}
\usepackage[T1]{fontenc}
\usepackage[utf8]{inputenc}
\usepackage[dvipsnames]{xcolor} 

\usepackage{hyperref,url}

\usepackage{wrapfig}
\usepackage{pbox} 
\usepackage{algorithm,algorithmic}

\usepackage{enumitem}
\usepackage[normalem]{ulem}

\newcommand{\kj}[1]{{\color{RedOrange}[#1]}}

{\color{kjsaved}}%

{\color{kjsavedenvkj}}%

\newcommand{\guide}[1]{{\color{Violet}[``#1'']}}

\newenvironment{revised}
{\colorlet{kjsavedrevised}{.}\color{MidnightBlue}}%
{\color{kjsavedrevised}}%
\usepackage{framed}
\usepackage[most]{tcolorbox}
\definecolor{kjgray}{rgb}{.7,.7,.7}

\newtheoremstyle{kjstyle}
{1ex} 
{\topsep} 
{\itshape} 
{} 
{\bfseries} 
{.} 
{.5em} 
{} 

\newtheoremstyle{kjstyle2}
{.0em} 
{.0em} 
{\itshape} 
{} 
{\bfseries} 
{.} 
{.5em} 
{} 

\newtheoremstyle{kjstylenoitalic}
{1ex} 
{\topsep} 
{} 
{} 
{\bfseries} 
{.} 
{.5em} 
{} 

\tcbset{kjboxstyle/.style={title={},breakable,colback=white,enhanced jigsaw,boxrule=1.3pt,sharp corners,colframe=kjgray,boxsep=0pt,coltitle={black},attach title to upper={},left=.8ex,bottom=.4em}}    

\tcbset{kjboxstylec/.style={title={},breakable,colback=white,enhanced jigsaw,boxrule=1.3pt,sharp corners,colframe=kjgray,boxsep=0pt,coltitle={black},attach title to upper={},before skip=1.5ex,left=.8ex,bottom=.4em,top=0.4em,enlarge top by=-0.3em,enlarge bottom by=-0.3em}}    

\usepackage{mdframed}
\usepackage{lipsum}
\definecolor{kjgray}{rgb}{.7,.7,.7}
\makeatletter



\makeatletter
\renewcommand{\paragraph}{%
  \@startsection{paragraph}{4}%
  {\z@}{0.50ex \@plus 1ex \@minus .2ex}{-1em}%
  {\normalfont\normalsize\bfseries}%
}
\makeatother

\usepackage{tabularx} 
\newcolumntype{P}[1]{>{\centering\arraybackslash}p{#1}}
\newcolumntype{M}[1]{>{\centering\arraybackslash}m{#1}}

\def\ddefloop#1{\ifx\ddefloop#1\else\ddef{#1}\expandafter\ddefloop\fi}

\def\ddef#1{\expandafter\def\csname #1#1\endcsname{\ensuremath{\mathbb{#1}}}}
\ddefloop ABCDFGHIJKLMNORSTUWXYZ\ddefloop 

\def\ddef#1{\expandafter\def\csname c#1\endcsname{\ensuremath{\mathcal{#1}}}}
\ddefloop ABCDEFGHIJKLMNOPQRSTUVWXYZ\ddefloop

\def\ddef#1{\expandafter\def\csname b#1\endcsname{\ensuremath{{\mathbf{#1}}}}}
\ddefloop ABCDEFGHIJKLMNOPQRSTUVWXYZ\ddefloop  
\def\ddef#1{\expandafter\def\csname b#1\endcsname{\ensuremath{{\boldsymbol{#1}}}}}
\ddefloop abcdeghijklmnopqrtsuvwxyz\ddefloop  

\def\ddef#1{\expandafter\def\csname h#1\endcsname{\ensuremath{\hat{#1}}}}
\ddefloop ABCDEFGHIJKLMNOPQRSTUVWXYZabcdefghijklmnopqrsuvwxyz\ddefloop 
\def\ddef#1{\expandafter\def\csname hc#1\endcsname{\ensuremath{\hat{\mathcal{#1}}}}}
\ddefloop ABCDEFGHIJKLMNOPQRSTUVWXYZ\ddefloop
\def\ddef#1{\expandafter\def\csname hb#1\endcsname{\ensuremath{\hat{\mathbf{#1}}}}}
\ddefloop ABCDEFGHIJKLMNOPQRSTUVWXYZ\ddefloop %
\def\ddef#1{\expandafter\def\csname hb#1\endcsname{\ensuremath{\hat{\boldsymbol{#1}}}}}
\ddefloop abcdefghijklmnopqrstuvwxyz\ddefloop %

\def\ddef#1{\expandafter\def\csname t#1\endcsname{\ensuremath{\tilde{#1}}}}
\ddefloop ABCDEFGHIJKLMNOPQRSTUVWXYZabcdefgijklmnpqtsuvwxyz\ddefloop 
\def\ddef#1{\expandafter\def\csname tc#1\endcsname{\ensuremath{\tilde{\mathcal{#1}}}}}
\ddefloop ABCDEFGHIJKLMNOPQRSTUVWXYZ\ddefloop
\def\ddef#1{\expandafter\def\csname tb#1\endcsname{\ensuremath{\tilde{\mathbf{#1}}}}}
\ddefloop ABCDEFGHIJKLMNOPQRSTUVWXYZ\ddefloop
\def\ddef#1{\expandafter\def\csname tb#1\endcsname{\ensuremath{\tilde{\boldsymbol{#1}}}}}
\ddefloop abcdefghijklmnopqrstuvwxyz\ddefloop %

\def\ddef#1{\expandafter\def\csname bar#1\endcsname{\ensuremath{\bar{#1}}}}
\ddefloop ABCDEFGHIJKLMNOPQRSTUVWXYZabcdefghijklmnopqrtsuvwxyz\ddefloop
\def\ddef#1{\expandafter\def\csname barc#1\endcsname{\ensuremath{\bar{\mathcal{#1}}}}}
\ddefloop ABCDEFGHIJKLMNOPQRSTUVWXYZ\ddefloop
\def\ddef#1{\expandafter\def\csname barb#1\endcsname{\ensuremath{\bar{\mathbf{#1}}}}}
\ddefloop ABCDEFGHIJKLMNOPQRSTUVWXYZ\ddefloop
\def\ddef#1{\expandafter\def\csname barb#1\endcsname{\ensuremath{\bar{\boldsymbol{#1}}}}}
\ddefloop abcdefghijklmnopqrstuvwxyz\ddefloop %

\def\ddef#1{\expandafter\def\csname war#1\endcsname{\ensuremath{\overline{#1}}}}
\ddefloop ABCDEFGHIJKLMNOPQRSTUVWXYZabcdefghijklmnopqrtsuvwxyz\ddefloop
\def\ddef#1{\expandafter\def\csname warc#1\endcsname{\ensuremath{\overline{\mathcal{#1}}}}}
\ddefloop ABCDEFGHIJKLMNOPQRSTUVWXYZ\ddefloop
\def\ddef#1{\expandafter\def\csname warb#1\endcsname{\ensuremath{\overline{\mathbf{#1}}}}}
\ddefloop ABCDEFGHIJKLMNOPQRSTUVWXYZ\ddefloop
\def\ddef#1{\expandafter\def\csname warb#1\endcsname{\ensuremath{\overline{\boldsymbol{#1}}}}}
\ddefloop abcdefghijklmnopqrstuvwxyz\ddefloop %

\def\eps{\varepsilon}
\def\epsilon{\varepsilon}

\usepackage{pgffor}
\def\greeksymbols{alpha,beta,gamma,gam,delta,dt,eps,epsilon,zeta,eta,theta,th,iota,kappa,kap,lambda,lam,mu,nu,xi,pi,rho,sigma,sig,tau,phi,chi,psi,omega,om,Gamma,Gam,Delta,Dt,Theta,Th,Lambda,Lam,Pi,Sigma,Sig,Phi,Psi,Omega,Om}
\def\greeksymbolsnoeta{alpha,beta,gamma,gam,delta,dt,eps,epsilon,zeta,theta,th,iota,kappa,kap,lambda,lam,mu,nu,xi,pi,rho,sigma,sig,tau,phi,chi,psi,omega,om,Gamma,Gam,Delta,Dt,Theta,Th,Lambda,Lam,Pi,Sigma,Sig,Phi,Psi,Omega,Om} 

\foreach \x in \greeksymbolsnoeta{\expandafter\xdef\csname b\x\endcsname{\noexpand\ensuremath{\noexpand\boldsymbol{\csname \x\endcsname}}}}

\foreach \x in \greeksymbols{\expandafter\xdef\csname h\x\endcsname{\noexpand\ensuremath{\noexpand\hat{\csname \x\endcsname}}}}
\foreach \x in \greeksymbolsnoeta{\expandafter\xdef\csname hb\x\endcsname{\noexpand\ensuremath{\noexpand\hat{\noexpand\boldsymbol{ \csname \x\endcsname}}}}}

\foreach \x in \greeksymbols{\expandafter\xdef\csname bar\x\endcsname{\noexpand\ensuremath{\noexpand\bar{\csname \x\endcsname}}}}
\foreach \x in \greeksymbolsnoeta{%
\expandafter\xdef\csname barb\x\endcsname{\noexpand\ensuremath{\noexpand\bar{\noexpand\boldsymbol{ \csname \x\endcsname}}}}
}

\foreach \x in \greeksymbols{\expandafter\xdef\csname t\x\endcsname{\noexpand\ensuremath{\noexpand\tilde{\csname \x\endcsname}}}}
\foreach \x in \greeksymbolsnoeta{\expandafter\xdef\csname tb\x\endcsname{\noexpand\ensuremath{\noexpand\tilde{\noexpand\boldsymbol{ \csname \x\endcsname}}}}}

\providecommand{\normz}[2][-1]{
\ensuremath{\mathinner{
\ifthenelse{\equal{#1}{-1}}{ 
\!\left\|#2\right\|}{}
\ifthenelse{\equal{#1}{0}}{ 
\|#2\|}{}
\ifthenelse{\equal{#1}{1}}{ 
\bigl\|#2\bigr\|}{}
\ifthenelse{\equal{#1}{2}}{ 
\Bigl\|#2\Bigr\|}{}
\ifthenelse{\equal{#1}{3}}{ 
\biggl\|#2\biggr\|}{}
\ifthenelse{\equal{#1}{4}}{ 
\Biggl\|#2\Biggr\|}{}
}} 
}  

\providecommand{\floor}[2][-1]{
\ensuremath{\mathinner{
\ifthenelse{\equal{#1}{-1}}{ 
\!\left\lfloor#2\right\rfloor}{}
\ifthenelse{\equal{#1}{0}}{ 
\lfloor#2\rfloor}{}
\ifthenelse{\equal{#1}{1}}{ 
\!\bigl\lfloor#2\bigr\rfloor}{}
\ifthenelse{\equal{#1}{2}}{ 
\!\Bigl\lfloor#2\Bigr\rfloor}{}
\ifthenelse{\equal{#1}{3}}{ 
\!\biggl\lfloor#2\biggr\rfloor}{}
\ifthenelse{\equal{#1}{4}}{ 
\!\Biggl\lfloor#2\Biggr\rfloor}{}
}} 
}

\providecommand{\ceil}[2][-1]{
\ensuremath{\mathinner{
\ifthenelse{\equal{#1}{-1}}{ 
\!\left\lceil#2\right\rceil}{}
\ifthenelse{\equal{#1}{0}}{ 
\lceil#2\rceil}{}
\ifthenelse{\equal{#1}{1}}{ 
\!\bigl\lceil#2\bigr\rceil}{}
\ifthenelse{\equal{#1}{2}}{ 
\!\Bigl\lceil#2\Bigr\rceil}{}
\ifthenelse{\equal{#1}{3}}{ 
\!\biggl\lceil#2\biggr\rceil}{}
\ifthenelse{\equal{#1}{4}}{ 
\!\Biggl\lceil#2\Biggr\rceil}{}
}} 
}

\newcommand{\wbar}[1]{{\ensuremath{\overline{#1}}}}  

\def\larrow{\ensuremath{\leftarrow}}

\definecolor{mygrn}{rgb}{0,.8,0}
\definecolor{myred}{rgb}{.8,0,0}

\DeclareMathOperator{\EE}{\mathbb{E}} 

\DeclareMathOperator{\PP}{\mathbb{P}}

\DeclareMathOperator{\Var}{{\mathrm{Var}}}

\DeclareMathOperator*{\argmax}{arg~max}
\DeclareMathOperator*{\argmin}{arg~min}

\DeclareMathOperator{\sign}{{\mathrm{sign}}}

\def\clip#1{\wbar{\del{#1}}}
\DeclareMathOperator{\one}{\mathds{1}\hspace{-.1em}}

\DeclarePairedDelimiterX{\inp}[2]{\langle}{\rangle}{#1, #2}

\newcommand\declareop[3]{%
  \newcommand#1{%
    \mskip\muexpr\medmuskip*#2\relax
    {#3}%
    \mskip\muexpr\medmuskip*#2\relax
}}
\declareop\capprox{1}{{\sr{\const}{\approx}}} 
\declareop\logapprox{1}{{\sr{\mathrm{log}}{\approx}}} 

\def\Bernoulli{{\ensuremath{\mathrm{Bernoulli}}}}

\def\const{\mathsf{const}}

\def\kl{{\mathrm{kl}}}

\usepackage{pifont}
\newcommand{\sr}{\stackrel}

\makeatletter
\newcommand{\vast}{\bBigg@{3}}
\newcommand{\Vast}{\bBigg@{4}}
\makeatother

\newenvironment{talign*}
 {\csname align*\endcsname}
 {\endalign}

\def\chrulefill{\leavevmode\leaders\hrule height 0.7ex depth \dimexpr0.4pt-0.7ex\hfill\kern0pt}

\renewcommand{\cite}{\citep}

\newcommand{\yinan}[1]{{\color{blue}[Yinan: #1]}}
\newcommand{\han}[1]{{\color{green}[Han: #1]}}
\usepackage[toc,page,header]{appendix}
\usepackage{minitoc}
\usepackage{silence}
\newif\ifFINAL

\FINALtrue 

\ifFINAL
  
  \def\guide#1{}
  \def\kj#1{}
  \def\kjnew#1{}
  \def\yinan#1{}
  \def\han#1{}

\else
  \usepackage{transparent}
  \usepackage[inline]{showlabels}
  \usepackage{rotating}
  \renewcommand{\showlabelfont}%
  {\transparent{0.8}\scriptsize\bf\slshape\color{Lavender}}
\fi

\def\Unif{\text{Unif}}

\newcommand{\ind}{\mathbf{1}}

\title{%
  {\fontsize{16}{19}\selectfont
  Sign Betting for Misspecified Regression:}\\
  {\fontsize{16}{19}\selectfont
  Variance Adaptivity and Optimal Approximation}%
}

\author{%
  Yinan Li \\
  Department of Computer Science \\
  University of Arizona \\
  \texttt{yinanli@arizona.edu}
  \And
  Han Han \\
  Department of Computer Science \\
  University of Arizona \\
  \texttt{hanhan1@arizona.edu}
  \AND
  Kwang-Sung Jun \\
  GSAI/CSE \\
  POSTECH \\
  \texttt{kwangsungjun@postech.ac.kr}
}

\begin{document}
\doparttoc 
\faketableofcontents 

\setlength{\abovedisplayskip}{4pt}%
\setlength{\belowdisplayskip}{4pt}%
\setlength{\abovedisplayshortskip}{4pt}%
\setlength{\belowdisplayshortskip}{4pt}

\textfloatsep=.5em

\maketitle

\begin{abstract}
We study bounded regression under model misspecification, evaluating a function $f$ by its $L_1(P_X)$ distance to the conditional mean
$\eta(x)=\mathbb{E}[Y\mid X=x]$, rather than the standard $L_2(P_X)$. This $L_1$ metric is natural for downstream decision problems such as bandits and reinforcement learning. We analyze the min--max estimator of \citet{baraud2025estimating}, which we call \emph{sign-betting regression}. Our main results show that sign-betting regression achieves high-probability
$L_1$ oracle inequalities under misspecification. 
Under realizability,
our bound adapts to the average conditional variance and recovers the variance-adaptive rate of \citet{pmlr-v336-li26a}.
We further characterize the optimal universal approximation factor for proper learning by showing that sign betting achieves factor 3, while no proper learner can satisfy a distribution-free, high-probability $L_1$ oracle inequality with a universal approximation factor strictly below 3.
We extend the upper-bound analysis beyond finite classes and complement the
theory with experiments in synthetic and real-world model-selection tasks.
\end{abstract}

\section{Introduction}
\label{sec:introduction}

\vspace{-.5em}

Regression is a fundamental problem in statistics and machine learning and often an intermediate step in a downstream decision problem, with a
quantity of interest being the underlying conditional mean. 
\kj{[NH] throughout the paper you say as if it is new to learn the conditional mean, but even the traditional results are still learning about the conditional mean. find a way to avoid keep saying we try to learn the conditional mean}
We study bounded regression, where $(X,Y)$ denotes a covariate--response pair with $Y\in[0,1]$, and evaluate a predictor $f$ by its expected absolute discrepancy from the conditional mean,
$\epsilon_1(f):=\EE\!\left[|f(X)-\eta(X)|\right]$, where $\eta(x):=\EE[Y\mid X=x]$.

The $L_1$ accuracy of a conditional-mean predictor has several downstream interpretations. For binary responses, it controls the excess risk of the plug-in classifier
$g_f(x):=\mathbf{1}\{f(x)\ge 1/2\}$ relative to the Bayes classifier
$g^\star(x):=\mathbf{1}\{\eta(x)\ge 1/2\}$. Writing
$R_{\mathrm{cls}}(g):=\PP(g(X)\neq Y)$ for the classification risk, the
standard comparison inequality gives $R_{\mathrm{cls}}(g_f)-R_{\mathrm{cls}}(g^\star)
\le 2\epsilon_1(f)$
\citep{audibert2007fast}. 
The $L_1$ metric is also natural for downstream decision problems such as bandits and reinforcement learning, where policy suboptimality is typically controlled through differences in expected rewards or costs rather than their squares~\citep{ayoub2025rectifying}.
Finally, when $f$ is used as a reward model to select the highest-scoring response among $N$ candidates, each marginally distributed according to $P_X$, the expected reward loss relative to the best candidate in the same set is at most $N\epsilon_1(f)$. This
connects $L_1$ reward-model accuracy to recent analyses of Best-of-$N$ and
related inference-time selection procedures
\citep{huang2025best}.
See derivations in Appendix~\ref{app:l1-downstream}.

\vspace{-.5em}

\paragraph{Regression under misspecification.}
Let $\cF$ be a class of $[0,1]$-valued predictors and define $\epsilon_{1,*}
    :=
    \inf_{f\in\cF}
    \epsilon_1(f)$. 
Under misspecification, we do not assume that $\eta\in\cF$. We ask whether a learner can compete with the best achievable $L_1$ error within the class. Ideally, an oracle inequality of
the form
\begin{equation}
    \epsilon_1(\widehat f)
    \le
    C_{\rm aprx}\,\epsilon_{1,*}
    +
    r_n,
    \qquad
    r_n\to0\quad\text{as }n\to\infty,
    \label{eq:basic-goal}
\end{equation}
where $C_{\rm aprx}$ is a universal constant. 
We call $C_{\rm aprx}$ the \emph{approximation factor} and $r_n$ the \emph{estimation term}. Beyond requiring
$r_n$ to vanish, we seek guarantees that adapt to the expected conditional variance $\EE[\Var(Y\mid X)]$ and become sharper when this quantity is small. 
This leads to two central questions:
\vspace{-.5em}
\begin{quote}
    \emph{What approximation and noise-adaptive guarantees are achievable
    for $L_1$ conditional-mean estimation under model misspecification?
    When the learner is required to return a member of $\cF$, what is the
    smallest universal approximation factor in a distribution-free
    high-probability oracle inequality?}
\end{quote}

\vspace{-.5em}

\paragraph{Why standard regression objectives do not resolve the problem.}
A loss that identifies the conditional mean under unrestricted optimization can select an approximation to $\eta$ that differs from the best $L_1$ approximation
when the model class is misspecified. Squared loss is the canonical example: minimizing its
population risk over $\cF$ selects the best $L_2$ approximation to $\eta$,
rather than the best $L_1$ approximation targeted by
\eqref{eq:basic-goal}. More generally, standard losses
such as squared and log loss induce their own notions of projection onto
$\cF$ under misspecification \citep{mey2021consistency}. We show that this
mismatch is fundamental: e.g., squared-loss ERM admits an $L_1$ guarantee whose
approximation term scales as $\sqrt{\epsilon_{1,*}}$, and this dependence is
unavoidable even in a noiseless problem with a function class of size 2.
Consequently, squared-loss ERM cannot satisfy \eqref{eq:basic-goal} with $C_{\rm aprx}$ being a universal constant. 
A second issue is statistical adaptation. Although an $n^{-1/2}$ estimation
term is unavoidable in the worst case over distributions of $(X,Y)$ with $Y\in[0,1]$ 
\citep{tsybakov2009introduction,wainwright19high}, it can be overly
pessimistic when the conditional variance of the labels is small
\citep{foster21efficient,pmlr-v336-li26a}. Our goal is to achieve
linear dependence on the best $L_1$ approximation error $\eps_{1,*}$ together with adaptation to low noise. As we show in Section~\ref{sec:standard-losses}, the standard regression
objectives considered therein
do not, in general, provide either of these desiderata.
\vspace{-0.5em}
\paragraph{Sign-Betting Regression.}
We study a min-max 
estimator introduced by
\citet{baraud2025estimating} for fixed-design regression with possibly
heteroscedastic and heavy-tailed errors. We refer to this procedure as
\emph{Sign-Betting Regression (SBR)}, reflecting the directional betting interpretation developed in Section~\ref{sec:sign-betting-population}.

\vspace{-0.5em}
\paragraph{Our results.}
Our first main result is a high-probability generalization inequality for SBR. In the general misspecified setting, it implies $\epsilon_1(\widehat f)
    \le
    3\epsilon_{1,*}
    +
    r_n$, 
where the estimation term $r_n$ adapts to the residual second moment of an $L_1$-optimal
comparator. Under realizability, the approximation
terms disappear and the leading statistical term adapts directly to the
average conditional variance $\EE[\Var(Y\mid X)]$, recovering the
variance-adaptive 
rate of \citet{pmlr-v336-li26a}.
Our second main result identifies the optimal universal approximation factor for proper learning. 
We prove that, for proper learning, \eqref{eq:basic-goal} cannot hold with any universal constant $C_{\rm aprx}<3$ and any distribution-independent $r_n$.
Since SBR achieves factor $3$, this characterizes the optimal universal proper approximation factor under our learning criterion. 
Finally, we extend the upper-bound analysis beyond finite classes and complement our theoretical results with experiments on synthetic data, real-world datasets, and LLM model selection. 

Our setting and questions differ from those of
\citet{baraud2025estimating} in several respects. They study
fixed-design regression and expected $L_1$ error evaluated on the observed
design points. We instead consider i.i.d.\ random design and study
high-probability oracle inequalities for the population error
$\|f-\eta\|_{L_1(P_X)}$. 
A direct black-box conversion---conditioning on the covariates, applying their fixed-design result, and then using global uniform convergence---yields a population guarantee in expectation but incurs a variance-independent $n^{-1/2}$ generalization term, which can be substantially larger than our bound on low-noise instances. 
Our direct random-design analysis achieves variance adaptation. See Appendix~\ref{app:comparison-baraud-maillard} for details.

\vspace{-0.5em}
\paragraph{Paper organization.}
Section~\ref{sec:setting} formalizes the learning problem, the $L_1$ oracle
criterion, and the noise quantities used throughout the paper.
Section~\ref{sec:standard-losses} shows how standard regression objectives
can fail both the desired approximation and noise-adaptation guarantees.
Section~\ref{sec:sign-betting} develops the population geometry of sign
betting and establishes our main variance-adaptive oracle inequality.
Section~\ref{sec:optimality} characterizes the optimal universal
approximation factor for proper learning, and
Section~\ref{sec:sbr-infinite-classes} extends the upper-bound framework
beyond finite classes. Appendix~\ref{app:related-work} discusses related work and connections to
prior results, and Appendix~\ref{app:exp} presents the experiments.
The remaining appendices contain proofs and supporting technical material.

\vspace{-1em}
\section{Problem setting and learning objective}
\label{sec:setting}

\vspace{-0.5em}

\paragraph{Setting and notation.} 
Let $\cX$ be a measurable input space, and let $P$ be an unknown distribution over $\cX\times[0,1]$.
We observe $n$ i.i.d.\ samples $(X_i,Y_i)_{i=1}^n\sim P^n$. 
We write $P_X$ for the marginal distribution of $X$ and $\eta(x):=\EE[Y\mid X=x]$ for the regression function. 
We consider a nonempty class $\cF\subseteq\{f:\cX\to[0,1]\}$ of measurable predictors, without assuming that $\eta\in\cF$. 
\kj{$\larrow$ up to this sentence, I would move to the first paragraph of introduction.}\yinan{I would love to, but the time is so limited for me to rewrite}
For a measurable $g:\cX\to\mathbb R$ and $p\in[1,\infty)$, we write $\normz{g}_p:=(\EE_{X\sim P_X}|g(X)|^p)^{1/p}$. For a measurable function $Z=Z(X,Y)$, we use the shorthand $PZ:=\EE[Z(X,Y)]$ and $P_nZ:=n^{-1}\sum_{i=1}^n Z(X_i,Y_i)$.

\vspace{-0.5em}

\paragraph{Approximation error.} For $p\in[1,\infty)$ and $f\in\cF$, define $\epsilon_p(f):=\normz{f-\eta}_p$ and $\epsilon_{p,*}:=\inf_{f\in\cF}\epsilon_p(f)$. Our main performance criterion is the $L_1$ conditional-mean approximation error $\epsilon_1(f)=\normz{f-\eta}_1=\EE|f(X)-\eta(X)|$, and $\epsilon_{1,*}$ is the smallest such error attainable within $\cF$. When the infimum exists, we let $f_\star\in\argmin_{f\in\cF}\epsilon_1(f)$ denote a best $L_1$ approximation to $\eta$ in $\cF$. In particular, such a minimizer always exists for the finite classes considered in our main results. 
We say that $(P,\cF)$ is \emph{realizable} (or \emph{well-specified}) if $\eta\in\cF$, and
\emph{misspecified} otherwise.
\kj{usually, you either use the pair of phrases ‘realizable vs agnostic’ or ‘well-specified vs misspecified’}
For a finite class, realizability is equivalent to $\epsilon_{1,*}=0$. Our results do not assume realizability. Under misspecification, our guarantees compare the learned predictor with arbitrary comparators $f_0\in\cF$, allowing the bound to trade off their approximation error (how well $f_0$ approximates $\eta$) and comparator-dependent estimation terms. 
We distinguish this criterion in Appendix~\ref{app:absolute-loss} from the absolute-loss risk $\EE|Y-f(X)|$, which is generally a different population objective. 

\vspace{-0.5em}
\paragraph{Noise and approximation decomposition.} Let $\xi:=Y-\eta(X)$, $\sigma^2(X):=\Var(Y\mid X)=\EE[\xi^2\mid X]$, and $\bar\sigma^2:=\EE[\sigma^2(X)]$. Then $\EE[\xi\mid X]=0$, and for every predictor $f$, $\EE[(Y-f(X))^2\mid X]=\sigma^2(X)+(\eta(X)-f(X))^2$. The first term is the conditional label variance, while the second is the squared approximation error. This decomposition will enter the variance-dependent bounds under misspecification.

\vspace{-0.5em}

\paragraph{Proper learning.}
A learner is \emph{proper} with respect to $\cF$ if it always returns a
predictor $\hf\in\cF$; otherwise it is \emph{improper}. Our main results focus on proper learning: our main algorithm is proper, and our lower bound characterizes the best universal approximation factor achievable by proper learners. 

\vspace{-0.5em}
\section{Limitations of standard regression objectives}
\label{sec:standard-losses}

\vspace{-0.5em}

We first ask whether familiar regression objectives already provide the $L_1$ oracle guarantee in \eqref{eq:basic-goal}. We focus here on squared loss and log loss, while deferring absolute loss to Appendix~\ref{app:absolute-loss} and betting loss of \citet{pmlr-v336-li26a} to Appendix~\ref{app:original-betting-lb}. These objectives fail for two conceptually different reasons. Absolute loss is misaligned with the target itself: its population minimizer is generally a conditional median rather than the conditional mean $\eta$. By contrast, squared loss, log loss, and betting loss are all aligned with conditional-mean regression under realizability, with betting loss enjoying variance-adaptive guarantees \citep{pmlr-v336-li26a}. Nevertheless, under model misspecification, squared loss, log loss, and betting loss may not produce a predictor whose $L_1$ error is within any universal constant factor of $\epsilon_{1,*}$. These limitations motivate an objective that achieves an $L_1$ oracle guarantee with a universal constant
approximation factor under misspecification while enjoying variance-adaptive statistical guarantees.

\vspace{-0.5em}

\subsection{Squared-loss ERM}
\label{sec:squared-loss}

\vspace{-0.5em}

We begin with the familiar squared-loss ERM, $\widehat f_{\rm sq}
    \in
    \argmin_{f\in\cF}
    \frac1n\sum_{i=1}^n \bigl(Y_i-f(X_i)\bigr)^2$.
Squared loss is calibrated to the conditional mean when optimization is unrestricted. 
Under misspecification, however, minimizing squared loss over $\cF$ does not yield the best $L_1$ approximation to $\eta$.
Indeed, for every $f\in\cF$,
\begin{equation}
    \EE\bigl[(Y-f(X))^2\bigr]
    =
    \EE\!\left[\Var(Y\mid X)\right]
    +
    \normz{f-\eta}_2^2.
    \label{eq:squared-loss-population-decomposition}
\end{equation}
Thus, at the population level, squared loss selects an $L_2$-optimal approximation to $\eta$, whereas our benchmark $\epsilon_{1,*}$ is defined by the best $L_1$ approximation. The following proposition quantifies the resulting finite-sample guarantee. The proof is deferred to Appendix~\ref{app:squared-loss-upper}.

\begin{proposition}[$L_1$ guarantee for squared-loss ERM]
\label{prop:squared-loss-upper}
Suppose $Y\in[0,1]$ and $\cF\subseteq[0,1]^{\cX}$ is finite. There exists a universal constant $C>0$ such that, with probability at least $1-\delta$,
\begin{equation}
    \epsilon_1(\widehat f_{\rm sq})
    \le
    \epsilon_{2,*}
    +
    C\sqrt{\frac{\log(2|\cF|^2/\delta)}{n}}
    \le
    \sqrt{\epsilon_{1,*}}
    +
    C\sqrt{\frac{\log(2|\cF|^2/\delta)}{n}}.
    \label{eq:squared-loss-l1-upper}
\end{equation}
\end{proposition}

The first inequality in \eqref{eq:squared-loss-l1-upper} reflects the natural geometry of squared loss: apart from estimation error, squared-loss ERM competes with $\epsilon_{2,*}$. The second inequality follows from $\epsilon_{2,*}\le\sqrt{\epsilon_{1,*}}$ for $[0,1]$-valued hypothesis class $\cF$. Next, we show in Proposition~\ref{prop:squared-loss-no-constant-factor} that this square-root dependence on $\epsilon_{1,*}$ is not merely an artifact of converting an $L_2$ bound to $L_1$.

\begin{proposition}[Squared loss has no constant $L_1$ approximation factor]
\label{prop:squared-loss-no-constant-factor}
For every $p\in(0,1/4)$, there exists a noiseless regression problem and a two-element class $\cF\subseteq[0,1]^{\cX}$ such that $\epsilon_{1,*}=p$, while the unique population squared-loss minimizer,
$
    f_{\rm sq}^*
    \in
    \argmin_{f\in\cF}\EE\bigl[(Y-f(X))^2\bigr],
$
satisfies
\begin{equation}
    \epsilon_1(f_{\rm sq}^*)
    =
    \frac12\sqrt{\epsilon_{1,*}}.
    \label{eq:squared-loss-approximation-lower}
\end{equation}
Consequently, squared-loss ERM cannot satisfy the basic goal \eqref{eq:basic-goal} for any finite universal approximation factor.
\end{proposition}

Proposition~\ref{prop:squared-loss-no-constant-factor} shows that the $\sqrt{\epsilon_{1,*}}$ dependence in \eqref{eq:squared-loss-l1-upper} is unavoidable up to a universal constant. This is a population-level obstruction: even with noiseless labels and arbitrarily many samples, squared loss can prefer an $L_2$-better predictor whose $L_1$ error is of order $\sqrt{\epsilon_{1,*}}$ rather than $\epsilon_{1,*}$. Appendix~\ref{app:squared-loss-approximation-lower} gives a sharper result (with proof) that, for any $\gamma\in(0,1)$, one can construct an instance satisfying $\epsilon_1(f_{\rm sq}^*)=(1-\gamma)\sqrt{\epsilon_{1,*}}$. Proposition~\ref{prop:squared-loss-no-constant-factor} is a specialization with $\gamma = 1/2$. 

The $n^{-1/2}$ statistical term in \eqref{eq:squared-loss-l1-upper} cannot be improved in general. Standard minimax lower bounds show an $\Omega(n^{-1/2})$ error for arbitrary estimators in simple bounded mean-estimation problems; see, e.g., \citet{tsybakov2009introduction,wainwright19high}. The limitation of the estimation error from squared-loss ERM is the failure to adapt to easier instances. In realizable full-information cost-sensitive classification, \citet[Theorem~2]{foster21efficient} show that least-squares plug-in prediction can incur $\Omega(n^{-1/2})$ excess risk even when the first-order rate is $O(1/n)$. Analogously, in realizable $[0,1]$-valued regression, \citet[Proposition~8]{pmlr-v336-li26a} show that squared-loss ERM can incur $\Omega(n^{-1/2})$ $L_1$ error even when the first-order rate is $O(1/n)$. 

Taken together, these results identify two distinct limitations of squared-loss ERM for our objective. Under misspecification, its $L_2$ approximation can be arbitrarily worse than the best $L_1$ approximation in multiplicative terms; even under realizability, squared-loss ERM does not adapt to favorable conditional variance.

\vspace{-1em}

\subsection{Log-loss ERM}
\label{subsec:log-loss-erm}

\vspace{-0.5em}

Another natural choice for conditional-mean estimation is log loss, which is widely used for learning probabilistic predictions. Recall the log-loss ERM
\begin{equation}
    \widehat f_{\log}
    \in
    \argmin_{f\in\cF}
    \frac1n\sum_{i=1}^n
    \ell_{\log}\bigl(Y_i,f(X_i)\bigr),
    \qquad
    \ell_{\log}(y,q)
    :=
    -y\log q-(1-y)\log(1-q).
    \label{eq:log-loss-erm}
\end{equation}
For any $Y\in[0,1]$, conditioning on $X=x$ gives
\[
    \EE[\ell_{\log}(Y,q)\mid X=x]
    =
    -\eta(x)\log q-(1-\eta(x))\log(1-q),
\]
which is uniquely minimized at $q=\eta(x)$. Thus, like squared loss, log loss is calibrated to the conditional mean when optimization is unrestricted.

Log loss is particularly relevant to our comparison because, under realizability, it can exhibit favorable first-order behavior that squared loss fails to achieve. In the cost-sensitive classification setting, \citet{foster21efficient} show that log-loss regression yields a first-order guarantee. More directly for our regression setting, \citet{pmlr-v336-li26a} establish a first-order guarantee for log-loss ERM measured in the same $L_1$ prediction error considered in this paper. 
This motivates asking whether it also avoids the other deficiency of squared loss under misspecification: the mismatch between its population objective and the best $L_1$ approximation in $\cF$. We show that it does not. Even when all predictions are uniformly bounded away from $0$ and $1$, log-loss ERM can incur $L_1$ error of order $\sqrt{\epsilon_{1,*}}$, and therefore cannot achieve a universal constant-factor $L_1$ oracle guarantee.

For $p\in[0,1]$ and $q\in(0,1)$, let
$
    \kl(p,q)
    :=
    p\log\frac{p}{q}
    +(1-p)\log\frac{1-p}{1-q},
$
with the convention $0\log 0=0$. At the population level  (see Appendix~\ref{app:log-loss-population-kl} for the derivation),
\begin{equation}
    \EE\!\left[
        \ell_{\log}(Y,f(X))
        -
        \ell_{\log}(Y,\eta(X))
    \right]
    =
    \EE\!\left[
        \kl\bigl(\eta(X),f(X)\bigr)
    \right].
    \label{eq:log-loss-population-kl}
\end{equation}
Thus, under misspecification, population log-loss minimization selects a KL-optimal approximation to $\eta$, which may not coincide with the best $L_1$ approximation.

Because log loss becomes increasingly sensitive as predictions approach $0$ or $1$, an $L_1$ guarantee in terms of $\epsilon_{1,*}$ requires explicit control of the prediction range. The following proposition gives the resulting guarantee under a restricted boundary condition.

\begin{proposition}[$L_1$ guarantee for log-loss ERM]
\label{prop:log-loss-upper}
Suppose $Y\in[0,1]$, $\cF$ is finite, and, for some $\gamma\in(0,1/2)$,
$
    \cF\subseteq[\gamma,1-\gamma]^{\cX}.
$
Then, with probability at least $1-\delta$,
\begin{equation}
    \epsilon_1(\widehat f_{\log})
    \le
    \sqrt{
        \frac{\log(1/\gamma)}{1-\gamma}
        \,\epsilon_{1,*}
    }
    +
    \sqrt{
        \left(
            2+\frac23\log\frac1\gamma
        \right)
        \frac{\log(2|\cF|^2/\delta)}{n}
    }.
    \label{eq:log-loss-l1-upper}
\end{equation}
In particular, for every fixed $\gamma$,
$\epsilon_1(\widehat f_{\log})
\lesssim_\gamma
\sqrt{\epsilon_{1,*}}
+
\sqrt{\log(|\cF|/\delta)/n}$.
\end{proposition}

The proof is deferred to Appendix~\ref{app:log-loss-upper}. 
Proposition~\ref{prop:log-loss-upper} shows that, even under a restricted boundary condition, the natural $L_1$ guarantee for log-loss ERM scales as $\sqrt{\epsilon_{1,*}}$ rather than linearly with $\epsilon_{1,*}$. As with squared loss, one might wonder whether this $\sqrt{\epsilon_{1,*}}$ dependence is merely an artifact of the analysis. The next proposition shows that it is unavoidable: even for a two-element class uniformly bounded away from $0$ and $1$, the population log-loss minimizer can have $L_1$ error of order $\sqrt{\epsilon_{1,*}}$.

\begin{proposition}[Log loss has no constant $L_1$ approximation factor]
\label{prop:log-loss-no-constant-factor}
For every $p\in(0,1)$, there exists a noiseless regression problem and a two-element class $\cF\subseteq[1/4,3/4]^{\cX}$ such that $\epsilon_{1,*}=p/4$, while the unique population log-loss minimizer
$
    f_{\log}^*
    \in
    \argmin_{f\in\cF}
    \EE[\ell_{\log}(Y,f(X))]
$
satisfies
\begin{equation}
    \epsilon_1(f_{\log}^*)
    =
    \frac12\sqrt{\epsilon_{1,*}}.
    \label{eq:log-loss-approximation-lower}
\end{equation}
Consequently, log-loss ERM cannot satisfy the basic goal \eqref{eq:basic-goal} for any finite universal approximation factor, even when all predictions are uniformly bounded away from $0$ and $1$.
\end{proposition}

Proposition~\ref{prop:log-loss-no-constant-factor} shows that the $\sqrt{\epsilon_{1,*}}$ dependence in Proposition~\ref{prop:log-loss-upper} is unavoidable up to a universal constant. Even when the hypothesis class is uniformly bounded away from $0$ and $1$, the geometry of KL divergence can make the population log-loss minimizer have $L_1$ error of order $\sqrt{\epsilon_{1,*}}$, even though $\cF$ contains another hypothesis with $L_1$ error $\epsilon_{1,*}$. Thus, as with squared loss, log-loss ERM cannot in general achieve linear dependence on $\epsilon_{1,*}$ or a finite universal $L_1$ approximation factor.

The restriction $\cF\subseteq[\gamma,1-\gamma]^{\cX}$ is essential for Proposition~\ref{prop:log-loss-upper}. Without requiring predictions to stay uniformly away from $0$ and $1$, log-loss ERM can fail much more severely: its population minimizer may have $L_1$ error bounded away from zero even as $\epsilon_{1,*}\to0$. We give this stronger boundary-sensitive obstruction in Appendix~\ref{app:log-loss-boundary}. Thus, even in the more favorable setting where all hypotheses are bounded away from $0$ and $1$, the worst-case approximation dependence of log-loss ERM remains of order $\sqrt{\epsilon_{1,*}}$, matching the qualitative limitation of squared-loss ERM.

\vspace{-1em}

\section{Sign betting loss and variance-adaptive bounds}
\label{sec:sign-betting}
\vspace{-0.5em}

\subsection{Sign betting for misspecified regression}
\label{sec:sign-betting-regression}

Motivated by the limitations of standard regression objectives established in Section~\ref{sec:standard-losses}, we turn to the min--max pairwise estimator
introduced by \citet{baraud2025estimating}. We show that this procedure, which
we call \emph{Sign-Betting Regression (SBR)}, aligns more with the best $L_1$
approximation to the conditional mean within $\cF$ and yields strong oracle
guarantees under misspecification.
Given $n$ i.i.d.\ samples $(X_i,Y_i)_{i=1}^n$, for $f,h\in\cF$, define
\begin{equation}
    \widehat T(f,h)
    :=
    \frac{1}{n}\sum_{i=1}^n
    \bigl(Y_i-f(X_i)\bigr)
    \sign\bigl(h(X_i)-f(X_i)\bigr),
    \label{eq:empirical-sign-betting-statistic}
\end{equation}
where $\sign(0)=0$. Our \emph{Sign-Betting Regression (SBR)} algorithm returns
\begin{equation}
    \hf
    \in
    \argmin_{f\in\cF}
    \sup_{h\in\cF}\widehat T(f,h).
    \label{eq:sign-betting-estimator}
\end{equation}
For the finite classes considered in our main results, the supremum above is a maximum.

The terminology ``sign betting'' comes from viewing $h$ as a challenger that places a directional bet against the residual of $f$. At an input $x$, the challenger bets $+1$ when $h(x)>f(x)$ and $-1$ when $h(x)<f(x)$. Its payoff is $(Y-f(X))\sign(h(X)-f(X))$. Thus, if $h$ lies above $f$, the challenger benefits when the observed response is also above $f$; if $h$ lies below $f$, it benefits when the observed response is below $f$. A large value of $\widehat T(f,h)$ therefore indicates that the direction from $f$ toward $h$ is systematically supported by the residuals.

The inner supremum in \eqref{eq:sign-betting-estimator} searches for the challenger $h\in\cF$ that achieves the largest empirical sign-betting value $\widehat T(f,h)$ against a given candidate $f$. The outer minimization then selects a predictor whose residuals are difficult to exploit by any $h\in\cF$. Thus, unlike ordinary empirical risk minimization with a pointwise loss $\ell(Y,f(X))$, the objective associated with $f$ depends on how $f$ compares with the other predictors in $\cF$.

We next study the population geometry of sign betting. The key identity shows that, for any $f,h\in\cF$, the two betting quantities associated with $(f,h)$ and $(h,f)$ sum exactly to $\normz{f-h}_1$. This identity will allow us to translate control of the sign-betting objective into control of $L_1$ error.

\vspace{-.5em}

\subsection{Population geometry of sign betting}
\label{sec:sign-betting-population}

\vspace{-.5em}

We first study the population behavior of sign betting. We will show that, for any pair $f,h\in\cF$, the two population quantities $T_P(f,h)$ and $T_P(h,f)$ sum exactly to the $L_1(P_X)$ distance between $f$ and $h$. This identity leads to a population performance guarantee for sign betting and provides the basic mechanism for the high-probability guarantees developed in the next subsection. For
$f,h\in\cF$, define
\begin{equation}
    T_P(f,h)
    :=
    \EE\!\left[
        (Y-f(X))\sign(h(X)-f(X))
    \right]
    = \EE\!\left[ (\eta(X)-f(X))\sign(h(X)-f(X)) \right].
    \label{eq:population-sign-betting-statistic}
\end{equation}
The second equality follows by conditioning on $X$ and using $\EE[Y\mid X]=\eta(X)$. 
We also define the population sign-betting objective
\begin{equation}
    B_P(f)
    :=
    \sup_{h\in\cF}T_P(f,h).
    \label{eq:population-sign-betting-objective}
\end{equation}

The basic geometry of $T_P(\cdot,\cdot)$ follows from a pairwise identity: 
\begin{lemma}[Pairwise identity for sign betting]
\label{lem:sign-betting-pairwise-identity}
For any $f,h\in\cF$,
\begin{equation}
    T_P(f,h)+T_P(h,f)
    =
    \normz{f-h}_1.
    \label{eq:sign-betting-pairwise-identity}
\end{equation}
\end{lemma}

\vspace{-1em}

\begin{proof}
Using $\sign(f-h)=-\sign(h-f)$,
\begin{align*}
    T_P(f,h)+T_P(h,f)
    &=
    \EE\!\left[
        (\eta-f)\sign(h-f)
        +
        (\eta-h)\sign(f-h)
    \right] \\
    &=
    \EE\!\left[
        (h-f)\sign(h-f)
    \right] \\
    &=
    \EE|h-f|
    =
    \normz{f-h}_1.
\end{align*}
\end{proof}

\vspace{-1em}

The population objective $B_P(f)$ is also directly controlled by the $L_1$ approximation error:
\begin{lemma}[Population sign-betting objective]
\label{lem:population-sign-betting-upper-bound}
For every $f\in\cF$,
\begin{equation}
    0
    \le
    B_P(f)
    \le
    \epsilon_1(f).
    \label{eq:population-sign-betting-upper-bound}
\end{equation}
\end{lemma}

\vspace{-1em}

\begin{proof}
Since $h=f$ is an admissible choice in the definition of $B_P(f)$ in \eqref{eq:population-sign-betting-objective} and
$T_P(f,f)=0$, we have $B_P(f)\ge 0$. Moreover, for every $h\in\cF$,
\[
    T_P(f,h)
    =
    \EE\!\left[
        (\eta(X)-f(X))\sign(h(X)-f(X))
    \right]
    \le
    \EE|\eta(X)-f(X)|
    =
    \epsilon_1(f).
\]
Taking the supremum over $h\in\cF$ yields $B_P(f)\le\epsilon_1(f)$.
\end{proof}

\vspace{-1em}

Lemma~\ref{lem:sign-betting-pairwise-identity} and Lemma~\ref{lem:population-sign-betting-upper-bound} yield a population performance guarantee for sign betting: 
\begin{proposition}[Population guarantee for sign betting]
\label{prop:population-sign-betting-guarantee}
Suppose the minimum of $B_P$ over $\cF$ is attained, and let
$f_B\in\argmin_{f\in\cF} B_P(f)$. Then, for every $f_0\in\cF$,
\begin{equation}
    \epsilon_1(f_B)
    \le
    \epsilon_1(f_0)+2B_P(f_0).
    \label{eq:population-sign-betting-guarantee}
\end{equation}
Consequently, $\epsilon_1(f_B)
    \le
    3\epsilon_1(f_0)$. 
\end{proposition}

\vspace{-1em}

\begin{proof}
By Lemma~\ref{lem:sign-betting-pairwise-identity},
\begin{align*}
    \normz{f_B-f_0}_1
    &=
    T_P(f_B,f_0)+T_P(f_0,f_B) 
    \le
    B_P(f_B)+B_P(f_0) 
    \le
    2B_P(f_0),
\end{align*}
where the last inequality follows from the optimality of $f_B$. Therefore,
\[
    \epsilon_1(f_B)
    \le
    \epsilon_1(f_0)+\normz{f_B-f_0}_1
    \le
    \epsilon_1(f_0)+2B_P(f_0).
\]
The second claim follows from
$B_P(f_0)\le\epsilon_1(f_0)$ by
Lemma~\ref{lem:population-sign-betting-upper-bound}.
\end{proof}

\vspace{-1em}

Equation~\eqref{eq:population-sign-betting-guarantee} is more informative than the worst-case factor-three bound: the additional approximation penalty is governed by the population sign-betting value $B_P(f_0)$, which can be much smaller than $\epsilon_1(f_0)$. 
We next show that the empirical Sign-Betting Regression algorithm inherits this population
guarantee, up to $f_0$-dependent estimation terms that decay with $n$ and adapt to the label variance.

\vspace{-1em}

\subsection{Main variance-adaptive guarantee}
\label{sec:sign-betting-main-theorem}

\vspace{-.5em}

We now turn to the finite-sample performance of SBR. As we will see, the leading estimation
term is controlled by an $f_0$-dependent residual second moment. For
$f_0\in\cF$, define
\begin{align}
    v(f_0)
    &:=
    \sup_{h\in\cF}
    \EE\!\left[
        (Y-f_0(X))^2
        \mathbf{1}\{h(X)\neq f_0(X)\}
    \right] \nonumber\\
    &=
    \sup_{h\in\cF}
    \EE\!\left[
        \left(
            \sigma^2(X)+(\eta(X)-f_0(X))^2
        \right)
        \mathbf{1}\{h(X)\neq f_0(X)\}
    \right].
    \label{eq:comparator-residual-variance}
\end{align}

The second equality follows from the noise--approximation decomposition in Section~\ref{sec:setting}. In particular, $v(f_0)\le \bar\sigma^2+\epsilon_2(f_0)^2$. Thus, under misspecification, $v(f_0)$ captures both the conditional label variance and the squared approximation residual, restricted to regions where $f_0$ disagrees with other predictors in $\cF$. This localization can make $v(f_0)$
smaller than the corresponding global residual second moment
$\bar\sigma^2+\epsilon_2(f_0)^2$.
Our main finite-class guarantee is the following.

\begin{theorem}[Variance-adaptive guarantee for SBR]
\label{thm:sbr-variance-adaptive}
Suppose $\cF$ is finite, and let $\hf$ be the output of SBR in
\eqref{eq:sign-betting-estimator}. 
For $\rho>0$, let
$C_\rho
    :=
    \frac{13}{3}
    +\frac{1}{\rho}
    +\frac{23}{12}\rho$. 
Then, for every $\delta\in(0,1)$, with probability at least $1-\delta$,
the following holds simultaneously for every $\rho>0$:
\begin{equation}
    \epsilon_1(\hf)
    \le
    \inf_{f_0\in\cF}
    \left\{
        \epsilon_1(f_0)
        +(2+\rho)B_P(f_0)
        +(2+\rho)
        \sqrt{
            \frac{2v(f_0)\log(2|\cF|^2/\delta)}{n}
        }
    \right\}
    +
    C_\rho
    \frac{\log(2|\cF|^2/\delta)}{n}.
    \label{eq:main-sbr-bound}
\end{equation}
\end{theorem}
The parameter $\rho$ is used only in the analysis and is not an input to SBR.
Since the event in Theorem~\ref{thm:sbr-variance-adaptive} holds
simultaneously for all $\rho>0$, $\rho$ can be chosen as a function of $f_0$ and the other quantities when interpreting the bound. In particular, the theorem implies that, for every $f_0\in\cF$,
\begin{equation} 
\epsilon_1(\hf) \le \epsilon_1(f_0)+2B_P(f_0)+r_n(f_0,\delta), 
\label{eq:sbr-population-plus-remainder} 
\end{equation} 
where $r_n(f_0,\delta)\to0$ as $n\to\infty$ with $P$, $f_0$, $\cF$, and $\delta$ fixed. Thus, the finite-sample guarantee recovers the population bound in Proposition~\ref{prop:population-sign-betting-guarantee} up to a vanishing estimation term. An explicit form of $r_n(f_0,\delta)$ is given in Corollary~\ref{cor:optimized-sbr-bound} in the appendix.

The full proof of Theorem~\ref{thm:sbr-variance-adaptive} is given in Appendix~\ref{app:proof-sbr-variance-adaptive}. At a high level, its proof follows the ingredients developed in the previous subsection.
To elaborate, the pairwise identity in
Lemma~\ref{lem:sign-betting-pairwise-identity} continues to hold exactly on
the sample:
\[
    \widehat T(f,h)+\widehat T(h,f)
    =
    P_n|f-h|.
\]
Hence, mirroring the proof of Proposition~\ref{prop:population-sign-betting-guarantee}, for any $f_0\in\cF$, the empirical optimality of $\hf$
implies
\begin{equation}
    P_n|\hf-f_0|
    \le
    2\sup_{h\in\cF}\widehat T(f_0,h).
    \label{eq:empirical-distance-key}
\end{equation}
Bernstein concentration then upper bounds the right-hand side in terms of $B_P(f_0)$ and $v(f_0)$, and a further concentration step converts the empirical distance on the left-hand side in
\eqref{eq:empirical-distance-key} into $\|\hf-f_0\|_1$. 

\vspace{-.5em}

\subsection{Consequences of the main guarantee}
\label{sec:sbr-consequences}

\vspace{-.5em}

We now specialize Theorem~\ref{thm:sbr-variance-adaptive} to three regimes:
the general agnostic setting, realizability, and favorable misspecification.
Each specialization represents a different aspect of the guarantee.

\vspace{-.5em}

\paragraph{General agnostic guarantee.}
We first consider the general agnostic setting, where we impose no additional
relationship between $\eta$ and $\cF$ or any additional condition on the
$L_1$-optimal predictor
$f_\star\in\argmin_{f\in\cF}\epsilon_1(f)$, beyond the basic assumptions of
Section~\ref{sec:setting}. By
Lemma~\ref{lem:population-sign-betting-upper-bound},
$
B_P(f_\star)
\le
\epsilon_1(f_\star)
=
\epsilon_{1,*}.
$
Combining this bound with \eqref{eq:sbr-population-plus-remainder} gives the
following guarantee.

\begin{corollary}[General agnostic guarantee]
\label{cor:sbr-agnostic}
Under the assumptions of Theorem~\ref{thm:sbr-variance-adaptive}, with
probability at least $1-\delta$,
\begin{equation}
    \epsilon_1(\hf)
    \le
    3\epsilon_{1,*}
    +
    r_n(f_\star,\delta),
    \label{eq:sbr-agnostic}
\end{equation}
where $r_n(f_\star,\delta)$ is given explicitly in
Corollary~\ref{cor:optimized-sbr-bound} and satisfies
$r_n(f_\star,\delta)\to0$ as $n\to\infty$ with $P$, $\cF$, and $\delta$
fixed.
\end{corollary}

Thus, without assuming realizability or any additional property of the
$L_1$-optimal predictor, SBR achieves the basic goal
\eqref{eq:basic-goal} with approximation factor $3$. The estimation term is
$f_\star$-dependent and retains variance adaptation through the residual
second moment $v(f_\star)$. The factor $3$ in
\eqref{eq:sbr-agnostic} is obtained by optimizing the analysis parameter
$\rho$ in Theorem~\ref{thm:sbr-variance-adaptive}, as detailed in
Corollary~\ref{cor:optimized-sbr-bound}. We return to the optimality of the
factor $3$ in Section~\ref{sec:optimality}.

\vspace{-.5em}

\paragraph{Realizable case.}
We next consider the realizable setting, where $\eta\in\cF$. Taking
$f_0=\eta$ in Theorem~\ref{thm:sbr-variance-adaptive} gives
$\epsilon_1(\eta)=0$ and $B_P(\eta)=0$. Moreover, by
\eqref{eq:comparator-residual-variance},
\[
    v(\eta)
    =
    \sup_{h\in\cF}
    \EE\!\left[
        \sigma^2(X)\mathbf{1}\{h(X)\neq\eta(X)\}
    \right]
    \le
    \bar\sigma^2,
\]
where $\bar\sigma^2=\EE[\sigma^2(X)]$ is the average conditional label
variance. Thus, under realizability, the approximation terms vanish and the
estimation error adapts directly to the conditional label variance.

\begin{corollary}[Realizable guarantee]
\label{cor:sbr-realizable}
Under the assumptions of Theorem~\ref{thm:sbr-variance-adaptive}, suppose
that $(P,\cF)$ is realizable. Then, for every $\delta\in(0,1)$, with
probability at least $1-\delta$, simultaneously for every $\rho>0$,
\begin{equation}
\begin{aligned}
    \epsilon_1(\hf)
    &\le
    (2+\rho)\sqrt{\frac{2v(\eta)\log(2|\cF|^2/\delta)}{n}}
    +C_\rho\frac{\log(2|\cF|^2/\delta)}{n} \\
    &\le
    (2+\rho)\sqrt{\frac{2\bar\sigma^2\log(2|\cF|^2/\delta)}{n}}
    +C_\rho\frac{\log(2|\cF|^2/\delta)}{n}.
    \label{eq:sbr-realizable}
\end{aligned}
\end{equation}
\end{corollary}

The second bound in \eqref{eq:sbr-realizable} recovers, up to constants, the
realizable variance-adaptive rate established by
\citet{pmlr-v336-li26a}, where the estimation error for their betting loss
scales as
$\sqrt{\bar\sigma^2\log(|\cF|/\delta)/n}
+\log(|\cF|/\delta)/n$.
Moreover, the localized quantity $v(\eta)$ can be smaller than the global
average conditional variance $\bar\sigma^2$, since it includes label
variance only on regions where some predictor in $\cF$ disagrees with
$\eta$. Thus, our Sign-Betting Regression retains the favorable noise adaptation available under
realizability while extending the guarantee to misspecified classes.

\vspace{-.5em}

\paragraph{Favorable misspecification.}
Finally, because the approximation term in
Theorem~\ref{thm:sbr-variance-adaptive} depends on $B_P(f_0)$ rather than only on $\epsilon_1(f_0)$, the guarantee can be substantially sharper on favorable misspecified instances. Suppose that an $L_1$-optimal predictor
$f_\star\in\argmin_{f\in\cF}\epsilon_1(f)$ satisfies
$B_P(f_\star)=0$. Taking $f_0=f_\star$ gives the following result.

\begin{corollary}[Favorable misspecification]
\label{cor:sbr-favorable-misspecification}
Under the assumptions of Theorem~\ref{thm:sbr-variance-adaptive}, suppose
that $f_\star\in\argmin_{f\in\cF}\epsilon_1(f)$ satisfies
$B_P(f_\star)=0$. Then, for every $\delta\in(0,1)$, with probability at least
$1-\delta$, simultaneously for every $\rho>0$,
\begin{equation}
    \epsilon_1(\hf)
    \le
    \epsilon_{1,*}
    +(2+\rho)
    \sqrt{
        \frac{2v(f_\star)\log(2|\cF|^2/\delta)}{n}
    }
    +C_\rho\frac{\log(2|\cF|^2/\delta)}{n}.
    \label{eq:sbr-favorable-misspecification}
\end{equation}
\end{corollary}

Thus, even when $\epsilon_{1,*}>0$, SBR achieves approximation factor $1$ on such misspecified instances while retaining a variance-adaptive estimation term. We give a sufficient condition under which an $L_1$-optimal predictor satisfies $B_P(f_\star)=0$ under genuine misspecification in Appendix~\ref{app:favorable-misspecification-condition}.

\vspace{-1em}

\section{Optimality of the universal proper approximation factor}
\label{sec:optimality}
\vspace{-0.5em}

Corollary~\ref{cor:sbr-agnostic} shows that SBR achieves approximation
factor $3$ while retaining the variance-adaptive, comparator-dependent
estimation term inherited from Theorem~\ref{thm:sbr-variance-adaptive}.
To characterize whether this factor can be improved by a different proper
learner, it is useful to first record a coarser consequence in which the
estimation term is uniform over the underlying distribution.
This places the upper bound in the same distribution-free framework as the
lower bound developed below: for any target accuracy $\varepsilon>0$, the
sample size required to guarantee
$\epsilon_1(\hf)
    \le
    3\epsilon_{1,*}+\varepsilon$
depends only on $|\cF|$, $\varepsilon$, and the confidence level $\delta$, and not on the underlying distribution $P$. We then show that no proper learner can replace
the factor $3$ by any universal constant strictly smaller than $3$.

\begin{corollary}[Uniform achievability of approximation factor three]
\label{cor:sbr-uniform-factor-three}
Under the assumptions of
Theorem~\ref{thm:sbr-variance-adaptive}, there exists a universal constant
$C>0$ such that, for every $\delta\in(0,1)$, with probability at least
$1-\delta$,
\begin{equation}
    \epsilon_1(\hf)
    \le
    3\epsilon_{1,*}
    +
    C\sqrt{
            \frac{\log(2|\cF|^2/\delta)}{n}
        }.
    \label{eq:sbr-uniform-factor-three}
\end{equation}
\end{corollary}
We defer the proof to Appendix~\ref{app:uniform-factor-three}. Note that, unlike Corollary~\ref{cor:sbr-agnostic}, the estimation term in \eqref{eq:sbr-uniform-factor-three} depends only on $n$, $|\cF|$, and $\delta$, and is therefore distribution-independent. For fixed $\cF$ and $\delta$, it vanishes as $n\to\infty$ uniformly over all distributions $P$. 
Corollary~\ref{cor:sbr-uniform-factor-three} shows that approximation factor $3$ is achievable
by a proper learner with a distribution-independent estimation term. We
next show that this factor is sharp: no proper learning algorithm
can achieve a universal approximation factor strictly smaller than $3$,
even if the estimation term is allowed to vanish arbitrarily slowly with
the sample size.

\begin{theorem}[Sharp lower bound for proper learning]
\label{thm:proper-factor-lower-bound}
Fix any $C\in[1,3)$. There exist a measurable input space $\cX$, a fixed
two-element class
    $\cF_C\subseteq\{f:\cX\to[0,1]\}$,
independent of the sample size, and a constant $\varepsilon_C>0$ such that
the following holds. For every $n\ge1$ and every possibly randomized proper
learner $\widehat f_n\in\cF_C$ based on $n$ i.i.d.\ observations, there
exists a distribution $P$ on $\cX\times\{0,1\}$, with finitely supported
marginal $P_X$, such that
\begin{equation}
    \PP\left(
        \epsilon_1(\widehat f_n)
        >
        C\epsilon_{1,*}
        +
        \varepsilon_C
    \right)
    \ge
    \frac38.
    \label{eq:proper-factor-lower-bound}
\end{equation}
The probability in \eqref{eq:proper-factor-lower-bound} includes any
internal randomization of the learner.
\end{theorem}

Theorem~\ref{thm:proper-factor-lower-bound} is a lower bound on proper
learning itself, rather than on SBR or on any particular proper learning algorithm. It rules out any uniform high-probability guarantee of the form
\begin{equation}
    \epsilon_1(\widehat f_n)
    \le
    C\epsilon_{1,*}
    +
    r_n(\cF,\delta),
    \qquad C\in[0,3),
    \label{eq:sub-three-uniform-guarantee}
\end{equation}
where the estimation term $r_n(\cF,\delta)$ is independent of $P$ and
satisfies $r_n(\cF,\delta)\to0$ as $n\to\infty$
for every fixed $\cF$ and $\delta$. Notably, the impossibility holds regardless
of how slowly the estimation term $r_n(\cF,\delta)$ vanishes with $n$.
Together with
Corollary~\ref{cor:sbr-uniform-factor-three}, this shows that $3$ is the
smallest universal approximation factor achievable by proper learners under
the distribution-free, high-probability criterion considered here.
We defer the proof of
Theorem~\ref{thm:proper-factor-lower-bound}, together with the formal
derivation of the impossibility of
\eqref{eq:sub-three-uniform-guarantee}, to
Appendix~\ref{app:proper-factor-lower-bound} and
Appendix~\ref{app:sub-three-uniform-guarantees}, respectively.

\vspace{-.5em}
\section{Beyond finite classes}
\label{sec:sbr-infinite-classes}
\vspace{-0.5em}

The variance-adaptive guarantee of SBR extends beyond finite hypothesis
classes. The extension is governed by the
pseudo-dimension of the hypothesis class $\cF$ and the VC dimension of its induced
pairwise comparison class. Define
$\cC_{\cF}:=\{x\mapsto \one\{h(x)>f(x)\}:f,h\in\cF\}$,
and let
$d_{\rm pred}:=\operatorname{Pdim}(\cF)$ be the pseudo-dimension of $\cF$,
$d_{\rm cmp}:=\operatorname{VCdim}(\cC_{\cF})$ be the VC dimension of $\cC_{\cF}$.

\begin{theorem}[Variance-adaptive guarantee beyond finite classes]
\label{thm:sbr-infinite-classes}
Let $\cF$ be a class of measurable functions from $\cX$ to $[0,1]$,
suppose $d_{\rm pred},d_{\rm cmp}<\infty$, and suppose the minimum in
\eqref{eq:sign-betting-estimator} is attained. Write
$d:=1+d_{\rm pred}+d_{\rm cmp}$. For $\rho>0$, let $C_\rho$ be as in
Theorem~\ref{thm:sbr-variance-adaptive}. There exists a universal
constant $c>0$ such that, for every $\delta\in(0,1)$, letting
$\alpha_\delta:=c\,\frac{d\log(en)+\log(2/\delta)}{n}$, with probability
at least $1-\delta$, the following holds simultaneously for every
$\rho>0$:
\begin{equation}
    \epsilon_1(\hf)
    \le
    \inf_{f_0\in\cF}
    \left\{
        \epsilon_1(f_0)
        +(2+\rho)B_P(f_0)
        +(2+\rho)\sqrt{2v(f_0)\alpha_\delta}
    \right\}
    +C_\rho\alpha_\delta.
    \label{eq:sbr-infinite-bound}
\end{equation}
\end{theorem}

Thus, the same approximation terms and comparator-dependent variance adaptation of Theorem~\ref{thm:sbr-variance-adaptive} continue to hold for infinite classes satisfying these complexity conditions, with the finite-class log-cardinality dependence replaced by a complexity term $d$, which captures both prediction and comparison complexity of the hypothesis class $\cF$. Appendix~\ref{app:sbr-infinite-classes} provides the proof and verifies the required complexity bounds for bounded linear, logistic, and sparse prediction classes as examples.

\vspace{-.5em}

\section{Conclusion}

\vspace{-.5em}

We studied conditional-mean regression under model misspecification with population $L_1$ error as the target metric. We showed that several standard regression objectives do not, in general, provide a constant-factor $L_1$ oracle guarantee under misspecification. In contrast, the min--max pairwise estimator of \citet{baraud2025estimating}, which we call \emph{Sign-Betting Regression}, admits high-probability $L_1$ oracle inequalities with comparator-dependent estimation terms. Under realizability, these bounds recover the variance-adaptive rates established by \citet{pmlr-v336-li26a}. We further showed that the resulting approximation factor $3$ is optimal for proper learning: no proper learner can achieve any universal approximation factor strictly below $3$ in a distribution-free high-probability $L_1$ oracle inequality with a vanishing estimation term. We also extended the upper bound beyond finite classes using prediction and pairwise-comparison complexity measures. Experiments in Appendix~\ref{app:exp} illustrate the behavior of Sign-Betting Regression in synthetic and real-data model-selection settings. Together, these results provide a sharp statistical characterization of proper $L_1$ conditional-mean learning under misspecification.

\clearpage
\bibliography{library-overleaf}
\bibliographystyle{iclr2027_conference}


\clearpage
\appendix
\addcontentsline{toc}{section}{Appendix} 
\part{Appendix}

\parttoc
\section{Downstream implications of $L_1$ conditional-mean accuracy}
\label{app:l1-downstream}

For completeness, we record several elementary consequences of controlling
the $L_1$ error
\[
    \epsilon_1(f)
    :=
    \EE\!\left[|f(X)-\eta(X)|\right],
    \qquad
    \eta(x):=\EE[Y\mid X=x].
\]
These observations motivate our choice of performance criterion.

\paragraph{Binary classification.}
Suppose $Y\in\{0,1\}$, so that
$\eta(x)=\PP(Y=1\mid X=x)$. Let
\[
    g_f(x):=\mathbf{1}\{f(x)\ge 1/2\},
    \qquad
    g^\star(x):=\mathbf{1}\{\eta(x)\ge 1/2\}.
\]
Then
\begin{equation}
    R_{\mathrm{cls}}(g_f)-R_{\mathrm{cls}}(g^\star)
    \le
    2\epsilon_1(f),
    \label{eq:l1-plugin-classification}
\end{equation}
where $R_{\mathrm{cls}}(g):=\PP(g(X)\neq Y)$.
Indeed, the standard excess-risk identity gives
\[
    R_{\mathrm{cls}}(g_f)-R_{\mathrm{cls}}(g^\star)
    =
    \EE\!\left[
        |2\eta(X)-1|
        \mathbf{1}\{g_f(X)\neq g^\star(X)\}
    \right].
\]
Whenever $g_f(X)\neq g^\star(X)$, the values $f(X)$ and $\eta(X)$
lie on opposite sides of $1/2$, and hence
\[
    |\eta(X)-1/2|
    \le
    |\eta(X)-f(X)|.
\]
Therefore,
\[
    |2\eta(X)-1|
    \mathbf{1}\{g_f(X)\neq g^\star(X)\}
    \le
    2|\eta(X)-f(X)|,
\]
which proves \eqref{eq:l1-plugin-classification}; see also
\citet{audibert2007fast}.

\paragraph{Calibration and subpopulation prediction errors.}
Consider the population $L_1$ calibration error
\[
    \mathrm{CE}_1(f)
    :=
    \EE\!\left[
        \left|
            \EE[Y\mid f(X)]-f(X)
        \right|
    \right].
\]
Since
\[
    \EE[Y\mid f(X)]
    =
    \EE[\eta(X)\mid f(X)],
\]
conditional Jensen's inequality yields
\begin{align}
    \mathrm{CE}_1(f)
    &=
    \EE\!\left[
        \left|
            \EE[\eta(X)-f(X)\mid f(X)]
        \right|
    \right]
    \nonumber\\
    &\le
    \EE|\eta(X)-f(X)|
    =
    \epsilon_1(f).
    \label{eq:l1-calibration}
\end{align}

More generally, for every measurable subpopulation
$G\subseteq\mathcal X$,
\begin{align}
    \left|
        \EE\!\left[
            (Y-f(X))\mathbf{1}\{X\in G\}
        \right]
    \right|
    &=
    \left|
        \EE\!\left[
            (\eta(X)-f(X))\mathbf{1}\{X\in G\}
        \right]
    \right|
    \nonumber\\
    &\le
    \EE\!\left[
        |\eta(X)-f(X)|\mathbf{1}\{X\in G\}
    \right]
    \nonumber\\
    &\le
    \epsilon_1(f).
    \label{eq:l1-subpopulation}
\end{align}
Equivalently, whenever $\PP(X\in G)>0$,
\[
    \left|
        \EE[Y-f(X)\mid X\in G]
    \right|
    \le
    \frac{\epsilon_1(f)}{\PP(X\in G)}.
\]
Thus, small $L_1$ conditional-mean error controls both global calibration
error and aggregate residual errors over arbitrary subpopulations. These
properties are closely related to the objectives studied in multiaccuracy
and multicalibration
\citep{hebert2018multicalibration,kim2019multiaccuracy}.

\paragraph{Best-of-$N$ selection.}
We next consider a setting in which $f$ serves as a reward model for
selecting among a collection of candidate responses. Let
$X_1,\ldots,X_N$ be candidate responses, and assume that each $X_i$ has
marginal distribution $P_X$. Independence among the candidates is not
required. Define
\[
    i^\star\in\argmax_{i\in[N]}\eta(X_i),
    \qquad
    \widehat i\in\argmax_{i\in[N]}f(X_i),
\]
with arbitrary tie-breaking. Then
\begin{equation}
    \EE\!\left[
        \eta(X_{i^\star})-\eta(X_{\widehat i})
    \right]
    \le
    N\epsilon_1(f).
    \label{eq:l1-best-of-n}
\end{equation}

To see this, by the definition of $\widehat i$,
\[
    f(X_{i^\star})-f(X_{\widehat i})\le 0.
\]
Hence, samplewise,
\begin{align*}
    \eta(X_{i^\star})-\eta(X_{\widehat i})
    &=
    \eta(X_{i^\star})-f(X_{i^\star})
    +
    f(X_{i^\star})-f(X_{\widehat i})
    +
    f(X_{\widehat i})-\eta(X_{\widehat i})
    \\
    &\le
    |\eta(X_{i^\star})-f(X_{i^\star})|
    +
    |\eta(X_{\widehat i})-f(X_{\widehat i})|
    \\
    &\le
    \sum_{i=1}^N
    |\eta(X_i)-f(X_i)|.
\end{align*}
Taking expectations and using the common marginal distribution of the
$X_i$'s gives
\[
    \EE\!\left[
        \eta(X_{i^\star})-\eta(X_{\widehat i})
    \right]
    \le
    \sum_{i=1}^N
    \EE|\eta(X_i)-f(X_i)|
    =
    N\epsilon_1(f).
\]
Thus, $L_1$ accuracy of a reward model yields a direct guarantee for
Best-of-$N$ selection. The bound requires only control of the marginal
prediction error of each candidate and does not require independence among
the candidates. This provides one connection between the regression
criterion studied here and recent analyses of inference-time selection with
imperfect reward models \citep{huang2025best}.

\section{Related work}
\label{app:related-work}

\paragraph{Closest prior work: fixed-design sign-based regression.}
The closest work to ours is \citet{baraud2025estimating}, who introduce the
same minimax signed-residual estimator, up to normalization and sign
conventions, in a fixed-design regression setting. Their design points are
deterministic, while the errors are independent but may be non-identically
distributed, heteroscedastic, and heavy-tailed; only a finite moment of order
$p\in[1,2]$ is required. They establish nonasymptotic oracle inequalities for
the expected $L_1$ error evaluated on the observed design points. Their bounds
already accommodate model misspecification, incur an approximation factor
three, and adapt to averaged noise moments and to comparator-dependent
complexity of the associated comparison sets. 

Our statistical target and questions are different. We consider i.i.d.\
random design and study the population error
$\|f-\eta\|_{L_1(P_X)}$ rather than error restricted to the observed design
points. We establish distribution-free, high-probability population oracle
inequalities whose second-order terms retain both directional approximation
geometry and disagreement-localized residual moments. In addition, we study
the information-theoretic approximation factor achievable by arbitrary
proper learners and prove that no proper learner can universally improve on
factor three under our high-probability learning criterion.
Appendix~\ref{app:comparison-baraud-maillard} gives a more detailed
comparison. In particular, conditioning on the random covariates and applying
the fixed-design result as a black box, followed by ordinary uniform
convergence, yields a population bound with an additional
variance-independent $n^{-1/2}$ generalization term. Our direct random-design
analysis avoids this crude global penalty through a localized treatment of
the empirical-to-population transition. 

\paragraph{Conditional-mean estimation and regression under misspecification.}
Our target,
\[
    \epsilon_1(f)
    =
    \EE|f(X)-\eta(X)|,
    \qquad
    \eta(x)=\EE[Y\mid X=x],
\]
should be distinguished from minimizing an observable loss such as
$\EE|Y-f(X)|$. The distinction between a statistical functional and the loss
used to elicit it is classical: squared loss elicits the conditional mean,
whereas absolute loss elicits a conditional median
\citep{gneiting2011making}. More generally, losses that correctly identify
the conditional mean under unrestricted optimization can induce a different
approximation geometry once optimization is restricted to a misspecified
class. For example, squared-loss minimization projects $\eta$ onto the model
class in squared $L_2$ distance rather than in $L_1$ distance.

Finite-sample $L_1$ estimation of binary class probabilities through
empirical risk minimization with probability-estimation losses was studied by
\citet{mey2021consistency}, including a discussion of model
misspecification. Their perspective starts from a surrogate loss and studies
the probability estimate induced by minimizing that loss. Under
misspecification, the resulting population optimizer is therefore determined
by the geometry of the chosen surrogate. Our benchmark instead fixes the
target metric first and asks to compete directly with
\[
    \epsilon_{1,*}
    =
    \inf_{f\in\cF}\|\eta-f\|_{L_1(P_X)}.
\]
This distinction motivates our negative results for standard regression
objectives: a loss may be well suited for estimating the conditional mean
under realizability while failing to provide any universal constant-factor
approximation to the best $L_1$ predictor in a misspecified class.

\paragraph{First- and second-order regression guarantees.}
A related line of work studies statistical guarantees that improve when the
underlying prediction or decision problem is favorable. In cost-sensitive
classification and contextual bandits, \citet{foster21efficient} show that
log-loss regression and triangular discrimination can yield first-order
guarantees whose leading term adapts to the performance of the optimal
decision rule. For bounded regression, \citet{pmlr-v336-li26a} introduce a
betting loss and establish variance-adaptive, or second-order, guarantees,
improving on first-order dependence by adapting directly to label variance
without requiring that variance to be known in advance. The present work asks
how such favorable second-order behavior interacts with \emph{model
misspecification}. In particular, the relevant second-order quantity under
misspecification is no longer determined by label variance alone: our bound
also reflects the approximation residual and the regions on which competing
predictors disagree with a comparator. At the same time, we require linear
control of the best $L_1$ approximation error. As shown in
Section~\ref{sec:standard-losses}, these two goals need not be inherited by
regression objectives that have favorable guarantees in realizable settings.

Although we use the term \emph{sign betting}, SBR should also be
distinguished from the betting loss of \citet{pmlr-v336-li26a}. The two
objectives are different. SBR compares a candidate $f$ against competitors
$h$ through signed residual statistics of the form
\[
    (Y-f(X))\sign(h(X)-f(X)),
\]
and minimizes the worst such comparison. The resulting population identity
directly exposes the $L_1$ distance between competing predictors. This
pairwise comparison structure is what permits the misspecified oracle
analysis developed here.

\paragraph{Minimum-distance estimation, pairwise testing, and the factor three.}
The pairwise-comparison structure of SBR has close antecedents in
minimum-distance and testing-based estimation. In density estimation,
\citet{yatracos1985rates} developed minimum-distance procedures based on
sets that distinguish pairs of candidate densities, leading to
constant-factor $L_1$ or total-variation oracle guarantees. Related
testing-based ideas underlie more recent robust $L_1$ density estimators,
including \citet{baraud2025robust}, and form part of the methodological
lineage of the regression estimator of \citet{baraud2025estimating}. These
methods share the principle that, rather than minimizing the target loss
directly, one constructs pairwise tests or signed statistics whose
expectations reveal which candidate is closer to the target.

The appearance of approximation factor three also has a particularly close
parallel in density hypothesis selection. \citet{bousquet2019optimal} show
that, when an estimator is required to output one of the candidate densities,
factor three is the optimal universal approximation factor in total
variation, even for a class containing only two candidates; if improper
outputs are allowed, the optimal factor improves to two. There is a useful
metric correspondence between their problem and ours. For a fixed covariate
marginal $P_X$, associate a predictor $f$ with the joint distribution
\[
    Q_f(dx,dy)
    :=
    P_X(dx)\,\mathrm{Bernoulli}(f(x))(dy).
\]
Then
\[
    \operatorname{TV}(Q_f,Q_\eta)
    =
    \EE_{P_X}|f(X)-\eta(X)|
    =
    \epsilon_1(f).
\]
Thus the same factor-three geometry appears naturally in the two problems.

The statistical models are nevertheless different. Density hypothesis
selection assumes access to fully specified candidate distributions, whereas
our candidates specify conditional-mean functions and the covariate marginal
$P_X$ is itself unknown and sampled as part of the regression problem.
Moreover, our lower bound concerns a uniform, high-probability regression
oracle inequality with a distribution-independent vanishing remainder.
Accordingly, our proper-learning lower bound is proved directly rather than
by invoking the density-estimation result. The connection to
\citet{bousquet2019optimal} should therefore be viewed as identifying a
common proper-versus-improper approximation phenomenon rather than as an
equivalence of the two learning problems.

\paragraph{Calibration, subgroup accuracy, and downstream decision making.}
Finally, the $L_1$ conditional-mean criterion has connections to several
other prediction objectives. Small $\epsilon_1(f)$ controls aggregate
residuals over every measurable subpopulation, relating it to multiaccuracy
\citep{kim2019multiaccuracy} and multicalibration
\citep{hebert2018multicalibration}. For binary responses it also controls
the excess error of the plug-in classifier, while when $f$ serves as a reward
model it controls the regret of selecting the largest predicted reward among
a finite collection of candidates. The latter connects to recent work on the
statistical behavior of Best-of-$N$ and other inference-time selection
procedures under imperfect reward models \citep{huang2025best}. These
connections motivate the choice of $L_1(P_X)$ error as our primary
statistical target; their elementary derivations are collected in
Appendix~\ref{app:l1-downstream}.

\section{Proofs and Additional Results for Squared-Loss ERM}
\label{app:squared-loss}

\subsection{Proof of Proposition~\ref{prop:squared-loss-upper}}
\label{app:squared-loss-upper}
We note that a generic uniform-convergence analysis of squared loss followed by the conversion $\epsilon_1(f)\le \epsilon_2(f)$ yields only an $O(n^{-1/4})$ statistical term. To show $O(n^{-1/2})$ rate in Proposition~\ref{prop:squared-loss-upper}, our proof proceeds by comparing empirical squared risks pairwise and exploiting the fact that the variance of a squared-loss difference is controlled by the $L_2$ distance between the corresponding predictors.

\begin{proof}

For $f,g\in\cF$, define the pairwise loss difference $\Delta_{f,g}(X,Y):=(Y-f(X))^2-(Y-g(X))^2$. Since $Y,f(X),g(X)\in[0,1]$, we have $|\Delta_{f,g}|\le 1$, and
\[
    \Delta_{f,g}^2
    =
    (f-g)^2(f+g-2Y)^2
    \le
    4(f-g)^2.
\]
Hence $\Var(\Delta_{f,g})\le 4\normz{f-g}_2^2$. A one-sided Bernstein inequality and a union bound over all ordered pairs $(f,g)\in\cF^2$ imply that, with probability at least $1-\delta$, simultaneously for every $f,g\in\cF$,
\begin{equation}
    (P-P_n)\Delta_{f,g}
    \le
    2\sqrt{2}\,\normz{f-g}_2
    \sqrt{\frac{\log(2|\cF|^2/\delta)}{n}}
    +
    \frac{4}{3}\frac{\log(2|\cF|^2/\delta)}{n}.
    \label{eq:squared-loss-pairwise-concentration}
\end{equation}

Let $f_{\rm sq}^*\in\argmin_{f\in\cF}\normz{f-\eta}_2$, so that $\normz{f_{\rm sq}^*-\eta}_2=\epsilon_{2,*}$. By the conditional-mean decomposition in \eqref{eq:squared-loss-population-decomposition}, for every $f,g\in\cF$,
\[
    P\Delta_{f,g}
    =
    \normz{f-\eta}_2^2-\normz{g-\eta}_2^2.
\]
Moreover, empirical optimality of $\widehat f_{\rm sq}$ gives $P_n\Delta_{\widehat f_{\rm sq},f_{\rm sq}^*}\le0$. Therefore, on the event \eqref{eq:squared-loss-pairwise-concentration}, writing $r:=\normz{\widehat f_{\rm sq}-\eta}_2$ and $e:=\epsilon_{2,*}$, we obtain
\[
    r^2-e^2
    \le
    2\sqrt{2}(r+e)
    \sqrt{\frac{\log(2|\cF|^2/\delta)}{n}}
    +
    \frac{4}{3}\frac{\log(2|\cF|^2/\delta)}{n},
\]
where we used $\normz{\widehat f_{\rm sq}-f_{\rm sq}^*}_2\le r+e$.

If $r\le e$, the desired conclusion is immediate. Otherwise, let $q:=r-e>0$. Since $r^2-e^2=q(r+e)$ and $r+e\ge q$, the preceding inequality yields
\[
    q^2
    \le
    2\sqrt{2}\,q
    \sqrt{\frac{\log(2|\cF|^2/\delta)}{n}}
    +
    \frac{4}{3}\frac{\log(2|\cF|^2/\delta)}{n}.
\]
Solving this quadratic inequality gives
\[
    q
    \le
    \left(\sqrt{2}+\sqrt{\frac{10}{3}}\right)
    \sqrt{\frac{\log(2|\cF|^2/\delta)}{n}}.
\]
Consequently,
\[
    \normz{\widehat f_{\rm sq}-\eta}_2
    \le
    \epsilon_{2,*}
    +
    \left(\sqrt{2}+\sqrt{\frac{10}{3}}\right)
    \sqrt{\frac{\log(2|\cF|^2/\delta)}{n}}.
\]
Since $\epsilon_1(\widehat f_{\rm sq})\le\normz{\widehat f_{\rm sq}-\eta}_2$, this proves the first inequality in \eqref{eq:squared-loss-l1-upper}.

For the second inequality, let $f_1^*\in\argmin_{f\in\cF}\epsilon_1(f)$. Because $f_1^*$ and $\eta$ are $[0,1]$-valued, $|f_1^*-\eta|\le1$, and hence $\normz{f_1^*-\eta}_2^2\le\normz{f_1^*-\eta}_1=\epsilon_{1,*}$. Therefore $\epsilon_{2,*}\le\sqrt{\epsilon_{1,*}}$, which completes the proof.

\end{proof}

\subsection{A population lower bound for squared loss}
\label{app:squared-loss-approximation-lower}

The following generalizes Proposition~\ref{prop:squared-loss-no-constant-factor} and implies that the $\sqrt{\epsilon_{1,*}}$ dependence in Eqn. \eqref{eq:squared-loss-l1-upper} is unavoidable.

\begin{proposition}[Sharpness of the $\sqrt{\epsilon_{1,*}}$ dependence]
\label{prop:squared-loss-sharp-approximation-lower}
For every $\gamma\in(0,1)$ and every $p\in(0,(1-\gamma)^2)$, there exists a noiseless regression problem and a two-element class $\cF\subseteq[0,1]^{\cX}$ such that $\epsilon_{1,*}=p$, while the unique population squared-loss minimizer $f_{\rm sq}^*$ satisfies
\begin{equation}
    \epsilon_1(f_{\rm sq}^*)
    =
    (1-\gamma)\sqrt{\epsilon_{1,*}}.
    \label{eq:squared-loss-sharp-approximation-lower}
\end{equation}
Moreover, squared-loss ERM selects $f_{\rm sq}^*$ eventually almost surely as $n\to\infty$.
\end{proposition}

\begin{proof}
Let $\cX=\{a,b\}$ with $\PP(X=b)=p$ and $\PP(X=a)=1-p$, and let $Y=0$ almost surely. Thus $\eta\equiv0$. Consider the two predictors $f_{\rm spike},f_{\rm flat}\in[0,1]^{\cX}$ defined by $f_{\rm spike}(a)=0$, $f_{\rm spike}(b)=1$, and $f_{\rm flat}(x)=(1-\gamma)\sqrt p$ for both $x\in\cX$. The latter is $[0,1]$-valued because $p<1$ and $\gamma\in(0,1)$.

Their $L_1$ approximation errors are
\[
    \epsilon_1(f_{\rm spike})=p,
    \qquad
    \epsilon_1(f_{\rm flat})=(1-\gamma)\sqrt p.
\]
Since $p<(1-\gamma)^2$, we have $p<(1-\gamma)\sqrt p$. Hence $f_{\rm spike}$ is the unique $L_1$-optimal element of $\cF:=\{f_{\rm spike},f_{\rm flat}\}$ and $\epsilon_{1,*}=p$.

On the other hand, because the regression problem is noiseless with $\eta\equiv0$, population squared risk equals squared $L_2$ error. For the two predictors,
\[
    \EE\bigl[(Y-f_{\rm spike}(X))^2\bigr]=p,
    \qquad
    \EE\bigl[(Y-f_{\rm flat}(X))^2\bigr]=(1-\gamma)^2p<p.
\]
Therefore $f_{\rm flat}$ is the unique population squared-loss minimizer, so $f_{\rm sq}^*=f_{\rm flat}$. Consequently,
\[
    \epsilon_1(f_{\rm sq}^*)
    =
    (1-\gamma)\sqrt p
    =
    (1-\gamma)\sqrt{\epsilon_{1,*}},
\]
which proves \eqref{eq:squared-loss-sharp-approximation-lower}.

It remains to verify the statement for empirical squared-loss ERM. The empirical squared loss of $f_{\rm flat}$ is deterministic and equal to $(1-\gamma)^2p$. For $f_{\rm spike}$, it is
\[
    \frac1n\sum_{i=1}^n \ind\{X_i=b\},
\]
which converges almost surely to $p$ by the strong law of large numbers. Since $(1-\gamma)^2p<p$, squared-loss ERM therefore uniquely selects $f_{\rm flat}=f_{\rm sq}^*$ for all sufficiently large $n$, almost surely.
\end{proof}

In particular, taking $\gamma=1/2$ recovers Proposition~\ref{prop:squared-loss-no-constant-factor}. More generally, because $\gamma$ can be arbitrarily small, Proposition~\ref{prop:squared-loss-sharp-approximation-lower} shows that the coefficient in front of $\sqrt{\epsilon_{1,*}}$ can be made arbitrarily close to one. It also implies the absence of any finite universal approximation factor: for any fixed $C<\infty$, choosing, for example, $\gamma=1/2$ and then $p>0$ sufficiently small gives $\epsilon_1(f_{\rm sq}^*)>C\epsilon_{1,*}$, and empirical squared-loss ERM converges almost surely to this population minimizer.

\section{Proofs and Additional Results for Log-Loss ERM}
\label{app:log-loss}

\subsection{Proof of Equation (\ref{eq:log-loss-population-kl}) }
\label{app:log-loss-population-kl}
By conditioning on $X$ and using $\EE[Y\mid X]=\eta(X)$,
\begin{align*}
    \EE\!\left[
        \ell_{\log}(Y,f(X))
        \,\middle|\,X
    \right]
    &=
    -\eta(X)\log f(X)
    -(1-\eta(X))\log(1-f(X)),\\
    \EE\!\left[
        \ell_{\log}(Y,\eta(X))
        \,\middle|\,X
    \right]
    &=
    -\eta(X)\log\eta(X)
    -(1-\eta(X))\log(1-\eta(X)).
\end{align*}
Therefore,
\begin{align*}
    &\EE\!\left[
        \ell_{\log}(Y,f(X))
        -
        \ell_{\log}(Y,\eta(X))
        \,\middle|\,X
    \right]\\
    &\qquad=
    \eta(X)\log\frac{\eta(X)}{f(X)}
    +(1-\eta(X))
    \log\frac{1-\eta(X)}{1-f(X)}\\
    &\qquad=
    \kl\bigl(\eta(X),f(X)\bigr).
\end{align*}
Taking expectation over $X$ proves \eqref{eq:log-loss-population-kl}.

The same calculation remains valid when $\eta(X)\in\{0,1\}$, with the
standard convention $0\log 0=0$.

\subsection{Proof of Proposition~\ref{prop:log-loss-upper}}
\label{app:log-loss-upper}

\begin{proof}
Let $M:=|\cF|$ and define
\[
    \alpha
    :=
    \frac{\log(2M^2/\delta)}{n}.
\]
For $f\in\cF$, write
\[
    K(f)
    :=
    \EE\!\left[
        \kl\bigl(\eta(X),f(X)\bigr)
    \right],
\]
and define the squared Hellinger-type distance
\[
    H^2(f,\eta)
    :=
    \EE\!\left[
        \bigl(\sqrt{f(X)}-\sqrt{\eta(X)}\bigr)^2
        +
        \bigl(\sqrt{1-f(X)}-\sqrt{1-\eta(X)}\bigr)^2
    \right].
\]
The elementary relation between total variation and Hellinger distance for Bernoulli distributions implies
\begin{equation}
    \epsilon_1(f)
    =
    \EE|f(X)-\eta(X)|
    \le
    H(f,\eta).
    \label{eq:log-proof-l1-hellinger}
\end{equation}

We first relate the Hellinger distance of the empirical minimizer to the empirical excess log loss. For $f\in\cF$, let
\[
    \Delta_f(X,Y)
    :=
    \ell_{\log}(Y,f(X))
    -
    \ell_{\log}(Y,\eta(X)).
\]
Fix $X=x$, and abbreviate $p=\eta(x)$ and $q=f(x)$. When $p\in(0,1)$,
\[
    \exp\left(-\frac12\Delta_f(x,y)\right)
    =
    \left(\frac{q}{p}\right)^{y/2}
    \left(\frac{1-q}{1-p}\right)^{(1-y)/2}.
\]
The right-hand side is a convex function of $y\in[0,1]$. Since
$\EE[Y\mid X=x]=p$, convexity and the chord joining its values at
$0$ and $1$ give
\begin{align*}
    \EE\!\left[
        e^{-\Delta_f/2}
        \,\middle|\, X=x
    \right]
    &\le
    p\sqrt{\frac{q}{p}}
    +(1-p)\sqrt{\frac{1-q}{1-p}} \\
    &=
    \sqrt{pq}
    +
    \sqrt{(1-p)(1-q)}.
\end{align*}
The same inequality holds directly when $p\in\{0,1\}$, since then
$Y=p$ almost surely conditional on $X=x$. Therefore,
\[
    \EE e^{-\Delta_f/2}
    \le
    1-\frac12 H^2(f,\eta)
    \le
    \exp\left(-\frac12H^2(f,\eta)\right).
\]
By independence,
\[
    \EE\exp\left\{
        \frac{n}{2}H^2(f,\eta)
        -\frac12\sum_{i=1}^n\Delta_f(X_i,Y_i)
    \right\}
    \le 1.
\]
Hence, by Markov's inequality,
\[
    \PP\left(
        H^2(f,\eta)
        >
        P_n\Delta_f+2\alpha
    \right)
    \le
    e^{-n\alpha}.
\]
A union bound over $f\in\cF$ shows that, with probability at least
$1-\delta/2$,
\begin{equation}
    H^2(f,\eta)
    \le
    P_n\Delta_f+2\alpha
    \qquad
    \text{for every }f\in\cF.
    \label{eq:log-proof-hellinger-empirical}
\end{equation}
Indeed,
$M e^{-n\alpha}=\delta/(2M)\le\delta/2$.

On this event, empirical optimality of $\widehat f_{\log}$ implies that,
for every $f_0\in\cF$,
\begin{equation}
    H^2(\widehat f_{\log},\eta)
    \le
    P_n\Delta_{\widehat f_{\log}}+2\alpha
    \le
    P_n\Delta_{f_0}+2\alpha.
    \label{eq:log-proof-erm-comparison}
\end{equation}

It remains to control the empirical excess loss of the comparator.
Fix $f_0\in\cF$ and write
\[
    K_0:=K(f_0),
    \qquad
    B:=\log\frac1\gamma.
\]
Since $f_0(X)\in[\gamma,1-\gamma]$, we have
\begin{equation}
    \Delta_{f_0}\le B
    \qquad\text{and}\qquad
    P\Delta_{f_0}=K_0.
    \label{eq:log-proof-excess-basic}
\end{equation}
We next show that
\begin{equation}
    P\Delta_{f_0}^2
    \le
    (B+2)K_0.
    \label{eq:log-proof-second-moment}
\end{equation}

To prove \eqref{eq:log-proof-second-moment}, again condition on $X=x$
and write $p=\eta(x)$ and $q=f_0(x)$. Since $\Delta_{f_0}(x,y)$ is
affine in $y$, its square is convex, and therefore
\begin{align*}
    \EE\!\left[
        \Delta_{f_0}^2
        \,\middle|\,X=x
    \right]
    \le{}&
    p\log^2\frac{p}{q}
    +(1-p)\log^2\frac{1-p}{1-q}.
\end{align*}
We use the elementary inequality
\begin{equation}
    t\log^2t
    \le
    (B+2)\bigl(t\log t-t+1\bigr),
    \qquad
    0\le t\le e^B.
    \label{eq:log-proof-scalar-ineq}
\end{equation}
For completeness, if $t\le1$, writing $u=-\log t$ and using
$e^u\ge1+u+u^2/2$ gives
\[
    t\log^2t
    \le
    2(t\log t-t+1).
\]
If $t\ge1$, write $u=\log t\in[0,B]$. Since
\[
    u-2+(u+2)e^{-u}\ge0,
\]
we have
\[
    u^2
    \le
    (u+2)(u-1+e^{-u})
    \le
    (B+2)(u-1+e^{-u}),
\]
which is equivalent to \eqref{eq:log-proof-scalar-ineq}.

Applying \eqref{eq:log-proof-scalar-ineq} to
$t=p/q$ and $t=(1-p)/(1-q)$, which are both at most
$1/\gamma=e^B$, yields
\begin{align*}
    p\log^2\frac{p}{q}
    &\le
    (B+2)
    \left[
        p\log\frac{p}{q}-p+q
    \right],\\
    (1-p)\log^2\frac{1-p}{1-q}
    &\le
    (B+2)
    \left[
        (1-p)\log\frac{1-p}{1-q}
        -(1-p)+(1-q)
    \right].
\end{align*}
Summing the two inequalities gives
\[
    \EE\!\left[
        \Delta_{f_0}^2
        \,\middle|\,X=x
    \right]
    \le
    (B+2)\kl(p,q).
\]
Taking expectation over $X$ proves
\eqref{eq:log-proof-second-moment}.

A one-sided Bernstein inequality, followed by a union bound over
$f_0\in\cF$, now implies that, with probability at least $1-\delta/2$,
simultaneously for every $f_0\in\cF$,
\begin{equation}
    P_n\Delta_{f_0}
    \le
    K_0
    +
    \sqrt{2(B+2)K_0\alpha}
    +
    \frac{2B}{3}\alpha.
    \label{eq:log-proof-comparator-bernstein}
\end{equation}
Here we used
$\Var(\Delta_{f_0})\le P\Delta_{f_0}^2$
together with \eqref{eq:log-proof-second-moment}.

Combining
\eqref{eq:log-proof-erm-comparison}
and
\eqref{eq:log-proof-comparator-bernstein},
we obtain, simultaneously for every $f_0\in\cF$,
\begin{equation}
    H^2(\widehat f_{\log},\eta)
    \le
    K_0
    +
    \sqrt{2(B+2)K_0\alpha}
    +
    \left(2+\frac{2B}{3}\right)\alpha.
    \label{eq:log-proof-h-bound}
\end{equation}
Since
\[
    2+\frac{2B}{3}
    \ge
    \frac{B+2}{2},
\]
the right-hand side of \eqref{eq:log-proof-h-bound} is at most
\[
    \left(
        \sqrt{K_0}
        +
        \sqrt{
            \left(2+\frac{2B}{3}\right)\alpha
        }
    \right)^2.
\]
Together with \eqref{eq:log-proof-l1-hellinger}, this yields
\begin{equation}
    \epsilon_1(\widehat f_{\log})
    \le
    \sqrt{K(f_0)}
    +
    \sqrt{
        \left(2+\frac{2B}{3}\right)\alpha
    }
    \qquad
    \text{for every }f_0\in\cF.
    \label{eq:log-proof-kl-oracle}
\end{equation}

Finally, we relate the KL approximation error to the $L_1$
approximation error. For any $p\in[0,1]$ and
$q\in[\gamma,1-\gamma]$, convexity of
$p\mapsto\kl(p,q)$, together with $\kl(q,q)=0$, gives
\[
    \kl(p,q)
    \le
    |p-q|
    \max\left\{
        \frac{-\log q}{1-q},
        \frac{-\log(1-q)}{q}
    \right\}.
\]
Because $q,1-q\ge\gamma$,
\[
    \max\left\{
        \frac{-\log q}{1-q},
        \frac{-\log(1-q)}{q}
    \right\}
    \le
    \frac{\log(1/\gamma)}{1-\gamma}.
\]
Hence
\begin{equation}
    K(f)
    \le
    \frac{\log(1/\gamma)}{1-\gamma}
    \epsilon_1(f)
    \qquad
    \text{for every }f\in\cF.
    \label{eq:log-proof-kl-l1}
\end{equation}

Let
$f_*\in\argmin_{f\in\cF}\epsilon_1(f)$, which exists because
$\cF$ is finite. Substituting $f_0=f_*$ into
\eqref{eq:log-proof-kl-oracle} and applying
\eqref{eq:log-proof-kl-l1} gives
\[
    \epsilon_1(\widehat f_{\log})
    \le
    \sqrt{
        \frac{\log(1/\gamma)}{1-\gamma}
        \,\epsilon_{1,*}
    }
    +
    \sqrt{
        \left(
            2+\frac23\log\frac1\gamma
        \right)
        \frac{\log(2|\cF|^2/\delta)}{n}
    }.
\]
This proves \eqref{eq:log-loss-l1-upper}. For fixed $\gamma$, the
display immediately implies
\[
    \epsilon_1(\widehat f_{\log})
    \lesssim_\gamma
    \sqrt{\epsilon_{1,*}}
    +
    \sqrt{\frac{\log(|\cF|/\delta)}{n}}.
\]
\end{proof}

\subsection{Proof of Proposition~\ref{prop:log-loss-no-constant-factor}}
\label{app:log-loss-approximation-lower}

\begin{proof}
Fix $p\in(0,1)$. Let $\cX=\{a,b\}$ with
\[
    \PP(X=b)=p,
    \qquad
    \PP(X=a)=1-p,
\]
and let
\[
    Y=\frac12
    \qquad\text{almost surely}.
\]
Thus the problem is noiseless and
\[
    \eta(a)=\eta(b)=\frac12.
\]

Consider the two hypotheses $f_{\rm spike}$ and $f_{\rm flat}$ defined by
\[
    f_{\rm spike}(a)=\frac12,
    \qquad
    f_{\rm spike}(b)=\frac34,
\]
and
\[
    f_{\rm flat}(x)
    =
    \frac12+\frac{\sqrt p}{4},
    \qquad x\in\{a,b\}.
\]
Since $p\in(0,1)$, both hypotheses take values in $[1/4,3/4]$. Let
$\cF=\{f_{\rm spike},f_{\rm flat}\}$.

Their $L_1$ errors are
\[
    \epsilon_1(f_{\rm spike})
    =
    p\left|\frac34-\frac12\right|
    =
    \frac p4,
\]
and
\[
    \epsilon_1(f_{\rm flat})
    =
    \frac{\sqrt p}{4}.
\]
Because $p<\sqrt p$ for $p\in(0,1)$, $f_{\rm spike}$ is the unique
$L_1$-optimal hypothesis, and hence
\begin{equation}
    \epsilon_{1,*}
    =
    \frac p4.
    \label{eq:log-loss-lower-opt}
\end{equation}

We next compare the population log risks. Since $\eta\equiv1/2$, the
excess population log risk of any $f$ is
$\EE[\kl(1/2,f(X))]$. For $f_{\rm spike}$,
\begin{align}
    \EE\!\left[
        \kl\!\left(\frac12,f_{\rm spike}(X)\right)
    \right]
    &=
    p\,\kl\!\left(\frac12,\frac34\right) \nonumber\\
    &=
    -\frac p2\log\frac34.
    \label{eq:log-loss-lower-spike-kl}
\end{align}
For $f_{\rm flat}$, using
\[
    4f_{\rm flat}(1-f_{\rm flat})
    =
    1-\frac p4,
\]
we obtain
\begin{align}
    \EE\!\left[
        \kl\!\left(\frac12,f_{\rm flat}(X)\right)
    \right]
    &=
    \kl\!\left(
        \frac12,
        \frac12+\frac{\sqrt p}{4}
    \right) \nonumber\\
    &=
    -\frac12\log\left(1-\frac p4\right).
    \label{eq:log-loss-lower-flat-kl}
\end{align}

To compare the two quantities, consider
$g(u):=-\log(1-u/4)$ on $[0,1]$. The function $g$ is strictly convex and
$g(0)=0$. Therefore, for every $p\in(0,1)$,
\[
    -\log\left(1-\frac p4\right)
    =
    g(p)
    <
    p g(1)
    =
    -p\log\frac34.
\]
Combining this inequality with
\eqref{eq:log-loss-lower-spike-kl}--\eqref{eq:log-loss-lower-flat-kl}
shows that
\[
    \EE\!\left[
        \ell_{\log}(Y,f_{\rm flat}(X))
    \right]
    <
    \EE\!\left[
        \ell_{\log}(Y,f_{\rm spike}(X))
    \right].
\]
Hence $f_{\rm flat}$ is the unique population log-loss minimizer:
\[
    f_{\log}^*=f_{\rm flat}.
\]
Its $L_1$ error is therefore
\[
    \epsilon_1(f_{\log}^*)
    =
    \frac{\sqrt p}{4}.
\]
Using \eqref{eq:log-loss-lower-opt},
\[
    \frac{\sqrt p}{4}
    =
    \frac12\sqrt{\frac p4}
    =
    \frac12\sqrt{\epsilon_{1,*}},
\]
which proves \eqref{eq:log-loss-approximation-lower}.

For completeness, this population obstruction also applies to empirical
log-loss ERM. Since $\cF$ contains only two hypotheses and both are
uniformly bounded away from $0$ and $1$, their log losses are bounded.
The strong law of large numbers therefore implies that their empirical
risks converge almost surely to their respective population risks.
Because the population inequality above is strict,
\[
    \widehat f_{\log}
    =
    f_{\rm flat}
\]
eventually almost surely. Consequently,
\[
    \epsilon_1(\widehat f_{\log})
    \longrightarrow
    \frac{\sqrt p}{4}
    =
    \frac12\sqrt{\epsilon_{1,*}}
    \qquad\text{almost surely}.
\]

Finally,
\[
    \frac{\epsilon_1(f_{\log}^*)}{\epsilon_{1,*}}
    =
    \frac{1}{\sqrt p},
\]
which diverges as $p\downarrow0$. Thus, for any finite constant $C$, one
can choose $p$ sufficiently small so that
$\epsilon_1(f_{\log}^*)>C\epsilon_{1,*}$. Since empirical log-loss ERM
eventually selects $f_{\log}^*$ almost surely, it cannot satisfy a
universal constant-factor $L_1$ oracle guarantee of the form
\eqref{eq:basic-goal}. This failure occurs even though
$\cF\subseteq[1/4,3/4]^{\cX}$ and the labels are noiseless.
\end{proof}

\subsection{Failure without uniform boundary control}
\label{app:log-loss-boundary}

We next show that the restriction that all hypotheses stay uniformly
away from $0$ and $1$ is essential for Proposition~\ref{prop:log-loss-upper}.
Without such a restriction, the failure of log-loss ERM can be
substantially stronger than the $\sqrt{\epsilon_{1,*}}$ phenomenon in
Proposition~\ref{prop:log-loss-no-constant-factor}: the best $L_1$
approximation can have arbitrarily small error, while the population
log-loss minimizer has $L_1$ error bounded away from zero.

\begin{proposition}[Failure of log-loss ERM without uniform boundary control]
\label{prop:log-loss-boundary}
For every $p\in(0,1/4]$, there exists a noiseless regression problem and
a two-element class $\cF\subseteq(0,1)^{\cX}$ such that
\[
    \epsilon_{1,*}
    =
    p\left(\frac12-e^{-1/p}\right)
    \le \frac p2,
\]
while the unique population log-loss minimizer $f_{\log}^*$ satisfies
\[
    \epsilon_1(f_{\log}^*)=\frac14.
\]
Consequently, there is no function $\Psi:[0,1]\to[0,\infty)$ satisfying
$\Psi(\epsilon)\to0$ as $\epsilon\to0$ for which population log-loss
minimization obeys
\[
    \epsilon_1(f_{\log}^*)
    \le
    \Psi(\epsilon_{1,*})
\]
uniformly over finite classes $\cF\subseteq(0,1)^{\cX}$. Moreover,
empirical log-loss ERM eventually selects $f_{\log}^*$ almost surely.
\end{proposition}

\begin{proof}
Fix $p\in(0,1/4]$. Let $\cX=\{a,b\}$ with
\[
    \PP(X=b)=p,
    \qquad
    \PP(X=a)=1-p,
\]
and let
\[
    Y=\frac12
    \qquad\text{almost surely}.
\]
Thus
\[
    \eta(a)=\eta(b)=\frac12,
\]
and the regression problem is noiseless.

Consider the two hypotheses $f_0$ and $g$ given by
\[
    f_0(a)=\frac12,
    \qquad
    f_0(b)=e^{-1/p},
\]
and
\[
    g(x)=\frac34,
    \qquad x\in\{a,b\}.
\]
Both hypotheses take values strictly between $0$ and $1$, so all
population log risks are finite.

Their $L_1$ errors are
\[
    \epsilon_1(f_0)
    =
    p\left(\frac12-e^{-1/p}\right)
    \le
    \frac p2
\]
and
\[
    \epsilon_1(g)=\frac14.
\]
Since $p\le1/4$ and $e^{-1/p}>0$,
\[
    \epsilon_1(f_0)<\frac14=\epsilon_1(g),
\]
so $f_0$ is the unique $L_1$-optimal hypothesis and
\begin{equation}
    \epsilon_{1,*}
    =
    p\left(\frac12-e^{-1/p}\right).
    \label{eq:log-loss-boundary-l1-opt}
\end{equation}
In particular, $\epsilon_{1,*}\to0$ as $p\downarrow0$.

We now compare the population log risks. Since $\eta\equiv1/2$, their
excess log risks are given by
\[
    K(f)
    :=
    \EE\!\left[
        \kl\!\left(\frac12,f(X)\right)
    \right].
\]
For $f_0$,
\begin{align}
    K(f_0)
    &=
    p\,\kl\!\left(\frac12,e^{-1/p}\right)
    \nonumber\\
    &=
    p\left[
        \frac12\log\frac{1/2}{e^{-1/p}}
        +
        \frac12\log\frac{1/2}{1-e^{-1/p}}
    \right]
    \nonumber\\
    &=
    \frac12
    -
    p\log 2
    -
    \frac p2\log\left(1-e^{-1/p}\right)
    \nonumber\\
    &\ge
    \frac12-p\log2.
    \label{eq:log-loss-boundary-f0-kl}
\end{align}
For $g$,
\begin{align}
    K(g)
    &=
    \kl\!\left(\frac12,\frac34\right)
    =
    \frac12\log\frac{4}{3}.
    \label{eq:log-loss-boundary-g-kl}
\end{align}
For $p\le1/4$,
\[
    \frac12-p\log2
    \ge
    \frac12-\frac14\log2
    >
    \frac12\log\frac43.
\]
Combining this with
\eqref{eq:log-loss-boundary-f0-kl} and
\eqref{eq:log-loss-boundary-g-kl} yields
\[
    K(f_0)>K(g).
\]
Hence $g$ is the unique population log-loss minimizer:
\[
    f_{\log}^*=g.
\]
Therefore,
\begin{equation}
    \epsilon_1(f_{\log}^*)
    =
    \frac14,
    \label{eq:log-loss-boundary-constant-error}
\end{equation}
whereas
$\epsilon_{1,*}\to0$ by
\eqref{eq:log-loss-boundary-l1-opt}.

This already rules out any uniform population guarantee of the form
$\epsilon_1(f_{\log}^*)\le\Psi(\epsilon_{1,*})$ with
$\Psi(\epsilon)\to0$ as $\epsilon\to0$.

Finally, both hypotheses have bounded log loss on this fixed instance,
and the population log-risk gap is strictly positive. By the strong law
of large numbers, their empirical log risks converge almost surely to
their respective population risks. Thus empirical log-loss ERM selects
$g$ eventually almost surely, and consequently
\[
    \epsilon_1(\widehat f_{\log})
    \longrightarrow
    \frac14
    \qquad\text{almost surely}.
\]
\end{proof}

Proposition~\ref{prop:log-loss-boundary} shows that some uniform control away from $0$ and $1$ is necessary for an $L_1$ oracle guarantee of the form in Proposition~\ref{prop:log-loss-upper}. The failure is not caused by allowing predictions to equal exactly $0$ or $1$: for every fixed instance in the construction, all predictions lie strictly inside $(0,1)$ and all log risks are finite. Rather, the problem is that the $L_1$-optimal hypothesis is allowed to approach the boundary as $\epsilon_{1,*}\to0$. In that regime, an arbitrarily small $L_1$ approximation error can still incur a nonvanishing log-loss penalty, causing log-loss minimization to prefer a hypothesis whose $L_1$ error stays bounded away from zero.

\section{Why not absolute loss?}
\label{app:absolute-loss}

Since our target is the $L_1$ error $\normz{f-\eta}_1=\EE|f(X)-\eta(X)|$, it is natural to ask whether one can simply minimize the empirical absolute loss $\frac1n\sum_{i=1}^n |Y_i-f(X_i)|$. Despite the superficial similarity in $L_1$ measure, this approach is fundamentally misaligned with our objective: its population counterpart $\EE|Y-f(X)|$ targets the conditional median of $Y\mid X$, whereas our target is the conditional mean $\eta(x)=\EE[Y\mid X=x]$.

Indeed, for each $x$, the conditional absolute risk $t\mapsto \EE[|Y-t|\mid X=x]$ is minimized by a conditional median of $Y\mid X=x$. Thus, even with infinite data, absolute-loss minimization may not recover the best approximation to $\eta$ in $L_1$. In fact, the failure can occur even under realizability.

Consider the simplest example where $X\equiv x_0$ and $Y\sim\Bernoulli(p)$ for some $0<p<1/2$. Let $\cF=\{0,p\}$ consist of two constant predictors. Since $\eta(x_0)=p\in\cF$, we have $\eps_{1,*}:=\inf_{f\in\cF}\normz{f-\eta}_1=0$. On the other hand, for any $t\in[0,1]$, the population absolute risk is
\[
\EE|Y-t|=p(1-t)+(1-p)t=p+(1-2p)t.
\]
Since $p<1/2$, this risk is strictly increasing in $t$, and hence its unique minimizer over $\cF$ is $0$, rather than the true conditional mean $p$. 

This example highlights the basic distinction between absolute loss and our sign-betting objective. Absolute loss compares a predictor directly with the observed response $Y$ and therefore inherits the median-seeking geometry of absolute deviation. In contrast, sign betting uses pairwise signed residual comparisons designed to recover the geometry of distance to the conditional mean. We next make this distinction more explicit at the population level.

\paragraph{The sign in absolute loss depends on the observed response.}
Recall that $|Y-f(X)|=(Y-f(X))\sign(Y-f(X))$. Thus absolute loss can equivalently be written as $\EE[\sup_{s\in[-1,1]}s(Y-f(X))]$. Importantly, the maximizing direction is chosen separately for each realization of $(X,Y)$ and therefore may depend on the observed response $Y$. Conditioning on $X=x$, the resulting objective is $t\mapsto \EE[|Y-t|\mid X=x]$, whose minimizers are exactly the conditional medians of $Y\mid X=x$. Hence the full conditional distribution of $Y\mid X$ matters, rather than only its conditional mean.

Our sign-betting objective instead takes the form
\[
    \widehat L_{\rm SB}(f)
    :=
    \sup_{h\in\cF}
    \frac1n\sum_{i=1}^n
    (Y_i-f(X_i))\sign(h(X_i)-f(X_i)).
\]
For every fixed pair $(f,h)$, the direction $\sign(h(X)-f(X))$ is a function of $X$ alone and does not depend on the realized response. Consequently,
\[
    \EE[(Y-f(X))\sign(h(X)-f(X))]
    =
    \EE[(\eta(X)-f(X))\sign(h(X)-f(X))],
\]
where the equality follows from $\EE[Y-\eta(X)\mid X]=0$. Thus the population sign-betting objective
\[
    B_P(f)
    :=
    \sup_{h\in\cF}
    \EE[(Y-f(X))\sign(h(X)-f(X))]
\]
depends on the conditional distribution of $Y\mid X$ only through its mean $\eta(X)$.

This difference can also be viewed as a difference in the order of the supremum and expectation. In the scalar case, absolute loss has the form
\[
    \EE|Y-f|
    =
    \EE\left[\sup_{s\in[-1,1]}s(Y-f)\right],
\]
where the sign may adapt to each realization of $Y$. In contrast, if the direction must be chosen before observing the response, then
\[
    \sup_{s\in[-1,1]}\EE[s(Y-f)]
    =
    |\EE[Y]-f|.
\]
The former is minimized by a median of $Y$, whereas the latter is minimized by its mean. Sign betting implements the analogous distinction in regression by restricting the betting direction to be $X$-measurable.

\paragraph{A direct comparison on the realizable Bernoulli example.}
Consider again the example at the beginning of this appendix section: $X\equiv x_0$, $Y\sim\Bernoulli(p)$ with $0<p<1/2$, and $\cF=\{0,p\}$. Here $\eta=p$. As shown at the beginning of this section, absolute loss strictly prefers the predictor $0$, since $\EE|Y-t|=p+(1-2p)t$ is increasing in $t$.

The sign-betting population objective makes the opposite choice. For $f=p$, choosing $h=p$ gives the zero direction, while choosing $h=0$ gives
\[
    \EE[(Y-p)\sign(0-p)]
    =
    -\EE[Y-p]
    =
    0.
\]
Hence $B_P(p)=0$. For $f=0$, choosing $h=p$ gives
\[
    \EE[(Y-0)\sign(p-0)]
    =
    \EE[Y]
    =
    p,
\]
and therefore $B_P(0)=p$. Thus sign betting uniquely prefers $p=\eta$, while absolute loss uniquely prefers the conditional median $0$. In particular,
\[
    \argmin_{f\in\cF}\EE|Y-f(X)|
    =
    \{0\},
    \qquad
    \argmin_{f\in\cF}B_P(f)
    =
    \{\eta\}.
\]

More generally, whenever $\eta\in\cF$, we have $B_P(\eta)=0$. For any $f\in\cF$, using $h=\eta$ as the challenger gives
\[
    B_P(f)
    \ge
    \EE[(\eta(X)-f(X))\sign(\eta(X)-f(X))]
    =
    \normz{f-\eta}_1.
\]
Hence $\eta$ is the unique population minimizer whenever $\normz{f-\eta}_1>0$ for every $f\neq\eta$. This is fundamentally different from absolute loss, which can fail to recover $\eta$ even under realizability.

\paragraph{Connection to the dual representation of the $L_1$ error.}
The distinction above can also be understood through
\[
    \normz{f-\eta}_1
    =
    \sup_{\substack{s:\cX\to[-1,1]\\ s\ {\rm measurable}}}
    \EE[(\eta(X)-f(X))s(X)]
    =
    \sup_{\substack{s:\cX\to[-1,1]\\ s\ {\rm measurable}}}
    \EE[(Y-f(X))s(X)].
\]
The optimal direction is $s(X)=\sign(\eta(X)-f(X))$. 
Our sign-betting objective is a restricted version of this dual formulation, using the comparison directions $\{\sign(h-f):h\in\cF\}$ available from the model class.
In contrast, absolute loss uses the response-dependent direction $\sign(Y-f(X))$, which lead to a target of conditional median.

\section{Failure of the original betting loss under misspecification}
\label{app:original-betting-lb}

Betting loss was introduced in \citet{pmlr-v336-li26a} to obtain variance-adaptive guarantees in the realizable setting. Under misspecification, however, it suffers from a similar approximation-error obstruction as squared and log loss: its $L_1$ error can scale as $\sqrt{\eps_{1,*}}$, rather than linearly in $\eps_{1,*}$. This is not merely a limitation of the existing realizable analysis. 
The following proposition shows that the original betting loss can have asymptotic $L_1$ error of order $\sqrt{\eps_{1,*}}$, ruling out any universal constant-factor agnostic $L_1$ oracle guarantee. 

Recall the betting objective of \citet{pmlr-v336-li26a}. For $f,h\in\cF$, $\phi\in[0,n/4]$, and $c\in[0,1/4]$, define
\[
    b_{h,f}^{\phi,c}(x)
    :=
    \clip{\phi(h(x)-f(x))}_{[-c,c]},
\]
and
\[
    \widehat L_{\rm bet,n}(f)
    :=
    \max_{h\in\cF}
    \max_{\phi\in[0,n/4]}
    \max_{c\in[0,1/4]}
    \frac1n\sum_{i=1}^n
    \log\left(
        1+(Y_i-f(X_i))b_{h,f}^{\phi,c}(X_i)
    \right).
\]
The original betting estimator is any empirical minimizer
$\hf_{\rm bet}\in\argmin_{f\in\cF}\widehat L_{\rm bet,n}(f)$.
For the proof, it is useful to introduce the corresponding population objective
\[
    L_{\rm bet,n}(f)
    :=
    \max_{h\in\cF}
    \max_{\phi\in[0,n/4]}
    \max_{c\in[0,1/4]}
    \EE\left[
        \log\left(
            1+(Y-f(X))b_{h,f}^{\phi,c}(X)
        \right)
    \right].
\]
We keep the subscript $n$ because the admissible range of $\phi$ depends on the sample size.

\begin{proposition}
\label{prop:original-betting-agnostic-lb}
For every $p\in(0,1/64]$, there exist a distribution on a two-point input space and a two-element class $\cF$ such that $\eps_{1,*}=p$, but every empirical minimizer $\hf_{\rm bet}$ of the original betting objective satisfies
\[
    \normz{\hf_{\rm bet}-\eta}_1
    \longrightarrow
    \frac{1-p}{4}\sqrt{p}
    \qquad\text{almost surely as } n\to\infty.
\]
Consequently, there is no universal constant $C<\infty$ such that the original betting estimator satisfies
\[
    \normz{\hf_{\rm bet}-\eta}_1
    \le
    C\eps_{1,*}+o_n(1)
\]
over arbitrary misspecified finite classes.
\end{proposition}

\begin{proof}
Fix $p\in(0,1/64]$ and let $\cX=\{a,b\}$, with $\PP(X=b)=p$ and $\PP(X=a)=1-p$. Let
\[
    Y\mid X=a\sim\Bernoulli(1/2),
    \qquad
    Y\mid X=b=0
    \quad\text{almost surely}.
\]
Thus $\eta(a)=1/2$ and $\eta(b)=0$.

Set $t:=\sqrt{p}/4$ and consider the two-element class $\cF=\{f_0,g\}$ defined by
\[
\begin{array}{c|cc}
    & a & b\\ \hline
    \eta & 1/2 & 0\\
    f_0  & 1/2 & 1\\
    g    & 1/2+t & 0.
\end{array}
\]
Their $L_1$ errors are
\[
    \normz{f_0-\eta}_1=p,
    \qquad
    \normz{g-\eta}_1=(1-p)t
    =\frac{1-p}{4}\sqrt p.
\]
Since $p\le 1/64$, we have $4\sqrt p\le1/2<1-p$, and therefore $(1-p)t>p$. Hence $f_0$ is the unique $L_1$ oracle and
\[
    \eps_{1,*}=p.
\]

We next show that the population betting objective strictly prefers $g$.

\paragraph{Step 1: the $L_1$ oracle has population betting loss larger than $p/5$.}
For $f=f_0$, choose the challenger $h=g$ and take $\phi=c=1/4$, which is admissible for every $n\ge1$. At $X=a$, the bet is $t/4$, whereas at $X=b$ the bet is $-1/4$. Therefore,
\[
\begin{aligned}
    L_{\rm bet,n}(f_0)
    &\ge
    p\log\frac54
    +
    \frac{1-p}{2}
    \log\left(1-\frac{t^2}{64}\right)\\
    &=
    p\log\frac54
    +
    \frac{1-p}{2}
    \log\left(1-\frac{p}{1024}\right).
\end{aligned}
\]
Using $\log(5/4)\ge 7/32$ and $\log(1-u)\ge-2u$ for $u\in[0,1/2]$,
\[
\begin{aligned}
    L_{\rm bet,n}(f_0)
    &\ge
    \frac{7p}{32}
    -
    \frac{(1-p)p}{1024}\\
    &\ge
    \left(
        \frac7{32}-\frac1{1024}
    \right)p
    >
    \frac p5.
\end{aligned}
\]
This bound holds for every $n\ge1$.

\paragraph{Step 2: the worse $L_1$ predictor has population betting loss at most $p/8$.}
For $f=g$, the challenger $h=g$ gives betting loss zero, so it remains to consider $h=f_0$. At $X=b$, we have $Y-g(b)=0$, and hence this point contributes zero regardless of the bet.

At $X=a$, write $u:=\min\{\phi t,c\}\in[0,1/4]$. Since $f_0(a)-g(a)=-t$, the bet equals $-u$. Conditional on $X=a$, the expected log payoff is
\[
\begin{aligned}
    &\frac12
    \log\left(1+u\left(\frac12+t\right)\right)
    +
    \frac12
    \log\left(1-u\left(\frac12-t\right)\right)\\
    &\qquad=
    \frac12
    \log\left(
        1+2tu-\left(\frac14-t^2\right)u^2
    \right).
\end{aligned}
\]
Completing the square gives
\[
    1+2tu-\left(\frac14-t^2\right)u^2
    \le
    \frac1{1-4t^2}.
\]
Consequently,
\[
\begin{aligned}
    L_{\rm bet,n}(g)
    &\le
    -\frac{1-p}{2}\log(1-4t^2)\\
    &=
    -\frac{1-p}{2}\log\left(1-\frac p4\right).
\end{aligned}
\]
Using $-\log(1-u)\le u/(1-u)$,
\[
    L_{\rm bet,n}(g)
    \le
    \frac{p(1-p)}{8(1-p/4)}
    <
    \frac p8.
\]
Thus, uniformly over $n$,
\[
    L_{\rm bet,n}(f_0)-L_{\rm bet,n}(g)
    >
    \frac p5-\frac p8
    =
    \frac{3p}{40}.
\]
In particular, the population betting objective strictly prefers $g$, despite $f_0$ being the unique $L_1$ oracle.

\paragraph{Step 3: the empirical objective eventually makes the same choice.}
The distribution of $(X,Y)$ is supported on only three atoms,
\[
    \{(a,0),(a,1),(b,0)\}.
\]
For any $f,h\in\cF$ and any admissible $(\phi,c)$, define
\[
    \ell_{f,h,\phi,c}(x,y)
    :=
    \log\left(
        1+(y-f(x))b_{h,f}^{\phi,c}(x)
    \right).
\]
Since $|y-f(x)|\le1$ and $|b_{h,f}^{\phi,c}(x)|\le1/4$, we have
\[
    \frac34
    \le
    1+(y-f(x))b_{h,f}^{\phi,c}(x)
    \le
    \frac54,
\]
and hence $|\ell_{f,h,\phi,c}(x,y)|\le\log(4/3)$ uniformly over all $f,h,\phi,c$ and all $n$.

Let $\widehat q_n(z)$ and $q(z)$ denote the empirical and population probabilities of an atom $z$. Then
\[
\begin{aligned}
    &\sup_{f,h,\phi,c}
    \left|
        (P_n-P)\ell_{f,h,\phi,c}
    \right|\\
    &\qquad\le
    \log(4/3)
    \sum_{z\in\{(a,0),(a,1),(b,0)\}}
    |\widehat q_n(z)-q(z)|
    \longrightarrow0
    \qquad\text{almost surely}.
\end{aligned}
\]
Importantly, this bound is uniform over all admissible betting parameters, so the fact that the range $\phi\in[0,n/4]$ grows with $n$ causes no difficulty.

It follows that
\[
    \max_{f\in\cF}
    \left|
        \widehat L_{\rm bet,n}(f)-L_{\rm bet,n}(f)
    \right|
    \longrightarrow0
    \qquad\text{almost surely}.
\]
Combining this with the uniform population gap above, we obtain
\[
    \widehat L_{\rm bet,n}(g)
    <
    \widehat L_{\rm bet,n}(f_0)
\]
for all sufficiently large $n$, almost surely. Hence every empirical minimizer satisfies $\hf_{\rm bet}=g$ eventually almost surely. Therefore,
\[
    \normz{\hf_{\rm bet}-\eta}_1
    \longrightarrow
    \normz{g-\eta}_1
    =
    \frac{1-p}{4}\sqrt p
    \qquad\text{almost surely}.
\]

Finally,
\[
    \frac{\normz{g-\eta}_1}{\eps_{1,*}}
    =
    \frac{1-p}{4\sqrt p}
    \longrightarrow\infty
    \qquad\text{as }p\downarrow0.
\]
Thus, for any fixed $C<\infty$, one can choose $p>0$ sufficiently small such that $\normz{g-\eta}_1>C\eps_{1,*}$. Since every empirical minimizer converges almost surely to $g$, no universal bound of the form
\[
    \normz{\hf_{\rm bet}-\eta}_1
    \le
    C\eps_{1,*}+o_n(1)
\]
can hold over arbitrary misspecified finite classes.
\end{proof}

\paragraph{The obstruction is already present at the population level.}
The proof establishes more than asymptotic failure of the empirical estimator. For every $n$, the population objective satisfies
\[
    L_{\rm bet,n}(g)
    <
    L_{\rm bet,n}(f_0),
\]
with a gap bounded below by a positive multiple of $p$. Thus the failure is not caused by sampling fluctuations or by insufficient concentration of the empirical betting loss. Even with exact knowledge of the population objective, minimizing the original betting loss selects the predictor with the larger $L_1$ error on this instance.

Proposition~\ref{prop:original-betting-agnostic-lb} shows that the original betting loss ERM can asymptotically select a predictor with error $\Theta(\sqrt{\eps_{1,*}})$, and this deficiency is not merely an artifact of that proof strategy. This motivates the use of a different objective in the misspecified setting.

\section{Proof of the finite-class variance-adaptive guarantee}
\label{app:proof-sbr-variance-adaptive}
\begin{proof}[Proof of Theorem~\ref{thm:sbr-variance-adaptive}]
Let $M:=|\cF|$ and, for brevity in the proof, write
\[
    \alpha
    :=
    \frac{\log(2M^2/\delta)}{n}.
\]
For $f,h\in\cF$, define
\[
    Z_{f,h}(X,Y)
    :=
    (Y-f(X))\sign(h(X)-f(X)),
    \qquad
    D_{f,h}(X)
    :=
    |f(X)-h(X)|.
\]
Thus $\widehat T(f,h)=P_n Z_{f,h}$ and
$T_P(f,h)=P Z_{f,h}$.

We first establish a uniform concentration event. Since $Y,f(X)\in[0,1]$,
we have $|Z_{f,h}|\le 1$, and hence
$|Z_{f,h}-PZ_{f,h}|\le 2$ and
$\Var(Z_{f,h})\le PZ_{f,h}^2$. A one-sided Bernstein inequality therefore
gives, for each fixed pair $(f,h)$,
\begin{equation}
    (P_n-P)Z_{f,h}
    \le
    \sqrt{2PZ_{f,h}^2\,\alpha}
    +\frac{4}{3}\alpha
    \label{eq:appendix-bernstein-betting}
\end{equation}
except on an event of probability at most $e^{-n\alpha}$.

Similarly, since $0\le D_{f,h}\le 1$,
\[
    \Var(D_{f,h})
    \le
    PD_{f,h}^2
    \le
    PD_{f,h}.
\]
The lower-tail Bernstein inequality gives
\begin{equation}
    (P-P_n)D_{f,h}
    \le
    \sqrt{2PD_{f,h}\,\alpha}
    +\frac{2}{3}\alpha
    \label{eq:appendix-bernstein-distance}
\end{equation}
except on an event of probability at most $e^{-n\alpha}$.

There are at most $M^2$ ordered pairs in each of
\eqref{eq:appendix-bernstein-betting} and
\eqref{eq:appendix-bernstein-distance}. Therefore, by a union bound, both
inequalities hold simultaneously for every $f,h\in\cF$ with probability at
least
\[
    1-2M^2 e^{-n\alpha}
    =
    1-\delta.
\]
We work on this event throughout the remainder of the proof.

Fix an arbitrary comparator $f_0\in\cF$. By the definition of $B_P(f_0)$,
for every $h\in\cF$,
\[
    PZ_{f_0,h}
    =
    T_P(f_0,h)
    \le
    B_P(f_0).
\]
Moreover,
\[
    PZ_{f_0,h}^2
    =
    \EE\!\left[
        (Y-f_0(X))^2
        \mathbf{1}\{h(X)\neq f_0(X)\}
    \right]
    \le
    v(f_0).
\]
Hence \eqref{eq:appendix-bernstein-betting} implies
\begin{equation}
    \sup_{h\in\cF}\widehat T(f_0,h)
    \le
    B_P(f_0)
    +
    \sqrt{2v(f_0)\alpha}
    +
    \frac{4}{3}\alpha.
    \label{eq:appendix-comparator-objective}
\end{equation}

The key deterministic step is the sample analogue of
Lemma~\ref{lem:sign-betting-pairwise-identity}. Pointwise,
\[
    Z_{f,h}+Z_{h,f}
    =
    |f-h|,
\]
and therefore
\begin{align}
    P_n|\hf-f_0|
    &=
    \widehat T(\hf,f_0)
    +
    \widehat T(f_0,\hf)
    \nonumber\\
    &\le
    \sup_{h\in\cF}\widehat T(\hf,h)
    +
    \sup_{h\in\cF}\widehat T(f_0,h)
    \nonumber\\
    &\le
    2\sup_{h\in\cF}\widehat T(f_0,h),
    \label{eq:appendix-empirical-distance}
\end{align}
where the last inequality follows from the empirical optimality of $\hf$ in
\eqref{eq:sign-betting-estimator}. Combining
\eqref{eq:appendix-comparator-objective} and
\eqref{eq:appendix-empirical-distance} gives
\begin{equation}
    P_n|\hf-f_0|
    \le
    2B_P(f_0)
    +
    2\sqrt{2v(f_0)\alpha}
    +
    \frac{8}{3}\alpha.
    \label{eq:appendix-empirical-distance-bound}
\end{equation}

It remains to convert the empirical distance to its population counterpart.
Let
\[
    d
    :=
    \normz{\hf-f_0}_1
    =
    P D_{\hf,f_0}.
\]
Because \eqref{eq:appendix-bernstein-distance} holds simultaneously for all
pairs in $\cF$, it may be applied to the data-dependent pair $(\hf,f_0)$.
Thus
\[
    d
    \le
    P_n|\hf-f_0|
    +
    \sqrt{2d\alpha}
    +
    \frac{2}{3}\alpha.
\]
Using \eqref{eq:appendix-empirical-distance-bound},
\begin{equation}
    d
    \le
    2B_P(f_0)
    +
    2\sqrt{2v(f_0)\alpha}
    +
    \sqrt{2d\alpha}
    +
    \frac{10}{3}\alpha.
    \label{eq:appendix-self-bounding-distance}
\end{equation}

Fix any $\rho>0$ and set
\[
    \theta
    :=
    \frac{\rho}{2+\rho}\in(0,1).
\]
Young's inequality yields
\[
    \sqrt{2d\alpha}
    \le
    \theta d+\frac{\alpha}{2\theta}.
\]
Substituting this into
\eqref{eq:appendix-self-bounding-distance} and using
$1/(1-\theta)=(2+\rho)/2$, we obtain
\begin{align}
    d
    &\le
    (2+\rho)
    \left(
        B_P(f_0)+\sqrt{2v(f_0)\alpha}
    \right)
    \nonumber\\
    &\quad+
    \left(
        \frac{13}{3}
        +\frac{1}{\rho}
        +\frac{23}{12}\rho
    \right)\alpha.
    \label{eq:appendix-distance-final}
\end{align}
Finally, by the triangle inequality,
\[
    \epsilon_1(\hf)
    =
    \normz{\hf-\eta}_1
    \le
    \epsilon_1(f_0)+\normz{\hf-f_0}_1.
\]
Combining this with \eqref{eq:appendix-distance-final} gives, for every
$f_0\in\cF$,
\begin{align*}
    \epsilon_1(\hf)
    \le{}&
    \epsilon_1(f_0)
    +(2+\rho)B_P(f_0)
    +(2+\rho)\sqrt{2v(f_0)\alpha}\\
    &+
    \left(
        \frac{13}{3}
        +\frac{1}{\rho}
        +\frac{23}{12}\rho
    \right)\alpha.
\end{align*}

The concentration event above is uniform over all pairs in $\cF$ and does not
depend on either $f_0$ or $\rho$. Hence the bound holds simultaneously for
every comparator $f_0\in\cF$ and every $\rho>0$. Taking the infimum over
$f_0\in\cF$ and substituting
$\alpha=\log(2|\cF|^2/\delta)/n$ completes the proof.
\end{proof}

\section{The finite-sample guarantee recovers the population guarantee}
\label{app:finite-sample-recovers-population}

\begin{corollary}[Optimized finite-sample bound]
\label{cor:optimized-sbr-bound}
Under the assumptions of Theorem~\ref{thm:sbr-variance-adaptive}, with
probability at least $1-\delta$, simultaneously for every $f_0\in\cF$,
\begin{equation}
    \epsilon_1(\hf)
    \le
    \epsilon_1(f_0)
    +2B_P(f_0)
    +r_n(f_0,\delta),
    \label{eq:optimized-sbr-bound}
\end{equation}
where, writing
\[
    \alpha
    :=
    \frac{\log(2|\cF|^2/\delta)}{n},
    \qquad
    s(f_0)
    :=
    \sqrt{2v(f_0)\alpha},
\]
we may take
\begin{equation}
    r_n(f_0,\delta)
    :=
    2s(f_0)
    +\frac{13}{3}\alpha
    +2\sqrt{
        \alpha
        \left(
            B_P(f_0)
            +s(f_0)
            +\frac{23}{12}\alpha
        \right)
    }.
    \label{eq:optimized-sbr-remainder}
\end{equation}
Consequently, for every fixed $f_0\in\cF$ and $\delta\in(0,1)$,
$r_n(f_0,\delta)\to0$ as $n\to\infty$.
\end{corollary}

\begin{proof}
Fix $f_0\in\cF$. On the event of
Theorem~\ref{thm:sbr-variance-adaptive}, for every $\rho>0$,
\begin{align*}
    \epsilon_1(\hf)
    \le{}&
    \epsilon_1(f_0)
    +(2+\rho)B_P(f_0)
    +(2+\rho)s(f_0) \\
    &+
    \left(
        \frac{13}{3}
        +\frac{1}{\rho}
        +\frac{23}{12}\rho
    \right)\alpha.
\end{align*}
Rearranging the $\rho$-dependent terms gives
\begin{align*}
    \epsilon_1(\hf)
    \le{}&
    \epsilon_1(f_0)
    +2B_P(f_0)
    +2s(f_0)
    +\frac{13}{3}\alpha \\
    &+
    \rho
    \left(
        B_P(f_0)
        +s(f_0)
        +\frac{23}{12}\alpha
    \right)
    +\frac{\alpha}{\rho}.
\end{align*}
Since the theorem holds simultaneously for every $\rho>0$, we may choose
\[
    \rho
    =
    \sqrt{
        \frac{\alpha}{
            B_P(f_0)
            +s(f_0)
            +\frac{23}{12}\alpha
        }
    }.
\]
Using
\[
    \inf_{\rho>0}
    \left\{
        \rho A+\frac{\alpha}{\rho}
    \right\}
    =
    2\sqrt{\alpha A},
    \qquad A>0,
\]
yields \eqref{eq:optimized-sbr-bound}--\eqref{eq:optimized-sbr-remainder}.
The same high-probability event is uniform over $f_0\in\cF$, so the result
holds simultaneously for all comparators $f_0$.

Finally, for fixed $f_0$, $\cF$, and $\delta$, we have $\alpha\to0$ and
$s(f_0)=\sqrt{2v(f_0)\alpha}\to0$. Hence every term in
\eqref{eq:optimized-sbr-remainder} converges to zero.
\end{proof}

Corollary~\ref{cor:optimized-sbr-bound} makes precise the connection between
the finite-sample result and the population guarantee in
Proposition~\ref{prop:population-sign-betting-guarantee}. For every fixed
comparator $f_0$, SBR satisfies the same leading population bound
$\epsilon_1(f_0)+2B_P(f_0)$, up to an estimation term that vanishes with the
sample size. The optimizing value of $\rho$ is used only in the analysis and
does not need to be known or supplied to SBR.

\section{Relation to the fixed-design guarantee of Baraud and Maillard}
\label{app:comparison-baraud-maillard}

We clarify here the relation between our random-design guarantees and the
fixed-design result of \citet{baraud2025estimating}. Their estimator
coincides with SBR up to normalization when the optimization is exact, but
their statistical setting and risk criterion are different. In particular,
they consider deterministic design points and control the expected $L_1$
error evaluated on those design points. Our results instead concern
high-probability population error under i.i.d.\ random design.

We record the consequence obtained by applying their theorem conditionally
on the random covariates and then using a standard uniform-convergence
argument. This provides a natural black-box baseline against which to
compare our direct random-design analysis.

\paragraph{An application of the fixed-design result.}
Suppose that $\cF$ is finite, with $M:=|\cF|$, and define
\[
    R(f)
    :=
    \epsilon_1(f)
    =
    \EE|f(X)-\eta(X)|,
    \qquad
    R_n(f)
    :=
    \frac1n\sum_{i=1}^n
    |f(X_i)-\eta(X_i)|.
\]
Let
\[
    \bar\sigma_n^2
    :=
    \frac1n\sum_{i=1}^n\sigma^2(X_i),
    \qquad
    \bar\sigma^2
    :=
    \EE[\sigma^2(X)].
\]

Condition on $X_{1:n}:=(X_1,\ldots,X_n)$. The responses then form a
fixed-design regression problem,
\[
    Y_i
    =
    \eta(X_i)+\xi_i,
    \qquad
    \EE[\xi_i\mid X_{1:n}]=0,
    \qquad
    \EE[\xi_i^2\mid X_{1:n}]
    =
    \sigma^2(X_i).
\]
Hence the $p=2$ specialization of the oracle inequality of
\citet{baraud2025estimating} applies conditionally on $X_{1:n}$.

For a finite class, the comparison-set families appearing in their
complexity measure contain at most $M$ sets, so their relevant VC
complexity is at most of order $\log M$. Consequently, their theorem
implies, for a universal constant $\kappa>0$,
\begin{equation}
    \EE\!\left[
        R_n(\hf)
        \,\middle|\,
        X_{1:n}
    \right]
    \le
    3\inf_{f\in\cF}R_n(f)
    +
    \kappa
    \sqrt{
        \frac{\bar\sigma_n^2 \log M}{n}
    }.
    \label{eq:bm-conditional-fixed-design}
\end{equation}
We suppress inessential modifications of the logarithmic complexity term.

Taking expectations over $X_{1:n}$, using
\[
    \EE\!\left[\inf_{f\in\cF}R_n(f)\right]
    \le
    \inf_{f\in\cF}\EE R_n(f)
    =
    \epsilon_{1,*},
\]
and Jensen's inequality gives
\begin{equation}
    \EE R_n(\hf)
    \le
    3\epsilon_{1,*}
    +
    \kappa
    \sqrt{
        \frac{\bar\sigma^2\log M}{n}
    }.
    \label{eq:bm-expected-empirical}
\end{equation}
Thus the fixed-design result already yields a variance-sensitive bound for
the expected error of SBR evaluated on the observed covariates.

\paragraph{A direct conversion to population risk.}
To pass from the empirical-design error to population error, define
\[
    \Delta_n
    :=
    \sup_{f\in\cF}
    |R_n(f)-R(f)|.
\]
Since $\hf\in\cF$,
\[
    R(\hf)
    \le
    R_n(\hf)+\Delta_n.
\]
Moreover, because $|f(X)-\eta(X)|\in[0,1]$, standard finite-class uniform
convergence gives
\[
    \EE[\Delta_n]
    \lesssim
    \sqrt{\frac{\log M}{n}}.
\]
Combining this with \eqref{eq:bm-expected-empirical} yields the population
bound
\begin{equation}
    \EE\epsilon_1(\hf)
    \le
    3\epsilon_{1,*}
    +
    \kappa
    \sqrt{
        \frac{\bar\sigma^2\log M}{n}
    }
    +
    C
    \sqrt{
        \frac{\log M}{n}
    }
    \label{eq:bm-naive-population-conversion}
\end{equation}
for a universal constant $C$.

The last term in
\eqref{eq:bm-naive-population-conversion} is independent of the conditional
label variance. In particular, even in the realizable noiseless case
$\epsilon_{1,*}=\bar\sigma^2=0$, this direct black-box argument leaves an
$O(\sqrt{\log(M)/n})$ population-generalization term.

This calculation highlights the additional issue created by random design:
a variance-sensitive fixed-design bound does not by itself provide a
variance-sensitive population guarantee. Our analysis does not treat the
fixed-design theorem as a black box followed by global uniform convergence.
Instead, it analyzes the signed-residual process and the
empirical-to-population transition jointly, leading to the localized,
high-probability guarantees of
Theorem~\ref{thm:sbr-variance-adaptive}. In particular, our bound retains
comparator-dependent approximation and second-order quantities and can
achieve substantially smaller statistical terms on favorable instances.

\section{A sufficient condition for favorable misspecification}
\label{app:favorable-misspecification-condition}

\begin{proposition}[A sufficient condition for favorable misspecification]
\label{prop:favorable-misspecification-condition}
Let
\[
    D_{\cF}
    :=
    \{x\in\cX:\exists f,h\in\cF,\ f(x)\neq h(x)\}
\]
be the disagreement region of $\cF$. Suppose there exists $f_\star\in\cF$
such that
\[
    f_\star(x)=\eta(x)
    \qquad
    \text{for every }x\in D_{\cF}.
\]
Then $f_\star$ is an $L_1$-optimal predictor and $B_P(f_\star)=0$.
\end{proposition}

\begin{proof}
Outside $D_{\cF}$, all predictors in $\cF$ agree, and hence every
$f\in\cF$ satisfies $f=f_\star$ on $D_{\cF}^c$. On $D_{\cF}$, we have
$f_\star=\eta$. Therefore, for every $f\in\cF$,
\[
    \epsilon_1(f)
    =
    \epsilon_1(f_\star)+\normz{f-f_\star}_1
    \ge
    \epsilon_1(f_\star),
\]
so $f_\star$ is $L_1$-optimal.

Moreover, for every $h\in\cF$, $\sign(h-f_\star)=0$ outside $D_{\cF}$,
whereas $\eta-f_\star=0$ on $D_{\cF}$. Hence
\[
    T_P(f_\star,h)
    =
    \EE[(\eta-f_\star)\sign(h-f_\star)]
    =
    0.
\]
Taking the supremum over $h\in\cF$ gives $B_P(f_\star)=0$.
\end{proof}

\section{Proofs for the sharp factor-three characterization}
\label{app:factor-three-optimality}

\subsection{Proof of the uniform factor-three guarantee}
\label{app:uniform-factor-three}

\begin{proof}[Proof of Corollary~\ref{cor:sbr-uniform-factor-three}]
Let
\[
    \alpha_\delta
    :=
    \frac{\log(2|\cF|^2/\delta)}{n},
\]
and fix
\[
    f_\star
    \in
    \argmin_{f\in\cF}\epsilon_1(f).
\]
Since $Y\in[0,1]$ and every $f\in\cF$ is $[0,1]$-valued,
\begin{equation}
    \epsilon_{1,*}\le 1,
    \qquad
    B_P(f_\star)\le \epsilon_{1,*}\le 1,
    \qquad
    v(f_\star)\le 1.
    \label{eq:uniform-factor-three-basic-bounds}
\end{equation}

Suppose first that $\alpha_\delta\le 1$. Since the event in
Theorem~\ref{thm:sbr-variance-adaptive} holds simultaneously for every
$\rho>0$, we may choose
\[
    \rho=\sqrt{\alpha_\delta}.
\]
Applying Theorem~\ref{thm:sbr-variance-adaptive} with $f_0=f_\star$ and
using \eqref{eq:uniform-factor-three-basic-bounds} gives
\begin{align*}
    \epsilon_1(\hf)
    &\le
    \epsilon_{1,*}
    +(2+\rho)B_P(f_\star)
    +(2+\rho)\sqrt{2v(f_\star)\alpha_\delta}
    +C_\rho\alpha_\delta
    \\
    &\le
    3\epsilon_{1,*}
    +\rho\epsilon_{1,*}
    +(2+\rho)\sqrt{2\alpha_\delta}
    +
    \left(
        \frac{13}{3}
        +\frac{1}{\rho}
        +\frac{23}{12}\rho
    \right)\alpha_\delta
    \\
    &\le
    3\epsilon_{1,*}
    +\sqrt{\alpha_\delta}
    +3\sqrt{2\alpha_\delta}
    +\frac{13}{3}\alpha_\delta
    +\sqrt{\alpha_\delta}
    +\frac{23}{12}\alpha_\delta^{3/2}.
\end{align*}
Because $\alpha_\delta\le1$,
\[
    \alpha_\delta
    \le
    \sqrt{\alpha_\delta},
    \qquad
    \alpha_\delta^{3/2}
    \le
    \sqrt{\alpha_\delta}.
\]
Hence
\begin{equation}
    \epsilon_1(\hf)
    \le
    3\epsilon_{1,*}
    +
    \left(
        \frac{33}{4}+3\sqrt{2}
    \right)
    \sqrt{\alpha_\delta}.
    \label{eq:uniform-factor-three-small-alpha}
\end{equation}

If instead $\alpha_\delta>1$, boundedness gives
\[
    \epsilon_1(\hf)\le1.
\]
Since
\[
    \frac{33}{4}+3\sqrt{2}>1
    \qquad\text{and}\qquad
    \sqrt{\alpha_\delta}>1,
\]
the bound in
\eqref{eq:uniform-factor-three-small-alpha} also holds trivially in this
case.

Therefore, for all $n\ge1$ and $\delta\in(0,1)$,
\begin{equation}
    \epsilon_1(\hf)
    \le
    3\epsilon_{1,*}
    +
    \left(
        \frac{33}{4}+3\sqrt{2}
    \right)
    \sqrt{
        \frac{\log(2|\cF|^2/\delta)}{n}
    }.
    \label{eq:uniform-factor-three-stronger}
\end{equation}
The estimation term in
\eqref{eq:uniform-factor-three-stronger} depends only on
$n$, $|\cF|$, and $\delta$, and hence is uniform over the underlying
distribution $P$.
\end{proof}

\subsection{Proof of Theorem~\ref{thm:proper-factor-lower-bound}}
\label{app:proper-factor-lower-bound}

\begin{proof}
Fix $C\in[1,3)$. Take $\cX=\mathbb N$,
set $a:=(C+1)/(C+5)$, and define
\[
    f_-(x)\equiv a,
    \qquad
    f_+(x)\equiv1-a,
    \qquad
    \cF_C:=\{f_-,f_+\}.
\]
Also set
\[
    q:=\frac{1}{2(1-a)}=\frac{C+5}{8},
    \qquad
    e:=\frac12-a,
    \qquad
    r:=\frac{1+a}{1-a}=\frac{C+3}{2}.
\]
Then $a\in[1/3,1/2)$, $q\in(0,1)$, $e>0$, and $r>C$.
Define $g:=(r-C)e>0$ and
\[
    \varepsilon_C
    :=\frac{g}{2}
    =\frac{(3-C)^2}{8(C+5)}.
\]
Both $\cF_C$ and $\varepsilon_C$ depend only on $C$.

Now fix $n\ge1$ and a possibly randomized proper learner with output
$\widehat f_n\in\cF_C$. All probabilities involving its output include
the learner's internal randomization.

For an integer $N$ to be chosen below, construct two priors $\Pi_-$ and
$\Pi_+$ over vectors $\theta=(\theta_1,\ldots,\theta_N)$ of conditional
means. Under each prior, the coordinates are independent, with
\[
    \begin{aligned}
        \Pi_-:\quad
        \theta_j
        &=
        \begin{cases}
            a, & \text{with probability }q,\\
            1, & \text{with probability }1-q,
        \end{cases}
        \\[1mm]
        \Pi_+:\quad
        \theta_j
        &=
        \begin{cases}
            1-a, & \text{with probability }q,\\
            0, & \text{with probability }1-q.
        \end{cases}
    \end{aligned}
\]
For each fixed $\theta$, let $P_\theta$ be the distribution defined by
\[
    X\sim\operatorname{Unif}([N]),
    \qquad
    Y\mid X=j\sim\operatorname{Bernoulli}(\theta_j).
\]
Its regression function satisfies $\eta_\theta(j)=\theta_j$ for
$j\in[N]$. For $f\in\cF_C$, write
\[
    d_\theta(f)
    :=\EE_{P_\theta}|f(X)-\eta_\theta(X)|
    =\frac1N\sum_{j=1}^N|f(j)-\theta_j|.
\]
Under $\Pi_-$, direct calculation gives
\begin{align*}
    \EE_{\Pi_-}d_\theta(f_-)
    &=(1-q)(1-a)=e,\\
    \EE_{\Pi_-}d_\theta(f_+)
    &=q(1-2a)+(1-q)a=re.
\end{align*}
By symmetry, $\EE_{\Pi_+}d_\theta(f_+)=e$ and
$\EE_{\Pi_+}d_\theta(f_-)=re$.

Let $t:=g/[4(1+C)]$. For $s\in\{-,+\}$, let $\bar s$ denote the
opposite sign and define
\[
    G_s
    :=
    \left\{
        |d_\theta(f_s)-e|\le t,\quad
        |d_\theta(f_{\bar s})-re|\le t
    \right\}.
\]
Each risk is an average of $N$ independent $[0,1]$-valued random
variables under the corresponding prior. Hoeffding's inequality and a
union bound therefore yield
\[
    \Pi_s(G_s^c)\le4\exp(-2Nt^2).
\]
On $G_s$,
\begin{align*}
    d_\theta(f_{\bar s})-C d_\theta(f_s)
    &\ge re-t-C(e+t)\\
    &=g-(1+C)t
      =\frac{3g}{4}
      >\varepsilon_C.
\end{align*}
Since $C\ge1$, this also implies
$d_\theta(f_{\bar s})>d_\theta(f_s)$. Thus $f_s$ is the unique
$L_1$-optimal member of $\cF_C$ on $G_s$, and
\begin{equation}
    d_\theta(f_{\bar s})
    >
    C\min_{f\in\cF_C}d_\theta(f)+\varepsilon_C.
    \label{eq:lower-bound-wrong-choice}
\end{equation}

We next show that the two priors induce nearly indistinguishable sample
distributions. Let
\[
    Q_s
    :=\int P_\theta^{\otimes n}\,\Pi_s(d\theta),
    \qquad s\in\{-,+\},
\]
be the mixture law of the entire training sample. In particular, $Q_s$
is a mixture of product distributions, not a product of
single-observation mixtures. Under both priors,
\[
    \EE_{\Pi_-}\theta_j=qa+(1-q)=\frac12,
    \qquad
    \EE_{\Pi_+}\theta_j=q(1-a)=\frac12.
\]
Let $D:=\{X_1,\ldots,X_n\text{ are all distinct}\}$. Conditional on
any distinct input sequence, the sampled coordinates
$\theta_{X_1},\ldots,\theta_{X_n}$ are independent under either prior.
After integrating out these coordinates, the conditional label law is
$\operatorname{Bernoulli}(1/2)^{\otimes n}$ under both $Q_-$ and $Q_+$.
The input law is also identical under the two mixtures. Hence their
restrictions to $D$ coincide.

Writing
$\operatorname{TV}(Q_-,Q_+):=\sup_A|Q_-(A)-Q_+(A)|$, we obtain
\[
    \operatorname{TV}(Q_-,Q_+)
    \le Q_-(D^c)
    \le\sum_{1\le i<j\le n}\PP(X_i=X_j)
    =\frac{n(n-1)}{2N},
\]
where the collision probabilities are computed under the common input
law. Choose an integer
\[
    N\ge
    \max\left\{
        1,\,
        6n(n-1),\,
        \frac{\log48}{2t^2}
    \right\}.
\]
The preceding bounds then give
\begin{equation}
    \operatorname{TV}(Q_-,Q_+)\le\frac1{12},
    \qquad
    \Pi_s(G_s^c)\le\frac1{12}
    \quad(s\in\{-,+\}).
    \label{eq:lower-bound-tv-and-good}
\end{equation}

Let $\mathbb P_s$ denote the joint experiment that first draws
$\theta\sim\Pi_s$, then draws the sample
$S\sim P_\theta^{\otimes n}$, and includes the learner's internal
randomization. For each sample $S$, let $\phi(S)\in[0,1]$ be the
probability, over this internal randomization, that the learner returns
$f_+$. Properness implies that it returns $f_-$ with probability
$1-\phi(S)$. Consequently,
\begin{align*}
    \mathbb P_-(\widehat f_n=f_+)
    +\mathbb P_+(\widehat f_n=f_-)
    &=
    \int\phi\,dQ_-+\int(1-\phi)\,dQ_+\\
    &=
    1+\int\phi\,d(Q_--Q_+)\\
    &\ge
    1-\operatorname{TV}(Q_-,Q_+)\\
    &\ge\frac{11}{12}.
\end{align*}
The first inequality follows from $0\le\phi\le1$ and the definition of
total variation; for example, write
$\phi(S)=\int_0^1\mathbf1\{\phi(S)>u\}\,du$.
Removing the bad environments and using
\eqref{eq:lower-bound-tv-and-good} yields
\begin{equation}
    \begin{aligned}
        \frac12\sum_{s\in\{-,+\}}
        \mathbb P_s\!\left(
            \widehat f_n=f_{\bar s},\,G_s
        \right)
        &\ge
        \frac{11}{24}
        -\frac12\sum_{s\in\{-,+\}}\Pi_s(G_s^c)\\
        &\ge
        \frac{11}{24}-\frac1{12}
        =\frac38.
    \end{aligned}
    \label{eq:lower-bound-good-testing}
\end{equation}

The left-hand side of \eqref{eq:lower-bound-good-testing} equals
\[
    \frac12\sum_{s\in\{-,+\}}
    \int
        \mathbf1_{G_s}(\theta)\,
        \PP_{S\sim P_\theta^{\otimes n}}
        \!\left(\widehat f_n=f_{\bar s}\right)
    \,\Pi_s(d\theta).
\]
Since both priors have finite support, this is a finite weighted
average. It follows that there exist $s\in\{-,+\}$ and a fixed
$\theta\in G_s$ such that
\[
    \PP_{S\sim P_\theta^{\otimes n}}
    \!\left(\widehat f_n=f_{\bar s}\right)
    \ge\frac38.
\]
For this $\theta$, \eqref{eq:lower-bound-wrong-choice} implies
\[
    \PP_{S\sim P_\theta^{\otimes n}}
    \left(
        d_\theta(\widehat f_n)
        >
        C\min_{f\in\cF_C}d_\theta(f)+\varepsilon_C
    \right)
    \ge\frac38.
\]
Taking $P=P_\theta$ proves
\eqref{eq:proper-factor-lower-bound}. The final counterexample is the
fixed distribution $P_\theta$, not either sample mixture $Q_s$;
its observations are i.i.d.\ and its input marginal has finite support.
\end{proof}


The construction above uses the same input marginal
$P_X=\Unif([N])$ in the two hard experiments. Hence the proof in fact
continues to hold even if the learner is given $P_X$.

\subsection{Impossibility of sub-three uniform guarantees}
\label{app:sub-three-uniform-guarantees}

We now derive the impossibility of
\eqref{eq:sub-three-uniform-guarantee}.
Fix $C\in[1,3)$, and let $\cF_C$ and $\varepsilon_C>0$ be supplied by
Theorem~\ref{thm:proper-factor-lower-bound}.
Suppose, toward a contradiction, that a possibly randomized proper
learning rule satisfies
\begin{equation}
    \PP_{S\sim P^{\otimes n}}
    \left(
        \epsilon_1(\widehat f_n)
        \le
        C\epsilon_{1,*}+r_n(\cF_C,\delta)
    \right)
    \ge1-\delta
    \label{eq:appendix-sub-three-guarantee}
\end{equation}
for every $n\ge1$, every $\delta\in(0,1)$, and every distribution $P$
on $\cX\times[0,1]$. Here the errors are evaluated under $P$ relative
to $\cF_C$, and the learner may depend on $\delta$.
Assume that the estimation term is deterministic, independent of $P$,
and satisfies $r_n(\cF_C,\delta)\to0$ for each fixed $\delta$.

Set $\delta=1/4$. Since $\cF_C$ and $\varepsilon_C$ are fixed
independently of $n$, there exists $n$ sufficiently large that
$r_n(\cF_C,1/4)\le\varepsilon_C$. For this $n$,
\eqref{eq:appendix-sub-three-guarantee} implies, for every $P$,
\begin{align*}
    &\PP_{S\sim P^{\otimes n}}
    \left(
        \epsilon_1(\widehat f_n)
        >
        C\epsilon_{1,*}+\varepsilon_C
    \right)\\
    &\qquad\le
    \PP_{S\sim P^{\otimes n}}
    \left(
        \epsilon_1(\widehat f_n)
        >
        C\epsilon_{1,*}+r_n(\cF_C,1/4)
    \right)
    \le\frac14.
\end{align*}
This contradicts Theorem~\ref{thm:proper-factor-lower-bound}, which
supplies a distribution $P$ for which the first probability is at
least $3/8$. The same contradiction applies if the proposed guarantee
is required only for all sufficiently large $n$, provided the
sample-size threshold is independent of $P$.

For $C<1$, the argument is trivial. Take the singleton class $\cF=\{f\equiv0\}$ and a
distribution with $X\equiv1$ and $Y\equiv1$. Every proper learner then
satisfies $\epsilon_1(\widehat f_n)=\epsilon_{1,*}=1$ deterministically.
Thus \eqref{eq:sub-three-uniform-guarantee} fails with probability one
whenever $r_n(\cF,1/4)<1-C$, which holds for all sufficiently large $n$
if $r_n(\cF,1/4)\to0$.

Therefore, for every $C\in[0,3)$, no proper learner can satisfy a
distribution-free high-probability oracle inequality of the form
\[
    \epsilon_1(\widehat f_n)
    \le
    C\epsilon_{1,*}
    +
    r_n(\cF,\delta)
\]
with a distribution-independent estimation term satisfying
$r_n(\cF,\delta)\to0$ for each fixed $\cF$ and $\delta$. In particular,
the conclusion is unaffected by how slowly the estimation term vanishes
with $n$.

Together with
Corollary~\ref{cor:sbr-uniform-factor-three}, which establishes such a
uniform guarantee with approximation factor $3$, this proves that $3$ is
the smallest universal approximation factor achievable by proper learners
under the distribution-free, high-probability criterion considered in
Section~\ref{sec:optimality}.

\section{Beyond finite classes: proof and examples}
\label{app:sbr-infinite-classes}

This appendix proves Theorem~\ref{thm:sbr-infinite-classes} and verifies
its complexity conditions for several standard infinite hypothesis classes.
The proof has two ingredients. First, prediction and comparison complexity
control the metric entropy of the empirical processes induced by SBR.
Second, a uniform Bernstein inequality converts these entropy bounds into
variance-sensitive concentration.

Throughout this appendix, covering numbers are taken with centers in the class. We suppress standard measurability issues throughout. 

\subsection{Prediction and comparison complexity}
\label{app:sbr-complexity}

Recall the pairwise comparison class
\[
    \cC_{\cF}
    :=
    \left\{
        x\mapsto\one\{h(x)>f(x)\}:f,h\in\cF
    \right\},
\]
and let $d_{\rm pred}:=\operatorname{Pdim}(\cF)$,
$d_{\rm cmp}:=\operatorname{VCdim}(\cC_{\cF})$, and
$d:=1+d_{\rm pred}+d_{\rm cmp}$. For $f,h\in\cF$, define
\[
    Z_{f,h}(x,y)
    :=
    (y-f(x))\sign(h(x)-f(x)),
    \qquad
    D_{f,h}(x):=|f(x)-h(x)|,
\]
and introduce the corresponding function classes
\[
    \cZ:=\{Z_{f,h}:f,h\in\cF\},
    \qquad
    \cD:=\{D_{f,h}:f,h\in\cF\}.
\]

The next proposition shows that the two complexity measures in
Theorem~\ref{thm:sbr-infinite-classes} imply uniform polynomial covering bounds for the function classes needed in the proof.

\begin{proposition}[Prediction and comparison complexity imply uniform entropy]
\label{prop:sbr-derived-class-entropy}
Suppose $\cF\subseteq\{f:\cX\to[0,1]\}$ and
$d_{\rm pred},d_{\rm cmp}<\infty$. There exist universal constants
$A_0\ge e$ and $c_0>0$ such that, for every probability measure $Q$ on
$\cX\times[0,1]$ and every $u\in(0,1]$,
\[
    \log N\bigl(u,\cZ,L_2(Q)\bigr)
    \le
    c_0d\log(A_0/u),
\]
and
\[
    \log N\bigl(u,\cD,L_2(Q_X)\bigr)
    \le
    c_0d\log(A_0/u),
\]
where $Q_X$ denotes the $\cX$-marginal of $Q$.
\end{proposition}

\begin{proof}
We first record the standard entropy bounds that will be used throughout.
Since $\operatorname{Pdim}(\cF)=d_{\rm pred}$, the class $\cF$ is a
VC-subgraph class with VC index at most $d_{\rm pred}+1$ (the shift by
one is only a matter of convention). Because all functions in $\cF$ take
values in $[0,1]$, Theorem~2.6.7 of
\citet{wellner2013weak} therefore implies that there exist
universal constants $a_0\ge e$ and $c_1>0$ such that, for every
probability measure $R$ on $\cX$ and every $u\in(0,1]$,
\begin{equation}
    \log N\bigl(u,\cF,L_2(R)\bigr)
    \le
    c_1(1+d_{\rm pred})\log(a_0/u).
    \label{eq:entropy-F}
\end{equation}
Likewise, since $\cC_{\cF}$ is a $\{0,1\}$-valued class of VC
dimension $d_{\rm cmp}$, the uniform entropy bound for VC classes of
sets \citep[Theorem~2.6.4]{wellner2013weak} gives
\begin{equation}
    \log N\bigl(u,\cC_{\cF},L_2(R)\bigr)
    \le
    c_1(1+d_{\rm cmp})\log(a_0/u).
    \label{eq:entropy-C}
\end{equation}
The precise universal constants in
\eqref{eq:entropy-F}--\eqref{eq:entropy-C} will not be important.

We first bound the entropy of $\cZ$.  For $f,h\in\cF$, let
\[
    c_{h,f}(x):=\one\{h(x)>f(x)\}.
\]
Under the convention $\sign(0)=0$,
\[
    \sign(h-f)=c_{h,f}-c_{f,h},
\]
and hence every $Z_{f,h}\in\cZ$ is contained in the larger product
class
\[
    \cH
    :=
    \left\{
        (x,y)\mapsto
        (y-f(x))(c_1(x)-c_2(x)):
        f\in\cF,\;
        c_1,c_2\in\cC_{\cF}
    \right\}.
\]

The map defining $\cH$ is Lipschitz in its three functional arguments.
Indeed, for $f,f'\in\cF$ and
$c_1,c_2,c_1',c_2'\in\cC_{\cF}$,
\begin{align}
    &\left|
        (y-f)(c_1-c_2)
        -(y-f')(c_1'-c_2')
    \right|
    \nonumber\\
    &\quad\le
    |f-f'|\,|c_1-c_2|
    +
    |y-f'|
    \bigl(
        |c_1-c_1'|+|c_2-c_2'|
    \bigr)
    \nonumber\\
    &\quad\le
    |f-f'|+|c_1-c_1'|+|c_2-c_2'|,
    \label{eq:Z-pointwise-lipschitz}
\end{align}
where the last inequality uses
$f',y\in[0,1]$ and $c_1,c_2\in\{0,1\}$.

Fix a probability measure $Q$ on $\cX\times[0,1]$, and let $Q_X$
denote its $\cX$-marginal. Set $\delta:=u/6$.  Choose
$\delta$-covers of $\cF$ and $\cC_{\cF}$ in $L_2(Q_X)$.
By \eqref{eq:Z-pointwise-lipschitz} and Minkowski's inequality,
approximating each of the three components within $\delta$ produces an
$L_2(Q)$ error of at most
\[
    3\delta=\frac u2.
\]
Thus the corresponding product functions form an external $u/2$-cover
of $\cZ$. If covering centers are required to belong to $\cZ$, then
for every covering ball that intersects $\cZ$ choose one element of
that intersection as its new center. By the triangle inequality this
produces an internal $u$-cover with no increase in cardinality.
Consequently,
\begin{equation}
    N\bigl(u,\cZ,L_2(Q)\bigr)
    \le
    N\bigl(u/6,\cF,L_2(Q_X)\bigr)
    N\bigl(u/6,\cC_{\cF},L_2(Q_X)\bigr)^2.
    \label{eq:Z-product-cover}
\end{equation}

Applying \eqref{eq:entropy-F}--\eqref{eq:entropy-C} to
\eqref{eq:Z-product-cover} gives
\begin{align*}
    \log N\bigl(u,\cZ,L_2(Q)\bigr)
    &\le
    c_1\Bigl[
        (1+d_{\rm pred})
        +2(1+d_{\rm cmp})
    \Bigr]\log(6a_0/u)\\
    &\le
    3c_1d\log(6a_0/u),
\end{align*}
since $d=1+d_{\rm pred}+d_{\rm cmp}$.

It remains to control $\cD$.  The absolute-value map satisfies
\[
    \bigl||f-h|-|f'-h'|\bigr|
    \le
    |f-f'|+|h-h'|.
\]
Therefore, taking $u/2$-covers of the two copies of $\cF$ and applying
Minkowski's inequality yields
\begin{equation}
    N\bigl(u,\cD,L_2(Q_X)\bigr)
    \le
    N\bigl(u/2,\cF,L_2(Q_X)\bigr)^2.
    \label{eq:D-product-cover}
\end{equation}
Here the resulting centers are themselves of the form
$|f'-h'|$ and hence belong to $\cD$ when the covers of $\cF$ are
chosen internally.

Using \eqref{eq:entropy-F} once more,
\[
    \log N\bigl(u,\cD,L_2(Q_X)\bigr)
    \le
    2c_1(1+d_{\rm pred})\log(2a_0/u)
    \le
    2c_1d\log(2a_0/u).
\]

Choosing universal constants, for example
$A_0\ge \max\{e,6a_0\}$ and $c_0\ge 3c_1$, proves both asserted
bounds.
\end{proof}

\subsection{A uniform Bernstein inequality}
\label{app:uniform-bernstein}

We next give the empirical-process inequality used to replace the finite
union bound in the proof of Theorem~\ref{thm:sbr-variance-adaptive}.
The important feature is that the leading deviation term depends on the individual second moment $Pg^2$.

\begin{lemma}[Uniform Bernstein inequality]
\label{lem:uniform-bernstein}
Let $W_1,\ldots,W_n$ be i.i.d.\ from a probability measure $P$, and let
$\cG$ be a class of measurable functions $g:\cW\to[-1,1]$. Suppose
there exist $A\ge e$ and $\nu\ge1$ such that, for every probability
measure $Q$ on $\cW$ and every $u\in(0,1]$,
\begin{equation}
    \log N\bigl(u,\cG,L_2(Q)\bigr)
    \le
    \nu\log(A/u).
    \label{eq:uniform-bernstein-entropy}
\end{equation}
For $\zeta\in(0,1)$, define
\[
    \beta
    :=
    \frac{\nu\log(An)+\log(2/\zeta)}{n}.
\]
Then, with probability at least $1-\zeta$,
\begin{equation}
    |(P_n-P)g|
    \le
    33\left(
        \sqrt{Pg^2\,\beta}+\beta
    \right)
    \qquad
    \text{simultaneously for all }g\in\cG.
    \label{eq:uniform-bernstein}
\end{equation}
\end{lemma}

\begin{proof}
The proof has two steps. We first derive a uniform deviation bound for a
fixed tuning parameter $\lambda$, and then optimize $\lambda$ on a finite
grid.

\paragraph{Step 1: A uniform bound for fixed $\lambda$.}

We begin with the following elementary inequality: for every
$a,b\in[-1,1]$ and $0<\lambda\le1/8$,
\begin{equation}
    \frac12
    e^{\lambda(a-b)-8\lambda^2b^2}
    +
    \frac12
    e^{\lambda(b-a)-8\lambda^2a^2}
    \le1.
    \label{eq:swap-mgf-elementary}
\end{equation}
Indeed, the left-hand side equals
\[
    e^{-4\lambda^2(a^2+b^2)}
    \cosh\!\left(
        \lambda(a-b)+4\lambda^2(a^2-b^2)
    \right).
\]
Since $\cosh(t)\le e^{t^2/2}$ and
\[
    \left|
        \lambda(a-b)+4\lambda^2(a^2-b^2)
    \right|
    =
    \lambda|a-b|\,|1+4\lambda(a+b)|
    \le
    2\lambda|a-b|,
\]
the preceding display is at most
\[
    \exp\!\left(
        -4\lambda^2(a^2+b^2)
        +2\lambda^2(a-b)^2
    \right)
    =
    e^{-2\lambda^2(a+b)^2}
    \le1.
\]

Now fix $0<\lambda\le1/8$ and $\delta\in(0,1)$, and write
\[
    H:=\nu\log(An).
\]
We claim that, with probability at least $1-\delta$,
\begin{equation}
    |(P_n-P)g|
    \le
    8\lambda Pg^2
    +
    \frac{H+1+\log(2/\delta)}{\lambda n}
    \qquad
    \text{for every }g\in\cG.
    \label{eq:fixed-lambda-uniform-bound}
\end{equation}

Let $W_1',\ldots,W_n'$ be an independent copy of the sample. This is the
usual ghost-sample symmetrization device; see, e.g.,
\citet[Section~2.3.1]{wellner2013weak}. Define
\[
    D_\lambda(g)
    :=
    \lambda\sum_{i=1}^n
    \bigl(g(W_i)-g(W_i')\bigr)
    -
    8\lambda^2\sum_{i=1}^n g(W_i')^2
\]
and
\[
    Z_\lambda
    :=
    \sup_{g\in\cG}
    \left\{
        \lambda n(P_n-P)g
        -8\lambda^2nPg^2
    \right\}.
\]
Conditional on the original sample,
\[
    \EE_{W'}D_\lambda(g)
    =
    \lambda n(P_n-P)g
    -8\lambda^2nPg^2.
\]
Hence Jensen's inequality gives
\begin{align}
    e^{Z_\lambda}
    &=
    \sup_{g\in\cG}
    \exp\!\left\{
        \EE_{W'}D_\lambda(g)
    \right\}
    \nonumber\\
    &\le
    \sup_{g\in\cG}
    \EE_{W'}e^{D_\lambda(g)}
    \nonumber\\
    &\le
    \EE_{W'}
    \sup_{g\in\cG}e^{D_\lambda(g)}.
    \label{eq:ghost-jensen}
\end{align}
Taking expectation over the original sample,
\begin{equation}
    \EE e^{Z_\lambda}
    \le
    \EE_{W,W'}
    \left[
        \sup_{g\in\cG}e^{D_\lambda(g)}
    \right].
    \label{eq:ghost-mgf}
\end{equation}

Fix the two samples and let
\[
    Q
    :=
    \frac1{2n}
    \sum_{i=1}^n
    \bigl(
        \delta_{W_i}+\delta_{W_i'}
    \bigr)
\]
be their pooled empirical measure. By
\eqref{eq:uniform-bernstein-entropy}, there exists an $L_2(Q)$-cover
$\{h_1,\ldots,h_M\}\subseteq\cG$ of radius $1/n$ satisfying
\[
    M
    \le
    (An)^\nu
    =
    e^H.
\]
Thus, for every $g\in\cG$, some $h_j$ satisfies
\begin{equation}
    \|g-h_j\|_{L_2(Q)}
    \le\frac1n,
    \qquad
    Q|g-h_j|
    \le\frac1n,
    \label{eq:pooled-cover}
\end{equation}
where the second inequality follows from Cauchy--Schwarz.

Next, independently for each $i$, swap $W_i$ and $W_i'$ with probability
$1/2$. Denote the resulting ordered pair by
$(\widetilde W_i,\widetilde W_i')$, and define
\[
    D_\lambda^\varepsilon(g)
    :=
    \lambda\sum_{i=1}^n
    \bigl(
        g(\widetilde W_i)-g(\widetilde W_i')
    \bigr)
    -
    8\lambda^2
    \sum_{i=1}^n
    g(\widetilde W_i')^2.
\]
The pooled empirical measure $Q$ is unchanged by the swaps, so the same
cover applies to every swap configuration.

Let $h_j$ satisfy \eqref{eq:pooled-cover}. Since
$|g^2-h_j^2|\le2|g-h_j|$, we have
\begin{align}
    \left|
        D_\lambda^\varepsilon(g)
        -D_\lambda^\varepsilon(h_j)
    \right|
    &\le
    \lambda\sum_{i=1}^n
    \left(
        |g-h_j|(\widetilde W_i)
        +
        |g-h_j|(\widetilde W_i')
    \right)
    \nonumber\\
    &\quad
    +
    16\lambda^2
    \sum_{i=1}^n
    |g-h_j|(\widetilde W_i')
    \nonumber\\
    &\le
    (2\lambda+32\lambda^2)nQ|g-h_j|
    \nonumber\\
    &\le
    2\lambda+32\lambda^2
    \le1.
    \label{eq:swap-cover-approximation}
\end{align}
Consequently,
\begin{equation}
    \sup_{g\in\cG}e^{D_\lambda^\varepsilon(g)}
    \le
    e\sum_{j=1}^M
    e^{D_\lambda^\varepsilon(h_j)}.
    \label{eq:swap-cover-sum}
\end{equation}

Conditional on the two samples, fix one cover element $h_j$. For the
$i$th pair, averaging over the two possible swap configurations gives
\[
    \frac12
    e^{\lambda(a-b)-8\lambda^2b^2}
    +
    \frac12
    e^{\lambda(b-a)-8\lambda^2a^2},
    \qquad
    a=h_j(W_i),\quad b=h_j(W_i'),
\]
which is at most $1$ by \eqref{eq:swap-mgf-elementary}. Since the swaps
are independent across $i$,
\begin{equation}
    \EE_\varepsilon
    e^{D_\lambda^\varepsilon(h_j)}
    \le1.
    \label{eq:swap-mgf-cover-center}
\end{equation}
Taking expectation in \eqref{eq:swap-cover-sum} therefore yields
\begin{equation}
    \EE_\varepsilon
    \left[
        \sup_{g\in\cG}
        e^{D_\lambda^\varepsilon(g)}
    \right]
    \le
    eM
    \le
    e^{H+1}.
    \label{eq:swap-mgf-uniform}
\end{equation}

Finally, since each pair $(W_i,W_i')$ consists of two i.i.d.
observations, independently swapping its coordinates does not change the
joint distribution. Hence
\[
    \EE_{W,W'}
    \left[
        \sup_{g\in\cG}e^{D_\lambda(g)}
    \right]
    =
    \EE_{W,W',\varepsilon}
    \left[
        \sup_{g\in\cG}e^{D_\lambda^\varepsilon(g)}
    \right]
    \le
    e^{H+1}.
\]
Together with \eqref{eq:ghost-mgf}, this gives
\begin{equation}
    \EE e^{Z_\lambda}
    \le
    e^{H+1}.
    \label{eq:Z-lambda-mgf}
\end{equation}

By Markov's inequality,
\[
    \Pr\!\left(
        Z_\lambda
        >
        H+1+\log(2/\delta)
    \right)
    \le
    \frac{\delta}{2}.
\]
Thus, except on an event of probability at most $\delta/2$,
\[
    (P_n-P)g
    \le
    8\lambda Pg^2
    +
    \frac{H+1+\log(2/\delta)}{\lambda n}
    \qquad
    \forall g\in\cG.
\]
Applying the same argument to $-\cG$ and taking a union bound proves
\eqref{eq:fixed-lambda-uniform-bound}.

\paragraph{Step 2: Simultaneously optimize $\lambda$.}

Consider the finite grid
\[
    \Lambda
    :=
    \left\{
        \frac{j}{8n}:j=1,\ldots,n
    \right\}.
\]
Apply \eqref{eq:fixed-lambda-uniform-bound} to each
$\lambda\in\Lambda$ with $\delta=\zeta/n$, and take a union bound.
With probability at least $1-\zeta$,
\begin{equation}
    |(P_n-P)g|
    \le
    8\lambda Pg^2
    +
    \frac{H+1+\log(2n/\zeta)}{\lambda n}
    \qquad
    \forall g\in\cG,\quad
    \forall\lambda\in\Lambda.
    \label{eq:lambda-grid-uniform}
\end{equation}

By definition,
\[
    n\beta
    =
    H+\log(2/\zeta).
\]
Moreover, since $A\ge e$ and $\nu\ge1$,
\[
    H=\nu\log(An)\ge1+\log n.
\]
Therefore
\[
    H+1+\log(2n/\zeta)
    =
    n\beta+1+\log n
    \le
    2n\beta.
\]
Hence, on the same event,
\begin{equation}
    |(P_n-P)g|
    \le
    8\lambda Pg^2
    +
    \frac{2\beta}{\lambda}
    \qquad
    \forall g\in\cG,\quad
    \forall\lambda\in\Lambda.
    \label{eq:lambda-grid-relative}
\end{equation}

Fix $g\in\cG$ and write
\[
    v:=Pg^2\in[0,1].
\]

If $v\le\beta$, choose $\lambda=1/8$ in
\eqref{eq:lambda-grid-relative}. Then
\[
    |(P_n-P)g|
    \le
    v+16\beta
    \le
    17\beta.
\]

Now suppose $v>\beta$, and define the continuous optimizer
\[
    \lambda_\star
    :=
    \frac18\sqrt{\frac{\beta}{v}}.
\]
Then $\lambda_\star<1/8$. Also,
$\beta\ge1/n$ and $v\le1$ imply
\[
    \lambda_\star
    \ge
    \frac{1}{8\sqrt n}
    \ge
    \frac{1}{8n}.
\]
Let $\lambda$ be the largest element of $\Lambda$ satisfying
$\lambda\le\lambda_\star$. Since the grid spacing is $1/(8n)$,
\[
    \frac{\lambda_\star}{2}
    \le
    \lambda
    \le
    \lambda_\star.
\]
Using \eqref{eq:lambda-grid-relative},
\begin{align*}
    |(P_n-P)g|
    &\le
    8\lambda_\star v
    +
    \frac{4\beta}{\lambda_\star}
    \\
    &=
    \sqrt{v\beta}
    +
    32\sqrt{v\beta}
    \\
    &=
    33\sqrt{v\beta}.
\end{align*}

Combining the two cases yields
\[
    |(P_n-P)g|
    \le
    33\left(
        \sqrt{Pg^2\,\beta}+\beta
    \right)
    \qquad
    \forall g\in\cG.
\]
Since the event in \eqref{eq:lambda-grid-relative} is common to all
$g\in\cG$, this proves the result.
\end{proof}

\subsection{Proof of Theorem~\ref{thm:sbr-infinite-classes}}
\label{app:proof-sbr-infinite}

We now combine Proposition~\ref{prop:sbr-derived-class-entropy} and
Lemma~\ref{lem:uniform-bernstein} with the same pairwise argument used
for the finite-class result.

\begin{proof}[Proof of Theorem~\ref{thm:sbr-infinite-classes}]
Let
\[
    \widehat L(f)
    :=
    \sup_{h\in\cF}\widehat T(f,h).
\]
By Proposition~\ref{prop:sbr-derived-class-entropy}, the classes $\cZ$
and $\cD$ satisfy the entropy condition in
Lemma~\ref{lem:uniform-bernstein} with parameters of order $d$.
Applying Lemma~\ref{lem:uniform-bernstein} separately to $\cZ$ and
$\cD$, each with failure probability $\delta/2$, and increasing the
universal constant $c$ in the definition of $\alpha_\delta$ if
necessary, we obtain an event of probability at least $1-\delta$ on
which, simultaneously for every $f,h\in\cF$,
\begin{align}
    (P_n-P)Z_{f,h}
    &\le
    \sqrt{2PZ_{f,h}^2\,\alpha_\delta}
    +\frac43\alpha_\delta,
    \label{eq:infinite-Z-concentration}\\
    (P-P_n)D_{f,h}
    &\le
    \sqrt{2PD_{f,h}\,\alpha_\delta}
    +\frac23\alpha_\delta.
    \label{eq:infinite-D-concentration}
\end{align}
For \eqref{eq:infinite-D-concentration}, we used
$PD_{f,h}^2\le PD_{f,h}$, since $0\le D_{f,h}\le1$.
We work on this event for the remainder of the proof.

Fix any comparator $f_0\in\cF$. By
\eqref{eq:infinite-Z-concentration}, for every $h\in\cF$,
\[
    P_nZ_{f_0,h}
    \le
    PZ_{f_0,h}
    +\sqrt{2PZ_{f_0,h}^2\,\alpha_\delta}
    +\frac43\alpha_\delta.
\]
Since $PZ_{f_0,h}\le B_P(f_0)$ and
$PZ_{f_0,h}^2\le v(f_0)$, taking the supremum over $h$ yields
\begin{equation}
    \widehat L(f_0)
    \le
    B_P(f_0)
    +\sqrt{2v(f_0)\alpha_\delta}
    +\frac43\alpha_\delta.
    \label{eq:infinite-comparator-loss}
\end{equation}

The empirical analogue of
Lemma~\ref{lem:sign-betting-pairwise-identity} holds exactly:
\[
    \widehat T(f,h)+\widehat T(h,f)
    =
    P_n|f-h|.
\]
Hence, by the empirical optimality of $\hf$,
\begin{align}
    P_n|\hf-f_0|
    &=
    \widehat T(\hf,f_0)
    +\widehat T(f_0,\hf)
    \nonumber\\
    &\le
    \widehat L(\hf)+\widehat L(f_0)
    \le
    2\widehat L(f_0).
    \label{eq:infinite-empirical-distance}
\end{align}

Let $r:=\normz{\hf-f_0}_1$. Since
\eqref{eq:infinite-D-concentration} holds uniformly over all pairs, it
applies to the data-dependent pair $(\hf,f_0)$:
\[
    r
    \le
    P_n|\hf-f_0|
    +\sqrt{2r\alpha_\delta}
    +\frac23\alpha_\delta.
\]
Combining this inequality with
\eqref{eq:infinite-comparator-loss} and
\eqref{eq:infinite-empirical-distance} gives
\begin{equation}
    r
    \le
    2B_P(f_0)
    +2\sqrt{2v(f_0)\alpha_\delta}
    +\sqrt{2r\alpha_\delta}
    +\frac{10}{3}\alpha_\delta.
    \label{eq:infinite-scalar-inequality}
\end{equation}

Fix $\rho>0$ and let $\theta:=\rho/(2+\rho)$. By Young's inequality,
\[
    \sqrt{2r\alpha_\delta}
    \le
    \theta r+\frac{\alpha_\delta}{2\theta}.
\]
Substituting this into \eqref{eq:infinite-scalar-inequality} and using
$1/(1-\theta)=(2+\rho)/2$ gives
\[
    r
    \le
    (2+\rho)
    \left(
        B_P(f_0)+\sqrt{2v(f_0)\alpha_\delta}
    \right)
    +C_\rho\alpha_\delta,
\]
where $C_\rho$ is the same constant as in
Theorem~\ref{thm:sbr-variance-adaptive}.

Finally, by the triangle inequality,
\[
    \epsilon_1(\hf)
    \le
    \epsilon_1(f_0)+\normz{\hf-f_0}_1.
\]
The concentration event is independent of $f_0$ and $\rho$, so the
resulting inequality holds simultaneously for every $f_0\in\cF$ and
every $\rho>0$. Taking the infimum over $f_0$ proves
\eqref{eq:sbr-infinite-bound}.
\end{proof}

\subsection{Examples of infinite hypothesis classes}
\label{app:sbr-parametric-examples}

We conclude by verifying the complexity conditions in
Theorem~\ref{thm:sbr-infinite-classes} for several standard prediction
classes. We use standard results on the VC dimension of linear threshold
classes and the pseudo-dimension of real-valued linear classes; see, e.g.,
\citet{vapnik1968uniform,anthony2009neural,
shalev-schwartz14understanding}. Recall that the relevant complexity is
$d=1+d_{\rm pred}+d_{\rm cmp}$.

\paragraph{Bounded linear prediction classes.}
Let $\phi:\cX\to\mathbb R^m$ be a fixed measurable feature map and
$a:\cX\to\mathbb R$ a fixed offset. Consider
\[
    \cF_{\rm lin}
    :=
    \left\{
        f_\theta(x)
        =
        a(x)+\langle\theta,\phi(x)\rangle:
        \theta\in\Theta
    \right\}
    \subseteq
    \{f:\cX\to[0,1]\},
\]
where $\Theta\subseteq\mathbb R^m$ is any parameter set for which the
predictions lie in $[0,1]$.

For a subgraph threshold $t$,
\[
    f_\theta(x)>t
    \quad\Longleftrightarrow\quad
    \left\langle
        (\theta,1),(\phi(x),a(x)-t)
    \right\rangle>0.
\]
Thus the subgraph class is a subclass of the class of homogeneous
halfspaces in $\mathbb R^{m+1}$, and hence
\[
    d_{\rm pred}\le m+1.
\]
Moreover,
\[
    f_\gamma(x)>f_\theta(x)
    \quad\Longleftrightarrow\quad
    \langle\gamma-\theta,\phi(x)\rangle>0.
\]
Therefore $\cC_{\cF_{\rm lin}}$ is a subclass of the class of
homogeneous halfspaces in $\mathbb R^m$, and hence
\[
    d_{\rm cmp}\le m.
\]
Consequently, $d=O(m)$, and
Theorem~\ref{thm:sbr-infinite-classes} applies with
\[
    \alpha_\delta
    \lesssim
    \frac{m\log(en)+\log(2/\delta)}{n}.
\]

\paragraph{Logistic and other monotone-link prediction classes.}
Let $\ell:\mathbb R\to(0,1)$ be a fixed strictly increasing bijection,
and consider
\[
    \cF_{\ell}
    :=
    \left\{
        f_\theta(x)
        =
        \ell\!\left(
            a(x)+\langle\theta,\phi(x)\rangle
        \right):
        \theta\in\Theta
    \right\}.
\]
This includes logistic prediction with
$\ell(z)=(1+e^{-z})^{-1}$ and probit prediction with
$\ell(z)=\Phi(z)$. Since $\ell$ is strictly increasing, for every
$t\in(0,1)$,
\[
    f_\theta(x)>t
    \quad\Longleftrightarrow\quad
    \left\langle
        (\theta,1),
        (\phi(x),a(x)-\ell^{-1}(t))
    \right\rangle>0.
\]
Thus the subgraph class is again a subclass of the class of homogeneous
halfspaces in $\mathbb R^{m+1}$, yielding
\[
    d_{\rm pred}\le m+1.
\]
Similarly,
\[
    f_\gamma(x)>f_\theta(x)
    \quad\Longleftrightarrow\quad
    \langle\gamma-\theta,\phi(x)\rangle>0,
\]
so $\cC_{\cF_{\ell}}$ is a subclass of the class of homogeneous
halfspaces in $\mathbb R^m$, and therefore
\[
    d_{\rm cmp}\le m.
\]
Consequently, $d=O(m)$ and
\[
    \alpha_\delta
    \lesssim
    \frac{m\log(en)+\log(2/\delta)}{n}.
\]

\paragraph{Sparse linear and sparse monotone-link prediction classes.}
Let $\phi:\cX\to\mathbb R^p$, let $1\le s\le p$, and restrict parameter
vectors to satisfy $\normz{\theta}_0\le s$. Consider either the bounded
linear class
\[
    f_\theta(x)
    =
    a(x)+\langle\theta,\phi(x)\rangle
\]
or the monotone-link class
\[
    f_\theta(x)
    =
    \ell\!\left(
        a(x)+\langle\theta,\phi(x)\rangle
    \right),
\]
with the linear class restricted as needed so that its predictions lie
in $[0,1]$.

There are at most
\[
    \sum_{j=0}^s\binom{p}{j}
    \le
    \left(\frac{ep}{s}\right)^s
\]
possible supports of size at most $s$, and conditional on a fixed
support the prediction class has pseudo-dimension $O(s)$. Combining
the corresponding growth-function bound with a union bound over
supports gives
\[
    d_{\rm pred}
    =
    O\!\left(
        s\log\frac{ep}{s}
    \right).
\]
See, e.g.,
\citet{anthony2009neural,shalev-schwartz14understanding}
for the standard growth-function and pseudo-dimension machinery.

For pairwise comparisons, if $\theta$ and $\gamma$ are both
$s$-sparse, then $\gamma-\theta$ is at most $2s$-sparse. Hence
$\cC_{\cF}$ is a subclass of the class of $2s$-sparse homogeneous
halfspaces in $\mathbb R^p$. Applying the same support-counting and
growth-function argument gives
\[
    d_{\rm cmp}
    =
    O\!\left(
        s\log\frac{ep}{s}
    \right).
\]
Related explicit lower and upper bounds on the VC dimension of sparse
halfspaces are given by \citet{ahsen2019approach}. Therefore
\[
    d
    =
    O\!\left(
        s\log\frac{ep}{s}
    \right),
\]
and Theorem~\ref{thm:sbr-infinite-classes} applies with
\[
    \alpha_\delta
    \lesssim
    \frac{
        s\log(ep/s)\log(en)
        +\log(2/\delta)
    }{n}.
\]

\begin{remark}[Realizable specialization]
For any of the classes above, if $\eta\in\cF$, then
$\epsilon_1(\eta)=B_P(\eta)=0$ and $v(\eta)\le\bar\sigma^2$. Hence
Theorem~\ref{thm:sbr-infinite-classes} implies, up to constants depending
only on $\rho$,
\[
    \epsilon_1(\hf)
    \lesssim_\rho
    \sqrt{
        \bar\sigma^2
        \frac{d\log(en)+\log(1/\delta)}{n}
    }
    +
    \frac{d\log(en)+\log(1/\delta)}{n}.
\]
Thus the same variance-adaptive interpolation as in the finite-class
result continues to hold for these infinite classes.
\end{remark}

\section{Experiments}
\label{app:exp}
\subsection{Synthetic Experiment}
In this section, we provide empirical results that validate our theoretical findings.
Our synthetic experiment is a controlled finite-class sanity check. It isolates the approximation behavior predicted by the theory by exactly enumerating a small candidate class. It is not intended to demonstrate scalable optimization of the sign-betting objective over large hypothesis classes.
\paragraph{Function class.}
We consider a $d$-dimensional regression setting with target regression
function $\eta$. We use the logistic prediction family
\[
f_\theta(x) := \sigma(x^\top \theta),
\]
where $\sigma(z):=1/(1+\exp(-z))$ is the sigmoid function.
We define the target regression function as
\[
\eta(x)=\frac{1}{2}, \qquad \forall x\in\mathcal X.
\]

To study model misspecification in a controlled setting, we consider a finite
function class that does not contain the target regression function.
Specifically, we choose two inputs
\[
x_c=e_1,
\qquad
x_r=e_2,
\]
where $e_1,e_2\in\mathbb{R}^d$ are the first two standard basis vectors, and
construct two groups of candidate functions. \han{20-candidate localized-versus-diffuse experiment}

The first group consists of \emph{local-error} candidates that agree with the
target function on $x_c$ but deviate from it on $x_r$:
\[
f_j^L(x_c)=\frac{1}{2},
\qquad
f_j^L(x_r)=\frac{1}{2}+A_j.
\]
Their corresponding parameters are
\[
\theta_j^L
=
\left(
0,\,
\operatorname{logit}\!\left(\frac{1}{2}+A_j\right),
0,\ldots,0
\right).
\]

The second group consists of \emph{distributed-error} candidates that deviate
from the target function on both inputs:
\[
f_j^D(x_c)
=
f_j^D(x_r)
=
\frac{1}{2}+B_j.
\]
Their corresponding parameters are
\[
\theta_j^D
=
\left(
\operatorname{logit}\!\left(\frac{1}{2}+B_j\right),
\operatorname{logit}\!\left(\frac{1}{2}+B_j\right),
0,\ldots,0
\right).
\]

Let $\mathcal{F}_L$ and $\mathcal{F}_D$ denote the two groups of candidate
functions. We finally construct the finite function class as
\[
\mathcal{F}
:=
\mathcal{F}_L \cup \mathcal{F}_D.
\]
The target regression function $\eta$ is not contained in $\mathcal{F}$.

\paragraph{Data distribution.}
We consider a two-point input distribution with
\[
x_c=e_1=(1,0)^\top,
\qquad
x_r=e_2=(0,1)^\top,
\]
and sample
\[
X=
\begin{cases}
x_c, & \text{with probability } 1-p,\\
x_r, & \text{with probability } p.
\end{cases}
\]
The target regression function is fixed as
\[
\eta(x)
=
\mathbb{E}[Y\mid X=x]
=
\frac{1}{2}.
\]

We use a parameter $\rho\in[0,1)$ to control the conditional label noise. This experimental noise parameter is unrelated to the parameter $\rho$ appearing in our theoretical bounds.
When $\rho=0$, the label is deterministic:
\[
Y\mid X=x=\frac{1}{2}.
\]
For $\rho>0$, we generate the label as
\[
Y\mid X=x
\sim
\operatorname{Beta}\left(
\eta(x)\frac{1-\rho}{\rho},
(1-\eta(x))\frac{1-\rho}{\rho}
\right).
\]
Since $\eta(x)=1/2$, this reduces to
\[
Y\mid X=x
\sim
\operatorname{Beta}\left(
\frac{1-\rho}{2\rho},
\frac{1-\rho}{2\rho}
\right).
\]
In both cases,
\[
\mathbb{E}[Y\mid X=x]=\frac{1}{2},
\qquad
\operatorname{Var}(Y\mid X=x)=\frac{\rho}{4},
\]
where $\rho=0$ corresponds to the deterministic label setting.

\paragraph{Training.}
Given the candidate set $\mathcal F$, we select a function by minimizing either
log loss, squared loss, or sign betting loss on the validation data
$D_n^{\mathrm{val}}$. The log loss of $f$ is
\[
L_n^{\log}(f)
=
-\frac{1}{n}
\sum_{(x,y)\in D_n}
\left(
y\log f(x)
+
(1-y)\log(1-f(x))
\right),
\]
and the selected function is
\[
\widehat f^{\log}
=
\arg\min_{f\in\mathcal F}
L_n^{\log}(f).
\]

Similarly, the squared loss of $f$ is
\[
L_n^{\mathrm{sq}}(f)
=
\frac{1}{n}
\sum_{(x,y)\in D_n}
\left(y-f(x)\right)^2,
\]
and the selected function is
\[
\widehat f^{\mathrm{sq}}
=
\arg\min_{f\in\mathcal F}
L_n^{\mathrm{sq}}(f).
\]

For sign betting loss, for each $f,h\in\mathcal F$, we define
\[
\widehat T(f,h)
=
\frac{1}{n}
\sum_{(x,y)\in D_n}
\left(y-f(x)\right)
\operatorname{sign}\!\left(h(x)-f(x)\right),
\]
where $\operatorname{sign}(0)=0$, and select
\[
\widehat f^{\mathrm{SB}}
=
\arg\min_{f\in\mathcal F}
\max_{h\in\mathcal F}
\widehat T(f,h).
\]
\han{The notation might not be consistent with main paper, need double check.}

\paragraph{Evaluation.}
For each selected function $\widehat f$, we evaluate its population mean absolute error (MAE),
\[
\epsilon_1(\widehat f)
=
\mathbb{E}_X
\left[
\left|
\widehat f(X)-\eta(X)
\right|
\right].
\]

Since the input distribution is supported on only two points, this expectation can be computed exactly as
\[
\epsilon_1(\widehat f)
=
(1-p)
\left|
\widehat f(x_c)-\frac{1}{2}
\right|
+
p
\left|
\widehat f(x_r)-\frac{1}{2}
\right|.
\]

We fix the feature dimension to $d=2$ and the rare-input probability to $p=0.1$, vary the training sample-to-dimension ratio $n/d \in \{2,4,8,16,32,64\},$ and consider different noise levels $\rho \in \{0,0.02,0.04\}.$
For each configuration $(n/d,\rho)$, we repeat the experiment $B=1000$  times.
If $\widehat f_b$ denotes the selected function in repetition $b$, we report
the average MAE
\[
\overline{\epsilon}_1
=
\frac{1}{B}
\sum_{b=1}^{B}
\epsilon_1(\widehat f_b),
\]
with standard error
\[
\mathrm{SE}
=
\frac{
\operatorname{SD}
\left(
\epsilon_1(\widehat f_1),\ldots,
\epsilon_1(\widehat f_B)
\right)
}{
\sqrt{B}
}.
\]

\begin{table}[t]
\centering

\caption{
Comparison of average mean absolute error (MAE) under log loss, squared loss,
and sign betting loss across varying $n/d$. Reported values are the mean MAE
over repeated trials, with $\pm$ denoting the standard error. Best results are
shown in \textbf{bold}. Sign betting loss achieves the lowest MAE for all tested
values of $n/d$, tying with the other methods at $n/d=2$.
}
\label{tab:synthetic-mae}

\begin{tabular}{c c c c}
\toprule
$n/d$
& Log loss
& Squared loss
& Sign betting loss \\
\midrule
2
& $\mathbf{0.05320 \pm 0.00082}$
& $\mathbf{0.05320 \pm 0.00082}$
& $\mathbf{0.05320 \pm 0.00082}$ \\

4
& $0.06629 \pm 0.00086$
& $0.06629 \pm 0.00086$
& $\mathbf{0.04490 \pm 0.00067}$ \\

8
& $0.07960 \pm 0.00068$
& $0.06068 \pm 0.00087$
& $\mathbf{0.04605 \pm 0.00070}$ \\

16
& $0.08142 \pm 0.00063$
& $0.06937 \pm 0.00084$
& $\mathbf{0.03945 \pm 0.00047}$ \\

32
& $0.08411 \pm 0.00054$
& $0.07850 \pm 0.00071$
& $\mathbf{0.03714 \pm 0.00034}$ \\

64
& $0.08879 \pm 0.00026$
& $0.08428 \pm 0.00053$
& $\mathbf{0.03538 \pm 0.00015}$ \\
\bottomrule
\end{tabular}

\end{table}

\paragraph{Results.}
Table~\ref{tab:synthetic-mae} reports the average MAE with standard error for the noiseless setting ($\rho=0$), while Figure~\ref{fig:syn-mae} summarizes the corresponding results across different noise levels $\rho\in\{0,0.02,0.04\}$ and varying values of sample-to-dimension ratio $n/d$.
In the noiseless setting, all three methods achieve the same average MAE at the smallest sample size. For all larger values of $n/d$, sign betting consistently achieves the lowest MAE, and its MAE decreases toward the best population MAE attainable within the function class as the sample size increases.
Figure~\ref{fig:syn-mae} shows that this advantage persists in the presence of label noise. Across both noisy settings, sign betting consistently achieves the lowest MAE over all tested values of $n/d$, and its MAE decreases as the sample size increases.

\begin{figure*}[t]
    \centering
    \includegraphics[width=\textwidth]{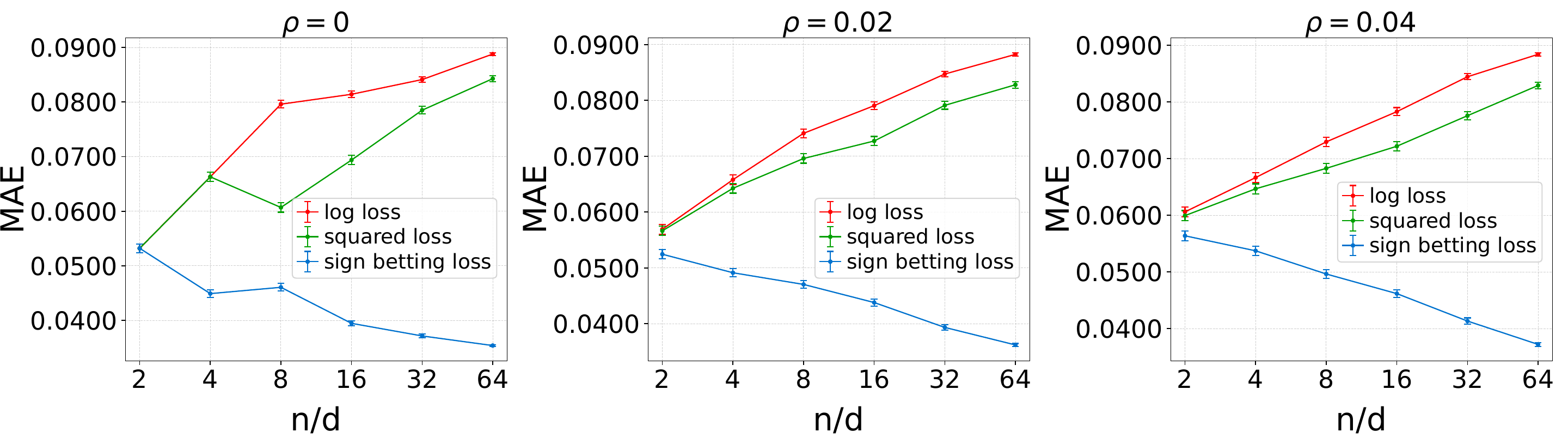}

    \caption{
    Comparison of average mean absolute error (MAE) obtained under log loss, squared loss, and sign betting loss in the synthetic experiment across varying $n/d$ and noise levels $\rho\in\{0,0.02,0.04\}$. Error bars denote the standard error.
    }
    \label{fig:syn-mae}
\end{figure*}

\subsection{Real-World Experiment}
\paragraph{Dataset.}
We use the KonIQ-10k~\cite{8968750} image-quality dataset, in which each image is associated with multiple human quality ratings on a five-point scale.
We rescale each rating $r\in\{1,2,3,4,5\}$ to
$
Y=\frac{r-1}{4}\in
\left\{
0,\frac14,\frac12,\frac34,1
\right\},
$
so that the responses lie in $[0,1]$, consistent with our regression setting.

For an image $x$, let $c_k(x)$ denote the number of observed ratings at level
$k\in\{1,\ldots,5\}$. We define the empirical rating distribution by
\[
\widehat q_k(x)
=
\frac{c_k(x)}
{\sum_{\ell=1}^{5} c_\ell(x)},
\]
and its empirical mean by
\[
\bar{\eta}(x)
=
\sum_{k=1}^{5}
\widehat q_k(x)\frac{k-1}{4}.
\]
Here $\eta(x)=\mathbb{E}[Y\mid X=x]$ denotes the population conditional mean from our theoretical formulation, while $\bar{\eta}(x)$ is its empirical counterpart computed from all available human ratings for image $x$.
\han{This paragraph try to say where eta comes from.}

\paragraph{Function class.}
We extract image features using a ResNet-18 ~\cite{7780459} backbone pretrained on
ImageNet~\cite{ILSVRC15}, with all backbone parameters frozen. Based on these features, we train five families of regression models on the training set using different model configurations and random seeds.Together, these models form a finite hypothesis class containing 38 predictors,
\[
\mathcal F
=
\{f_1,\ldots,f_{38}\}.
\]
Once the candidate models are trained, all model parameters are frozen during model selection.

\paragraph{Training.}
Given the hypothesis class $\mathcal F$, we select a function using log loss, squared loss, or sign betting loss on the validation set. For each repetition, we use the same $n$ distinct validation images for all three methods and sample one human rating for each selected image, forming
\[
D_n^{\mathrm{val}}
=
\{(X_i,Y_i)\}_{i=1}^n.
\]

The log loss of $f$ is
\[
L_n^{\log}(f)
=
-\frac{1}{n}
\sum_{(x,y)\in D_n^{\mathrm{val}}}
\left(
y\log f(x)
+
(1-y)\log(1-f(x))
\right),
\]
and the selected function is
\[
\widehat f^{\log}
=
\arg\min_{f\in\mathcal F}
L_n^{\log}(f).
\]

Similarly, the squared loss of $f$ is
\[
L_n^{\mathrm{sq}}(f)
=
\frac{1}{n}
\sum_{(x,y)\in D_n^{\mathrm{val}}}
\left(y-f(x)\right)^2,
\]
and the selected function is
\[
\widehat f^{\mathrm{sq}}
=
\arg\min_{f\in\mathcal F}
L_n^{\mathrm{sq}}(f).
\]

For sign betting, we use
\[
\widehat T(f,h)
=
\frac{1}{n}
\sum_{(x,y)\in D_n^{\mathrm{val}}}
\left(y-f(x)\right)
\operatorname{sign}\!\left(h(x)-f(x)\right),
\]
and select
\[
\widehat f^{\mathrm{SB}}
=
\arg\min_{f\in\mathcal F}
\max_{h\in\mathcal F}
\widehat T(f,h),
\]
where $\operatorname{sign}(0)=0$.

\paragraph{Evaluation.}
We evaluate each selected function $\widehat f$ on the test set using mean absolute error (MAE) with respect to the empirical mean human score
$\bar{\eta}(x)$:
\[
\widehat R_{\mathrm{test}}(\widehat f)
=
\frac{1}{m}
\sum_{i=1}^{m}
\left|
\widehat f(X_i^{\mathrm{test}})
-
\bar{\eta}(X_i^{\mathrm{test}})
\right|,
\]
where $m=2015$ is the number of test images.

We also report the model-selection regret relative to the best model in $\mathcal F$. Specifically, we define
\[
R^{\mathrm{best}}
=
\min_{f\in\mathcal F}
\widehat R_{\mathrm{test}}(f),
\]
and the regret of a selected function $\widehat f$ as
\[
\operatorname{Regret}(\widehat f)
=
\widehat R_{\mathrm{test}}(\widehat f)
-
R^{\mathrm{best}}.
\]
Here, $R^{\mathrm{best}}$ denotes the test MAE of the best model among the
38 candidates in $\mathcal F$.

\paragraph{Results.}
 Figure~\ref{fig:real-llm-mae}(a) summarizes the average test MAE with standard
error from function selection with log loss, squared loss, and sign betting
loss across different validation sizes $n$, together with the test MAE of the
best model in each hypothesis class.
The results show that the test MAE decreases as the validation size increases
for all three selection methods, with sign betting consistently achieving the
lowest MAE and approaching the best model as validation size increases.
Similarly, the model selection regret decreases with increasing $n$, with sign
betting consistently attaining the lowest regret among the three methods.

\begin{figure*}[t]
    \centering
    \includegraphics[width=0.9\textwidth]{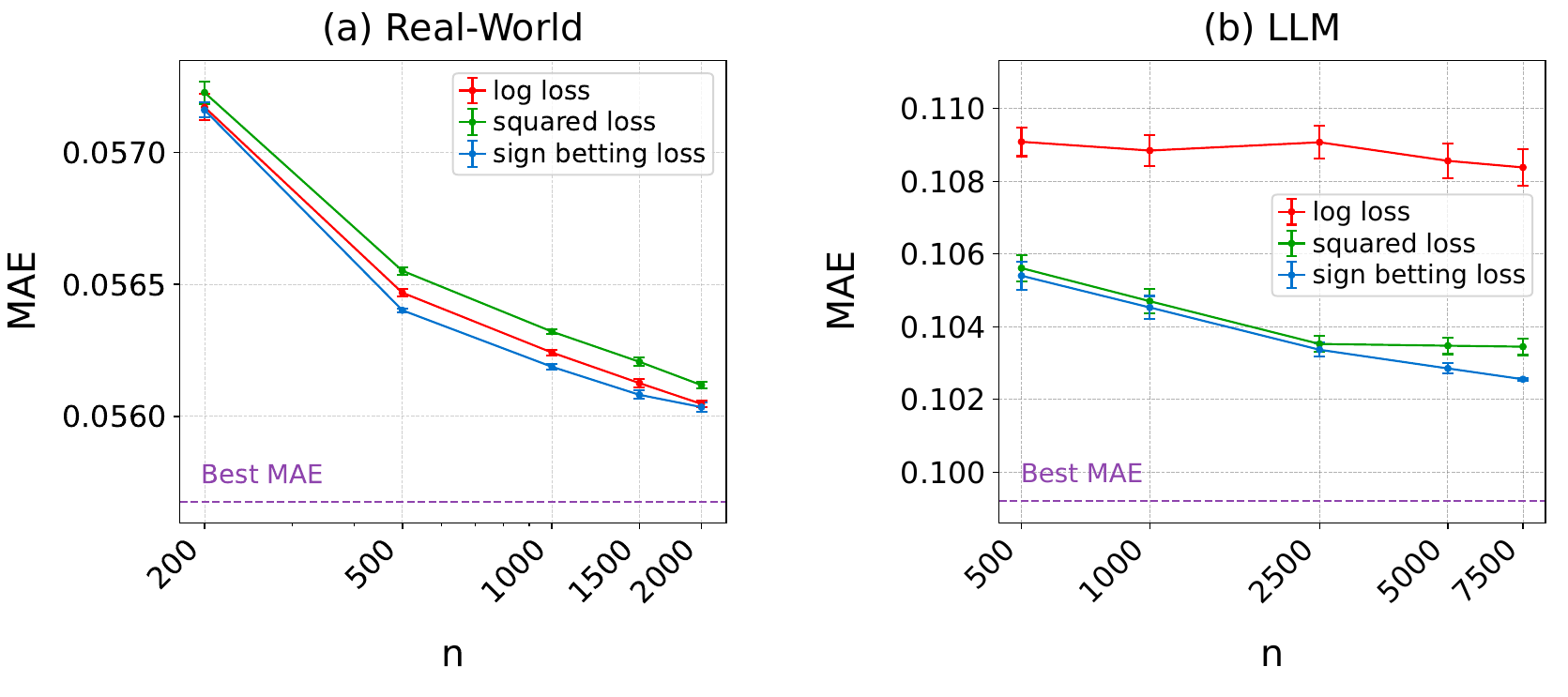}

    \caption{
    Comparison of average test MAE obtained under log loss, squared loss, and sign betting loss in the real-world and LLM experiments across varying validation sizes $n$. Error bars denote the standard error, and dashed horizontal lines indicate the MAE of the best model in each function class.
    }
    \label{fig:real-llm-mae}
\end{figure*}

\subsection{LLM Experiment}

\paragraph{Dataset.}
We use the UltraFeedback~\cite{pmlr-v235-cui24f} dataset, which contains
prompts paired with multiple responses. In our experiments, each prompt is associated with $4$ responses, and we treat each prompt-response
pair as one regression sample. Specifically, we represent an input as
\[
x=(q,r),
\]
where $q$ denotes the prompt and $r$ denotes one of its responses.

To obtain reference rewards, we use Llama-3.3-70B-Instruct~\cite{grattafiori2024llama3herdmodels} as the teacher model.
Given a prompt-response pair $x$, the teacher assigns a reward
\[
s_{\mathrm{teacher}}(x)\in\{0,1,2,3,4\}.
\]
We normalize the score to $[0,1]$ and define
\[
Y
=
g(x)
=
\frac{s_{\mathrm{teacher}}(x)}{4}.
\]
Thus, each regression example consists of a prompt-response pair $x$ and its
normalized teacher reward $Y$. We use $g(x)$ as the teacher reference reward
for both model selection and evaluation.

\paragraph{Function class.}
We construct a finite class of candidate reward models using Llama-3.2-3B-Instruct~\cite{grattafiori2024llama3herdmodels} as the
backbone. We attach a five-class scoring head that predicts
\[
p_\theta(k\mid x),
\qquad
k\in\{0,1,2,3,4\}.
\]

The models are trained on the training set using the teacher reward as labels. 
We use seven different random seeds and retain five checkpoints along each training trajectory, resulting in $7\times5=35$ reward models. Although the models are trained as five-class classifiers, each model outputs a prediction reward in $[0,1]$ using the normalized expected score
\[
f_j(x)
=
\sum_{k=0}^{4}
p_j(k\mid x)\frac{k}{4}.
\]
We therefore define the finite hypothesis class as
\[
\mathcal F
=
\{f_1,\ldots,f_{35}\}.
\]
All reward models are frozen before the model selection stage, and the
teacher model is not included in $\mathcal F$.

\paragraph{Training.}
Given the candidate set $\mathcal F$, we select a function by minimizing either
log loss, squared loss, or sign betting loss on the validation data.
For each repetition, we sample $n$ prompt-response pairs from the validation set, forming
\[
D_n^{\mathrm{val}}
=
\{(X_i,Y_i)\}_{i=1}^n.
\]

The log loss of $f$ is
\[
L_n^{\log}(f)
=
-\frac{1}{n}
\sum_{(x,y)\in D_n^{\mathrm{val}}}
\left[
y\log \widetilde f(x)
+
(1-y)\log(1-\widetilde f(x))
\right],
\]
where
\[
\widetilde f(x)
=
\operatorname{clip}\!\left(f(x),10^{-6},1-10^{-6}\right),
\]
and the selected function is
\[
\widehat f^{\log}
=
\arg\min_{f\in\mathcal F}
L_n^{\log}(f).
\]

Similarly, the squared loss of $f$ is
\[
L_n^{\mathrm{sq}}(f)
=
\frac{1}{n}
\sum_{(x,y)\in D_n^{\mathrm{val}}}
\left(y-f(x)\right)^2,
\]
and the selected function is
\[
\widehat f^{\mathrm{sq}}
=
\arg\min_{f\in\mathcal F}
L_n^{\mathrm{sq}}(f).
\]

For sign betting, we use
\[
\widehat T(f,h)
=
\frac{1}{n}
\sum_{(x,y)\in D_n^{\mathrm{val}}}
\left(y-f(x)\right)
\operatorname{sign}\!\left(h(x)-f(x)\right),
\]
and select
\[
\widehat f^{\mathrm{SB}}
=
\arg\min_{f\in\mathcal F}
\max_{h\in\mathcal F}
\widehat T(f,h),
\]
where $\operatorname{sign}(0)=0$.
\han{same training paragraph}

\paragraph{Evaluation.}
We evaluate each selected reward model $\widehat f$ on the test set using
mean absolute error (MAE) with respect to the teacher reference reward:
\[
R_{\mathrm{test}}(\widehat f)
=
\frac{1}{m}
\sum_{i=1}^{m}
\left|
\widehat f(x_i^{\mathrm{test}})
-
y_i^{\mathrm{test}}
\right|,
\]
where the test set contains $m=4000$ prompt-response pairs and $y_i^{\mathrm{test}}$ is the normalized reward assigned by the teacher
model. Lower MAE indicates that the selected reward model more closely matches
the teacher scoring function on the test set.

We also report the model selection regret relative to the best model in the
function class. Specifically, we define
\[
R^{\mathrm{best}}
=
\min_{f\in\mathcal F}
R_{\mathrm{test}}(f),
\]
and
\[
\operatorname{Regret}(\widehat f)
=
R_{\mathrm{test}}(\widehat f)
-
R^{\mathrm{best}}.
\]
Here, $R^{\mathrm{best}}$ is the test MAE of the best model among the 35
candidates in $\mathcal F$. 

\paragraph{Results.}
Figure~\ref{fig:real-llm-mae}(b) summarizes the average test MAE and model selection regret, with standard error, for function selection with log loss, squared loss, and sign betting loss across different validation sizes $n$. The test MAE of the best model in the function class is also shown as a reference.
Across all validation sizes, sign betting consistently achieves lower
test MAE than both log loss and squared loss. As the validation size increases,
the test MAE of sign betting decreases and moves closer to the best model. Similarly, the model selection regret of sign betting decreases with increasing $n$ and remains the lowest among the three methods across all validation sizes.

\end{document}